\documentclass{article}

\providecommand*{\abs}[1]{\left|{#1}\right|}
\providecommand*{\absnormal}[1]{|{#1}|}
\providecommand*{\N}[1]{\left\|{#1}\right\|} 
\providecommand*{\Nnormal}[1]{\|{#1}\|} 
\providecommand*{\Nbig}[1]{\big\|{#1}\big\|} 

\providecommand{\bbR}{{\mathbb{R}}}

\newcommand{\x}{\mathbf{x}}
\newcommand{\y}{\mathbf{y}}

\providecommand{\Id}{\mathrm{Id}}                     

\providecommand{\CB}{{\cal B}}

\providecommand{\CD}{{\cal D}}

\providecommand{\CO}{{\cal O}}
\providecommand{\CP}{{\cal P}}

\providecommand{\bbE}{\mathbb{E}}

\providecommand{\bbN}{\mathbb{N}}

\providecommand{\bbP}{\mathbb{P}}

\providecommand{\bbR}{\mathbb{R}}

\providecommand*{\abs}[1]{\left|{#1}\right|}
\providecommand*{\absnormal}[1]{|{#1}|}
\providecommand*{\N}[1]{\left\|{#1}\right\|} 
\providecommand*{\Nnormal}[1]{\|{#1}\|} 
\providecommand*{\Nbig}[1]{\big\|{#1}\big\|} 

\providecommand*{\abs}[1]{\left|{#1}\right|} 

\usepackage{titletoc}

\makeatletter
\newcommand{\startnewtoc}{%
  \immediate\write\@auxout{\string\@writefile{toc}{\string\contentsline{section}{}{}}}%
}
\makeatother

\DeclareFontEncoding{LS1}{}{}
\DeclareFontSubstitution{LS1}{stix2}{m}{n}
\DeclareMathAlphabet{\mathscr}{LS1}{stix2scr}{m}{n}
\renewcommand{\CB}{\mathscr{B}}

\usepackage{amssymb,amsthm}
\usepackage{amsmath}

\newtheorem{theorem}{Theorem}
\newtheorem{lemma}[theorem]{Lemma} 
\newtheorem{corollary}[theorem]{Corollary}

\newtheorem{assumption}[theorem]{Assumption}
\theoremstyle{definition}

\newtheorem{remark}[theorem]{Remark}

\usepackage{xcolor}
\definecolor{OXBLUE}{RGB}{0,33,71}
\definecolor{OXROYAL}{RGB}{29,66,166}

\definecolor{OXGRAY}{RGB}{97, 97, 95}

\definecolor{OXROYAL}{RGB}{29,66,166}\definecolor{OXGREEN}{RGB}{149, 193, 31}
\definecolor{OXORANGE}{RGB}{251, 86, 7}
\definecolor{OXRED}{RGB}{190, 15, 52}

\definecolor{OXGREENtwo}{RGB}{0, 170, 180}

\definecolor{OXMIX}{RGB}{89,130,99}

\usepackage{bm}
\usepackage{bbm}
\usepackage{arydshln}
\usepackage{subcaption}
\usepackage{xcolor}
\usepackage{ifdraft}
\usepackage{scalerel}
\usepackage{mathrsfs}  
\usepackage{etoc}
\usepackage{comment}
\usepackage{soul}

\usepackage{tikz}
\usetikzlibrary{decorations.pathreplacing}

\usepackage{latexsym,amsfonts}
\usepackage[utf8]{inputenc}
\usepackage[final=true,protrusion=true,expansion=true]{microtype}
\usepackage[noadjust]{cite}
\usepackage{enumitem}
\usepackage{color}
\usepackage{mathtools}

\usepackage{booktabs}
\usepackage{multirow}

     \PassOptionsToPackage{numbers, compress}{natbib}

\usepackage[main, final]{neurips_2026}

\newif\ifarxiv
\arxivtrue       

\usepackage[utf8]{inputenc} 
\usepackage[T1]{fontenc}    
\usepackage{hyperref}       
\usepackage{url}            
\usepackage{booktabs}       
\usepackage{amsfonts}       
\usepackage{nicefrac}       
\usepackage{microtype}      
\usepackage{xcolor}         
\usepackage{cleveref}

\title{Convergence Analysis of Newton's Method for\\ Neural Networks in the Overparameterized Limit}

\author{%
  Konstantin Riedl \\
  Mathematical Institute\\
  University of Oxford\\
  \texttt{Konstantin.Riedl@maths.ox.ac.uk} \\
  \And
  Konstantinos Spiliopoulos \\
  Department of Mathematics \& Statistics \\
  Boston University\\
  \texttt{kspiliop@bu.edu}
  \AND
    Justin Sirignano \\
  Mathematical Institute\\
  University of Oxford\\
  \texttt{Justin.Sirignano@maths.ox.ac.uk} \\
}

\newcommand{\acknowledgments}{%
\begin{ack}
    This research project was supported by ``DMS-EPSRC: Asymptotic Analysis of Online Training Algorithms in Machine Learning: Recurrent, Graphical, and Deep Neural Networks'' (NSF DMS-2311500).
    The authors would like to acknowledge the use of the University of Oxford Advanced Research Computing (ARC) facility in carrying out this work. 
    
    For the purpose of Open Access, the authors have applied a CC BY public copyright license to any Author Accepted Manuscript (AAM) version arising from this submission.
\end{ack}%
}

\begin{document}

\maketitle


\begin{abstract}
A convergence analysis is developed for the regularized Newton method for training neural networks (NNs) in the overparameterized limit. We prove that as the number of hidden units tends to infinity, the NN training dynamics converge in probability to the solution of a deterministic limit equation involving a ``Newton neural tangent kernel'' (NNTK). Explicit rates characterizing this convergence are provided and, in the infinite-width limit, we prove that the NN converges exponentially fast to the target data (i.e., a global minimizer with zero loss). Crucially, we show that this convergence is uniform across the frequency spectrum, addressing the spectral bias inherent in gradient descent. The eigenvalues of the NTK for gradient descent accumulate at zero, leading to slow convergence for target data with high-frequency components. In contrast, the NNTK has uniformly lower bounded eigenvalues if the regularization parameter is selected appropriately, allowing Newton's method to converge more quickly for data with high-frequency components. Mathematical challenges that our analysis needs to address include the implicit parameter update of the Newton method with a potentially indefinite Hessian matrix and the fact that the dimension of this linear system of equations tends to infinity as the NN width grows. This substantially complicates deriving the training dynamics in the overparameterized limit as well as proving the convergence of the finite-width dynamics thereto. Our analysis identifies a scaling formula for selecting the regularization parameter, which we show can vanish at a suitable rate as the NN width becomes larger. In addition, we prove that, for sufficiently large numbers of hidden units, the regularized Hessian remains positive definite during training and the Newton updates for individual NN parameters converge to zero, demonstrating that the model behaves as a linearization around the initialization.
\end{abstract}


\section{Introduction}\label{S:Introduction}

Although the mean-field and overparameterized asymptotics of gradient descent methods for neural network (NN) training have been widely-studied~\cite{jacot2018neural,chizat2019lazy,sirignano2019scaling,sirignano2020meanlawlarenumbers,sirignano2020meancentrallimit,sirignano2022mean,mei2018mean,chizat2018global,rotskoff2018trainability,chizat2024infinite},
there has been limited mathematical analysis of second-order optimization methods for training NNs~\cite{arbel2023rethinking,cai2019gram,cayci2024gauss,adeoye2024regularized,ishikawa2023parameterization,zhang2019fast,karakida2020understanding,bai2022generalized}.
Specifically, the literature currently lacks
a convergence analysis of Newton's method for NNs,
despite it being a cornerstone of classical optimization and having been the foundation for the development of numerous popular second-order optimization algorithms \cite{nocedal2006optimization,boyd2004convex,fletcher2001practical,conn2000trust,martens2016second,mishchenko2023regularized,liu1989limited} employed in science and engineering as well as in machine learning.
These include quasi-Newton (like Broyden's method and (L-)BFGS, or the more recent K-FAC~\cite{martens2015optimizing} or Shampoo~\cite{gupta2018shampoo}) optimizers as well as trust-region methods~\cite{conn2000trust},
which are the focus of significant current research efforts regarding their applications in machine learning and AI~\cite{martens2016second,anil2020scalable,vyas2024soap,jordan2024muon,gomes2024adafisher} (although not as widely-used in the past in machine learning as first-order methods).
In particular, second-order optimization methods have had significant recent success in a number of scientific machine learning applications where they have outperformed first-order methods  \cite{bonfanti2024challenges,dangel2024kronecker,kiyani2025optimizing,daryakenari2026representation,jnini2025gauss,jnini2026curvature,muller2023achieving,wang2025discovery}.

In this paper, we develop a convergence analysis for the classical regularized Newton method for training NNs in the overparameterized regime \cite{jacot2018neural,chizat2019lazy,sirignano2019scaling}.
We consider a fully-connected single-hidden-layer NN~$f_\theta^N:\bbR^d\rightarrow\bbR$ with $N$ neurons of the form
\begin{equation}
    \label{eq:NN}
    f_\theta^N(x)
    = \frac{1}{N^\beta} \sum_{i=1}^N c^i \sigma(w^i \cdot x + \eta^i),
\end{equation}
where the vector $\theta=(\theta^i)_{i=1}^N$ with $\theta^i=(c^i,w^i,\eta^i)$ collects all $N(d+2)$ trainable NN parameters.
For our theoretical results, we require Assumption~\ref{def:NN_assumptions} on the NN \eqref{eq:NN}.
The factor~$1/N^\beta$ in \eqref{eq:NN} with $\beta\in(1/2,1)$ is a normalization/scaling~\cite{spiliopoulos2025mathematical,sirignano2019scaling,sirignano2023pde,riedl2025global}.
The NN parameters~$\theta$ are trained to minimize the loss~$L^N$ on the dataset $\CD = (x^m,y^m)_{m=1}^M \subset \bbR^d\times\bbR$,
given as
\begin{equation}
    \label{eq:loss}
    L^N(\theta)
    = \frac{1}{2M}\sum_{m=1}^M \left( y^m - f^N_\theta(x^m) \right)^2
    = \frac{1}{2M} \N{\y-f_{\theta}^N(\x)}_2^2,
\end{equation}
where $f_{\theta}^N(\x)$ and $\y$ denote the vectorized NN output for all $M$ data samples and the vector of all $M$ labels, respectively.
To find the optimal NN parameters for \eqref{eq:loss},
we employ the regularized Newton method with step size $\alpha\in\left(0,2/\!\left(\lambda_{\max}(\CB^*_0)+\lambda_{\min}(\CB^*_0)\right)\right)\supset(0,1]$,
where $\CB^*_0$ denotes the Newton neural tangent kernel (NNTK) in the overparameterized limit defined in \eqref{eq:NNTK*}.
That is,
\begin{equation}
    \label{eq:regularizedNewton}
    \theta_{k+1}
    = \theta_k + \alpha z^N_k
    \quad\text{ with }\quad
    \theta_0^i \stackrel{\text{i.i.d.\@}}{\sim} \mu_0
    \text{ for all }i=1,\dots,N,
\end{equation}
where the Newton update~$z^N_k$ is the unique solution (and is set to zero otherwise) of the linear system
\begin{equation}
    \label{eq:regularizedNewtonStep}
    \left(\gamma^N \Id_{N(d+2)} + \nabla^2_\theta L^N(\theta_k)\right) z^N_k
    = -\nabla_\theta L^N(\theta_k)
\end{equation}
with regularizer $\gamma^N$ scaling as $\gamma^N = \frac{\gamma}{N^{2\beta-1}}$, where the regularizer parameter~$\gamma$ is assumed to be strictly positive ($\gamma>0$) throughout the paper.
Note that we allow for the practicable step size $\alpha=1$.


\paragraph{Contributions.}
We develop a convergence analysis for the regularized Newton method for training NNs of the form \eqref{eq:NN} in the overparameterized limit (i.e., as $N\rightarrow\infty$).
\mbox{Our main contributions are:}
\begin{itemize}[label=\large \textbullet,labelsep=8pt,leftmargin=25pt,topsep=-4pt,itemsep=-1pt]
    \item \emph{NNTK.}
    We derive the NNTK~\eqref{eq:NNTK}
    governing the training evolution~\eqref{eq:NNfunction_evolution_prelimit} of the NN~\eqref{eq:NN} when its parameters are trained with the regularized Newton method~\eqref{eq:regularizedNewton}--\eqref{eq:regularizedNewtonStep}.
    This kernel is random at initialization, varies during training, and involves the implicit solution of a linear system of equations whose dimensionality tends to infinity as the NN width $N\rightarrow\infty$ (see Section~\ref{subsec:reduction}).
    \item \emph{Infinite-width limit.}
    As the width $N\rightarrow\infty$, we rigorously show for a given number of training steps $K<\infty$ that the random, time-varying NNTK converges in probability (with explicit rates) to a deterministic limit~\eqref{eq:NNTK*} that remains constant during training.
    This allows us to prove for a given number of training steps $K<\infty$ that the finite-width discrete-time NN training trajectory~\eqref{eq:NNfunction_evolution_prelimit} converges in probability to the solution of an ``overparameterized'' deterministic limit equation~\eqref{eq:NNfunction_evolution_limit} as $N \rightarrow \infty$ (see Theorem~\ref{thm:NTK_limit}). 
    A key step is to show that the Newton updates for individual NN parameters vanish in the infinite-width regime (see Lemma~\ref{prop:NTK_limit_bounds_z}). 
    \item \emph{Global convergence in the overparameterized limit.}
    In this infinite-width limit, we prove that the NN function output~\eqref{eq:NNfunction_evolution_limit} converges exponentially fast to the target data (i.e., a global minimizer with zero loss) as the number of training steps $k\rightarrow\infty$ (see Theorem~\ref{thm:convergence}).
    Crucially and in contrast to gradient descent, we show that this convergence is uniform across the frequency spectrum, provided a sufficiently small regularizer parameter~$\gamma$,
    thereby provably mitigating the spectral bias inherent in gradient descent.
    This is attributed to better conditioning of the NNTK compared to its gradient descent counterpart (see Remark~\ref{rem:NNTKB_conv}). The unregularized Newton method ($\gamma=0$) converges in the infinite-width limit in one training step (see Remark~\ref{rem:conv_one_step}).
    \item \emph{Probabilistic  convergence for finite-width NNs and a finite number of training steps.}
    Combining the previous two results allows us to bridge
    the global convergence results to the finite-width regime,
    resulting in a high-probability convergence result in terms of the NN width and the number of Newton steps used during training (see Theorem~\ref{thm:convergence_finite}).    
    \item \emph{Choice of the regularization parameter.} As a by-product of our mathematical analysis, we identify the scaling $\gamma^N = \frac{\gamma}{N^{2\beta-1}}$ for selecting the regularization parameter, which can vanish as $N\rightarrow\infty$.
    {We further characterize the role of the regularizer parameter~$\gamma$ as a spectral filter, balancing the learning rates of high-frequency components against the amount of regularization in the finite-width optimization problem (see Remark~\ref{rem:NNTKB_conv}).}
    \item \emph{Mathematical challenges.}
    {Unlike standard NTK analyses of gradient descent, which rely on explicit parameter updates and where the positive definiteness of the NTK matrix is immediate,}
    key challenges arise in the analysis of the Newton method from the implicit parameter update~\eqref{eq:regularizedNewtonStep} with a potentially indefinite Hessian matrix and the fact that the dimension of the resulting linear system tends to infinity as the NN width grows.
    {
    To overcome these challenges, we develop several new analytical techniques.
    In particular, by exploiting the block-diagonal plus low-rank structure of the regularized Newton system, we derive, via the push-through identity, an explicit expression for the Newton update, which in turn yields the NNTK.
    Furthermore, leveraging the identified scaling of the regularization parameter together with operator-norm estimates for the non-convex contribution of the Hessian, we prove that the regularized Hessian remains positive definite throughout training for sufficiently large NN widths, a key ingredient in establishing the convergence of the finite-width dynamics (see Remark~\ref{R:MathContribution}).}
    \item \emph{Broader implications.}
    Beyond providing the first convergence analysis of the regularized Newton method in the overparameterized regime,
    our analysis establishes the core mathematical tools that may serve as an analytical framework for future convergence analyses of approximate second-order optimization methods (see Section~\ref{sec:conclusions}).
\end{itemize}

Before discussing the outline of the manuscript, let us highlight a connection to trust-region methods.
Since trust-region optimization is mathematically equivalent to Newton's method for an appropriate regularization parameter, being the dual and primal perspectives,
our results with $\alpha=1$ provide theoretical guarantees for trust-region methods.
Provided $\gamma^N\Id_{N(d+2)}\!+\!\nabla_\theta^2L^N\!(\theta_k)\!\succ\!0$, the Newton update~$z^N_k$ in \eqref{eq:regularizedNewtonStep} is the solution {of the trust-region subproblem}
$\min_{z:\N{z}_2 \leq \Delta_k} \!L^N\!(\theta_k) \!+\! \nabla_\theta L^N\!(\theta_k)^\top z \!+\! \frac{1}{2} z^\top \nabla^2_\theta L^N\!(\theta_k) z$ for a trust-region radius $\Delta_k$ chosen such that $\Delta_k = \Nnormal{z^N_k}_2$.
{The equivalence follows from the KKT conditions of the trust-region subproblem, with}
the regularization parameter $\gamma^N$ acting as the Lagrange multiplier associated with the constraint $\frac12\N{z}_2^2 \leq \frac12\Delta_k^2$.
Consequently, for any such $\gamma^N > 0$, there exists a corresponding trust-region radius $\Delta_k$ yielding the same update $z^N_k$,
establishing the equivalence between the regularized Newton method and trust-region optimization.

\paragraph{Organization of the paper.}
After discussing related works in what follows,
we present in Section~\ref{sec:main} our convergence analysis for the regularized Newton method for training NNs in the overparameterized limit.
We first state in Section~\ref{subsec:assumptions} the mathematical assumptions that hold throughout this analysis,
and then derive in Section~\ref{subsec:reduction} the evolution of the NN function~\eqref{eq:NN} when the NN parameters are trained with Newton's method as in \eqref{eq:regularizedNewton}--\eqref{eq:regularizedNewtonStep}, in particular identifying the NNTK.
In Section~\ref{subsec:convergence}, we study the global convergence of the NN function in the infinite-width limit to the target data.
Section~\ref{subsec:NTK_limit} proves that the finite-width NN converges to the overparameterized limit as the number of hidden units $N \rightarrow \infty$. Section~\ref{subsec:convergence_final} presents a probabilistic convergence result for finite-width NNs and a finite number of training steps.
Proofs for these results can be found in Appendix~\ref{app:proofs_main}.
Numerical illustrations {in Sections~\ref{sec:main} and~\ref{sec:numerics}} support our theoretical findings.\footnote{The code for these experiments is provided in the GitHub repository \url{https://github.com/KonstantinRiedl/NNTK}.}
Section~\ref{sec:conclusions} concludes the paper.

\paragraph{Related works.}
Second-order optimization methods have been studied extensively in the classical optimization literature, see \cite{nocedal2006optimization,boyd2004convex,fletcher2001practical,conn2000trust,martens2016second} and the references therein. Our literature review below focuses on papers closely related to the objective of this paper,
which is the analysis of second-order methods for training NNs in the overparameterized regime.
While our paper is the first to rigorously study the convergence properties and establish assumptions under which global convergence holds  for the regularized Newton method when applied to shallow NNs,
similar questions have been previously explored for different second-order optimization methods.
Studying the regularized Newton method poses unique mathematical challenges inherent to its implicit parameter update with a potentially
indefinite Hessian matrix,
demanding the development of new techniques that did not exist before in the literature.
We discuss in Remark~\ref{R:MathContribution} how we overcome these difficulties in the paper. 

\textit{Gauss--Newton method.}
The authors of \cite{arbel2023rethinking} study the global convergence of the Gauss--Newton (GN) method when optimizing overparameterized one-hidden-layer NNs in the mean-field regime, i.e., in the case where $\beta=1$ in \eqref{eq:NN}.
The kernel regime, i.e., the case where $\beta=1/2$, is studied for the GN method in \cite{cai2019gram},
where the authors propose and study the Gram--Gauss--Newton method for regression problems with mean squared error. Specifically, they solve the kernel regression w.r.t.\@ the NTK at each step of the optimization.
\cite{cayci2024gauss} investigates the convergence behavior of the GN method in both the overparameterized and underparameterized regime leveraging the Levenberg-Marquardt dynamics and using tools from Riemannian optimization, respectively.
In particular, it is shown that the convergence is independent from the smallest eigenvalue of the NTK matrix.
A generalized GN method, which incorporates curvature estimates to resemble Newton's method more closely, is studied when optimizing one-hidden-layer NNs with explicit regularization in \cite{adeoye2024regularized}.
In \cite{ishikawa2023parameterization}, the authors identify a specific parameterization for second-order optimization enabling a stable feature learning
even if the network width increases significantly. They study the GN method as well as other  second-order optimization algorithms like K-FAC~\cite{martens2015optimizing} and Shampoo~\cite{gupta2018shampoo}.

\textit{Natural gradient descent.}
The authors of \cite{zhang2019fast} investigate the convergence of natural gradient descent (NGD)
on one-hidden-layer NNs with mean squared error loss in the kernel regime.
To ensure linear convergence, they identify stability conditions on the Jacobian, which hold for randomly initialized NNs with high probability.
The paper further covers an analogous convergence analysis for K-FAC.
In \cite{karakida2020understanding},
the convergence of vanilla NGD as well as NGD with approximated Fisher information matrix is investigated on deep NNs in the infinite-width limit with mean squared error loss via the neural tangent kernel.
The authors in particular observe that the training dynamics are independent of the NTK matrix.
\cite{bai2022generalized} investigates NGD from a more geometric perspective by proving that a generalized form of the NTK, called the generalized tangent kernel, unifies NGD and standard gradient descent.

Let us emphasize that the aforementioned works focus on methods which do not face the challenges of an implicit parameter update resulting from an infinite-dimensional linear system with a potentially indefinite Hessian. This distinction is central to the mathematical techniques developed in this paper.

\paragraph{Notation.}
The residual is denoted by $r^N_\theta(\x)=\y-f^N_\theta(\x)$ with $r^N_\theta(x^m)=y^m-f^N_\theta(x^m)$ being its $m$-th entry.
For a generic data sample~$x$, we introduce the variable $s^N_\theta(x) = \left(s^N_{\theta^1}(x),\dots,s^N_{\theta^N}(x)\right)\in\bbR^{1\times N(d+2)}$ where $s^N_{\theta^i}(x) = \left(\sigma(w^i \cdot x+\eta^i), c^i\sigma'(w^i \cdot x+\eta^i)x^\top,c^i\sigma'(w^i \cdot x+\eta^i)\right)\in\bbR^{1\times (d+2)}$.
The matrices $s^N_\theta\in\bbR^{M\times N(d+2)}$ and $s^N_{\theta^i}\in\bbR^{M\times (d+2)}$ are defined as $(s^N_\theta)_{m,:} =s^N_\theta(x^m)$ and $(s^N_{\theta^i})_{m,:}= s^N_{\theta^i}(x^m)$, respectively.
Furthermore,
we define the matrix $h_{\theta^i}^N(x)\in\bbR^{(d+2)\times(d+2)}$ as
\begin{equation}
    \label{eq:h}
    h_{\theta^i}^N(x)=
        \left(\begin{smallmatrix}0 & \sigma'(w^i \cdot x + \eta^i)x^\top & \sigma'(w^i \cdot x + \eta^i) \\
        \sigma'(w^i \cdot x + \eta^i)x & c^i\sigma''(w^i \cdot x + \eta^i)xx^\top & c^i\sigma''(w^i \cdot x + \eta^i)x \\
        \sigma'(w^i \cdot x + \eta^i) & c^i\sigma''(w^i \cdot x + \eta^i)x^\top & c^i\sigma''(w^i \cdot x + \eta^i) \end{smallmatrix}\right).
\end{equation}




\section{Convergence analysis of Newton's method in the infinite-width limit}
\label{sec:main}

This section presents our main theoretical contributions.
Under the assumptions given in Section~\ref{subsec:assumptions}, we prove that NNs trained with the regularized Newton method converge to a deterministic limit equation as the number of hidden units $N \rightarrow \infty$ and, as the number of training steps $k \rightarrow \infty$, the NN converges to a global minimizer. We analyze the convergence rate of the method in the infinite-width limit and compare it against the convergence rate of gradient descent.
 Section~\ref{subsec:reduction} derives the evolution of the NN function~\eqref{eq:NN} during training when its parameters are updated with Newton's method as in \eqref{eq:regularizedNewton}--\eqref{eq:regularizedNewtonStep}. The training dynamics are governed by a finite-width NNTK~\eqref{eq:NNTK}, which is random at initialization and varies during training.
However, as $N\rightarrow\infty$, the NNTK converges to a deterministic limit NNTK~\eqref{eq:NNTK*}, which remains constant during training (as we rigorously show in Section~\ref{subsec:NTK_limit}).
In this infinite-width limit, we establish global convergence to the target data $\y$ (i.e., a global minimizer with zero loss) as $k\rightarrow\infty$, a result we discuss in Section~\ref{subsec:convergence}.
Crucially, by analyzing the eigenvalues of the NNTK~\eqref{eq:NNTK*}, we demonstrate in Remark~\ref{rem:NNTKB_conv} that this convergence is uniform across the frequency spectrum.
The subsequent Section~\ref{subsec:NTK_limit} then proves the convergence in probability of the discrete training trajectory~$(f_{\theta_{k}}^N(\x))_{k=0,\dots,K}$ to the infinite-width limit with explicit rates as $N \rightarrow \infty$.
This analysis identifies the correct scaling formula for selecting the regularizer~$\gamma^N$, which we further show can vanish
as $N\rightarrow\infty$ (see Theorem~\ref{thm:NTK_limit}).
This yields a probabilistic convergence result for finite-width NNs~\eqref{eq:NN} trained with a finite number of Newton steps~\eqref{eq:regularizedNewton}--\eqref{eq:regularizedNewtonStep}  in Section~\ref{subsec:convergence_final}.

\subsection{Assumptions}
\label{subsec:assumptions}

The assumptions required for the convergence results in this paper are stated below, with assumptions related to the NN~\eqref{eq:NN} listed in Assumption~\ref{def:NN_assumptions} and the ones on the training data listed in Assumption~\ref{def:data_assumptions}.

\begin{assumption}[NN~$f_\theta^N$ in \eqref{eq:NN}]\label{def:NN_assumptions}
	The NN is such that
	\begin{enumerate}[label=A\arabic*,labelsep=7pt,leftmargin=26pt,topsep=-3pt,itemsep=0pt,parsep=0.5ex]
        \item\label{asm:NN_nonlinearity} the activation function~$\sigma$ 
        \begin{enumerate}[label=(\roman*),labelsep=7pt,leftmargin=26pt,topsep=0pt,itemsep=0pt,parsep=0ex]
            \item\label{asm:NN_sigma} is non-polynomial, bounded (i.e., $\abs{\sigma}\leq C_\sigma$), and $L_\sigma$-Lipschitz continuous,
            \item\label{asm:NN_sigma'} has a derivative $\sigma'$ which is bounded (i.e., $\abs{\sigma'} \leq C_{\sigma'}$) and $L_{\sigma'}$-Lipschitz continuous,
            \item\label{asm:NN_sigma''} has a second derivative $\sigma''$ which is bounded (i.e., $\abs{\sigma''} \leq C_{\sigma''}$),
        \end{enumerate}
        \item\label{asm:NN_mu0} the randomly initialized NN parameters~$\theta_0^i = (c^i_0,w^i_0,\eta^i_0)$ are i.i.d.\@ and drawn from a distribution $\mu_0\in\CP(\bbR\times\bbR^{d}\times\bbR)$
        which is such that
        \begin{enumerate}[label=(\roman*),labelsep=7pt,leftmargin=26pt,topsep=0pt,itemsep=0pt]
            \item $c^i_0$ is independent from $(w^i_0,\eta^i_0)$, \label{asm:NN_mu0i}
            \item the marginal distribution $\mu_{0,c}$ of $c^i_0$ is mean-zero and compactly supported,  \label{asm:NN_mu0ii}
            \item the marginal distribution $\mu_{0,(w,\eta)}$ of $(w^i_0,\eta^i_0)$ assigns positive probability to every set with positive Lebesgue measure. \label{asm:NN_mu0iv}
        \end{enumerate}
    \end{enumerate}
\end{assumption}
\vspace{-0.1em}
Assumption~\ref{asm:NN_nonlinearity} specifies the activation function $\sigma$, ensuring that it is discriminatory~\cite{cybenko1989approximation,hornik1991approximation,ito1996nonlinearity} and regular; $\tanh$ or sigmoid~$\sigma(z)=\frac{1}{1+e^{-z}}$ functions, for example, satisfy Assumption~\ref{asm:NN_nonlinearity}.
Assumption~\ref{asm:NN_mu0} characterizes the random parameter initialization~$\mu_0$ of the NN~\eqref{eq:NN}.

\begin{assumption}[Training dataset $\CD$]\label{def:data_assumptions}
	The training dataset $\CD = (x^m,y^m)_{m=1}^M \subset \bbR^d\times\bbR$ is such that
	\begin{enumerate}[label=B\arabic*,labelsep=7pt,leftmargin=26pt,topsep=-3pt,itemsep=0pt,parsep=0.5ex]
        \item\label{asm:data} the data samples $(x^m,y^m) \in \bbR^d\times\bbR$ are bounded (i.e., $\N{x^m}_2\leq C_x$ and $\abs{y^m}\leq C_y$),

        \item\label{asm:data_distinct}
        the data samples $x^m$ are in distinct directions, i.e., $x^m\not=0$ for all $m=1,\dots,M$, and the lines $\ell_{x^m}\!\coloneqq\!\{x\in\bbR^d\!:\!x\!=\!\xi x^m ,\xi\in\bbR\}$ meet at the origin only, i.e., $\ell_{x^m}\!\cap\!\ell_{x^n}\!=\!\{0\}$ for $m\!\not=\!n$.
    \end{enumerate}
\end{assumption}
\vspace{-0.1em}
Assumptions~\ref{asm:data} and \ref{asm:data_distinct} require that the data samples and labels are bounded, and in particular that no two samples~$x^m$ are collinear.
From a learning-theoretic perspective,
such redundancies in the dataset would lead to degeneracies, as collinear inputs carry the same directional information, which may result in ill-conditioning of optimization problems.
In particular, if there are repeated training samples, the NTK matrix degenerates~\cite[Remark~1]{carvalho2025positivity}.
Hence, this is a common requirement in the analysis of NNs ensuring that each data point identifies a unique hyperplane (activation direction), thereby preventing technical degeneracies during optimization,
see also \cite[Definition~19.5]{spiliopoulos2025mathematical}.

\subsection{Evolution of the NN function during training}
\label{subsec:reduction}

Let us first derive how the NN function~$f_{\theta}^N(\x)$ in \eqref{eq:NN} evolves when its NN parameters~$\theta$ are trained with the regularized Newton method \eqref{eq:regularizedNewton}--\eqref{eq:regularizedNewtonStep}. The NN output for a single data sample~$x^m$ can be written with a Taylor series expansion as 
\begin{equation}
    \label{eq:NTK_limit:evolution_fk+1_1}
\begin{split}
    f_{\theta_{k+1}}^N(x^m)
    &= f_{\theta_{k}}^N(x^m) + \nabla_\theta f^N_{\theta_k}(x^m) \cdot (\theta_{k+1}-\theta_k) + \frac{1}{2} (\theta_{k+1}-\theta_k)^\top H^N_{\tilde{\theta}_k^m}(x^m) (\theta_{k+1}-\theta_k)\\
    &= f_{\theta_{k}}^N(x^m) + \alpha \nabla_\theta f^N_{\theta_k}(x^m) \cdot z^N_k + R_k^N(x^m),
\end{split}
\end{equation}
where $\tilde{\theta}_k^m$ is a point on the line segment $[\theta_{k},\theta_{k+1}]$, $H^N_{\theta}(x)\in\bbR^{N(d+2)\times N(d+2)}$ the Hessian of the NN function~$f_{\theta}^N$ for a data sample~$x$ (see \eqref{eq:Hessian_fx}--\eqref{eq:Hessian_fxi}), and $R_k^N(x^m) = \frac{\alpha^2}{2} (z^N_k)^\top H^N_{\tilde{\theta}_k^m}(x^m)z^N_k$ the remainder.
Define $J_{\theta}^N\in\bbR^{M\times N(d+2)}$ as the Jacobian of the NN function~$f_{\theta}^N$  w.r.t.\@ the NN parameters $\theta$ for all data samples~$x^1,\dots,x^M$,
and the submatrix $J_{\theta^i}^N\in\bbR^{M\times (d+2)}$ as the Jacobian w.r.t.\@ the NN parameters $\theta^i$ of the $i$-th neuron (see \eqref{eq:Jacobian_f}--\eqref{eq:Jacobian_fxi}).
Then, denoting by $(z^N_k)^i\in\bbR^{d+2}$ the part of the Newton update~$z^N_k$ in \eqref{eq:regularizedNewtonStep} that updates the NN parameters $\theta^i_k$ of the $i$-th neuron in \eqref{eq:regularizedNewton},
\eqref{eq:NTK_limit:evolution_fk+1_1} can be written for all training data samples in vector-form as
\begin{equation}
    \label{eq:NTK_limit:evolution_fk+1_2}
    \begin{split}
        f_{\theta_{k+1}}^N(\x)
        &= f_{\theta_{k}}^N(\x) + \alpha J_{\theta_k}^N z^N_k + R^N_k(\x)
        = f_{\theta_{k}}^N(\x) + \alpha \sum_{i=1}^N J_{\theta^i_k}^N \left(z^N_k\right)^i + R^N_k(\x).
    \end{split}
\end{equation}

Let us now investigate the structure of the Newton update~$z_k^N$ given by \eqref{eq:regularizedNewtonStep}.
We first notice that the Hessian~$\nabla^2_\theta L^N(\theta)
    = G_\theta^N + S_\theta^N
    = \frac{1}{M}(J_\theta^N)^\top J_\theta^N - \frac{1}{M} \sum_{m=1}^M (y^m - f^N_\theta(x^m)) H_\theta^N(x^m)$ 
of the loss~$L^N(\theta)$ in \eqref{eq:loss} decomposes (see \eqref{eq:Hessian_loss})
into the sum of the low-rank Gram-Gauss--Newton matrix $G_\theta^N = \frac{1}{M}(J_\theta^N)^\top J_\theta^N$ and the non-convex contribution $S_\theta^N = - \frac{1}{M} \sum_{m=1}^M (y^m - f^N_\theta(x^m)) H_\theta^N(x^m)$,
which is block-diagonal with individual blocks corresponding to different neurons.
This enables us to leverage the Woodbury matrix identity (also known as the push-through identity) to derive
\begin{equation}
    \label{eq:zN_reformulation}
\begin{split}
    z_k^N
    &\!=\! \frac{1}{M}\!\left(D_{\theta_k}^N\right)^{-1}\!\! \left(J_{\theta_k}^N\right)^\top \!\!\zeta^N_k
    \;\;\text{ with }\;\;
    \zeta^N_k \!\coloneqq\! \left(\Id_M \!+\! \frac{1}{M} J_{\theta_k}^N\left(D_{\theta_k}^N\right)^{-1}\!\!\left(J_{\theta_k}^N\right)^\top \right)^{-1}\!\!\left(\y \!-\! f_{\theta_k}^N(\x)\right)\!,
\end{split}
\end{equation}
provided the positive definiteness of the block-diagonal matrix
\begin{equation}
    \label{eq:D}
    D_{\theta_k}^N=\gamma^N \Id_{N(d+2)} + S_{\theta_k}^N,
\end{equation}
which can be assured (see Lemma~\ref{lem:estimate_Dinv_operatornorm}) throughout training for sufficiently large $N$ by satisfying the condition $\gamma \geq C_\gamma \frac{1}{N^{1-\beta}}$, where $C_\gamma=C_\gamma(\alpha,\gamma,\sigma,\CD,\mu_0,C_{R,0},K)$ is of the form $C_\gamma=\sqrt{2}C_RC_S(\sigma,\CD,C_c)$ with $C_R=C_R(\alpha,\gamma,\sigma,\CD,\mu_0,C_{R,0},K)$ and $C_c=C_c(\alpha,\gamma,\sigma,\CD,\mu_0,C_{R,0},K)$ as in \eqref{eq:proof:thm:NTK_limit:aprioribounds} and $C_S$ as in Lemma~\ref{lem:estimate_S_operatornorm}, and $C_{R,0} > C_y$.
A proof of identity \eqref{eq:zN_reformulation} can be found in Appendix~\ref{sec:proof:eq:zN_reformulation}.
Since $D_{\theta_k}^N$ is block-diagonal with blocks $D_{\theta_k^i}^N$,
also its inverse~$(D_{\theta_k}^N)^{-1}$ is block-diagonal.
Exploiting this, \eqref{eq:zN_reformulation} reads neuron-wise $(z_k^N)^i= \frac{1}{M}(D_{\theta^i_k}^N)^{-1} (J_{\theta^i_k}^N)^\top \zeta^N_k$ for all $i=1,\dots,N$.\\
Recalling the choice of the regularizer $\gamma^N=\frac{\gamma}{N^{2\beta-1}}$ and the formulas $J^N_{\theta^i} = \frac{1}{N^\beta}s^N_{\theta^i}$ and $H^N_{\theta^i}(x) = \frac{1}{N^\beta}h^N_{\theta^i}(x)$,
we can rewrite \eqref{eq:NTK_limit:evolution_fk+1_2} as
\begin{equation}
    \label{eq:NTK_limit:evolution_fk+1_3}
    \begin{split}
        f_{\theta_{k+1}}^N(\x)
        &= f_{\theta_{k}}^N(\x) + \alpha \frac{1}{MN}\sum_{i=1}^N s^N_{\theta_k^i} \left(d_{\theta^i_k}^N\right)^{-1} \left(s^N_{\theta_k^i}\right)^\top \zeta_k^N + R^N_k(\x),
    \end{split}
\end{equation}
where $d_{\theta^i}^N = \gamma \Id_{d+2} - \frac{1}{N^{1-\beta}} \frac{1}{M} \sum_{m=1}^M (y^m - f_{\theta}^N(x^m)) h_{\theta^i}^N(x^m)$ and
\begin{equation}
        \label{eq:NTK_limit:evolution_fk+1_3aux}
    \begin{split}
        \zeta_k^N
        &= \left(\Id_M + \frac{1}{MN}\sum_{i=1}^N s^N_{\theta_k^i} \left(d_{\theta^i_k}^N\right)^{-1} \left(s^N_{\theta_k^i}\right)^\top \right)^{-1}\left(\y - f_{\theta_k}^N(\x)\right)\!.
    \end{split}
\end{equation}
Defining the \emph{NNTK} $\CB^N_k\in\bbR^{M\times M}$ as
\begin{equation}
    \label{eq:NNTK}
    \CB^N_k
    = \frac{1}{MN}\sum_{i=1}^N s^N_{\theta_k^i} \left(d_{\theta^i_k}^N\right)^{-1} \left(s^N_{\theta_k^i}\right)^\top\left(\Id_M + \frac{1}{MN}\sum_{i=1}^N s^N_{\theta_k^i} \left(d_{\theta^i_k}^N\right)^{-1} \left(s^N_{\theta_k^i}\right)^\top \right)^{-1},
\end{equation}
the equations \eqref{eq:NTK_limit:evolution_fk+1_3}--\eqref{eq:NTK_limit:evolution_fk+1_3aux} can be re-written as
\begin{equation}
        \label{eq:NNfunction_evolution_prelimit}
    \begin{split}
        f_{\theta_{k+1}}^N(\x)
        &= f_{\theta_{k}}^N(\x) + \alpha \CB^N_k\left(\y - f_{\theta_{k}}^N(\x)\right)  + R^N_k(\x).
    \end{split}
\end{equation}
For a detailed derivation of the above steps, we refer the reader to \eqref{eq:proof:thm:NTK_limit:evolution_fk+1_10}--\eqref{eq:proof:thm:NTK_limit:evolution_fk+1_41rewr} in Appendix~\ref{sec:app:proof_NTK_limit}.

The NNTK~\eqref{eq:NNTK} is random at initialization and varies during training.
Unlike the standard NTK (for gradient descent)~\cite{jacot2018neural,chizat2018global,sirignano2019scaling},
it involves the implicit solution of a linear system of equations whose dimensionality tends to infinity as the NN width $N\rightarrow\infty$.
This substantially complicates deriving the limit equation for the training dynamics \eqref{eq:NNfunction_evolution_prelimit} in the overparameterized limit and showing the convergence of \eqref{eq:NNfunction_evolution_prelimit} to it.
In Section~\ref{subsec:NTK_limit}, however,
we will rigorously prove that \eqref{eq:NNfunction_evolution_prelimit} indeed converges to a deterministic limit equation as $N \rightarrow \infty$.
We will further observe that, similarly to the standard NTK (for gradient descent),
the NNTK converges to the deterministic limit NNTK \eqref{eq:NNTK*}.

\begin{remark}[{Mathematical core techniques}]
\label{R:MathContribution}
    At the core, our mathematical arguments rely on ensuring that the regularized Hessian~$\gamma^N\Id_{N(d+2)}+\nabla^2_\theta L^N(\theta_k)$ remains positive definite during training,
    which is, due to $G_{\theta_k}^N$ being positive semi-definite, guaranteed by the block-diagonal matrix $D_{\theta_k}^N$ being positive definite.
    Using the scaling of $\gamma^N$,
    we can prove that this is the case for NN widths large enough such that $\gamma \geq C_\gamma \frac{1}{N^{1-\beta}}$ (see Lemma~\ref{lem:estimate_Dinv_operatornorm}).
    Using the closed-form formula \eqref{eq:zN_reformulation} for the Newton step $z_k^N$ in 
    \eqref{eq:regularizedNewton},
    we can then prove that the increment $z_k^N$ vanishes neuron-wise at rate $\CO(\frac{1}{N^{1-\beta}})$ (see Lemma~\ref{prop:NTK_limit_bounds_z}).
    This then enables us to show that $\CB^N_k$ converges to a deterministic limit as $N\rightarrow\infty$, 
    and in particular that \eqref{eq:NNfunction_evolution_prelimit} converges to a deterministic limit equation of the same form (see Theorem~\ref{thm:NTK_limit}).
\end{remark}

\subsection{Global convergence in the overparameterized NN limit}
\label{subsec:convergence}

In the infinite-width limit,
the NNTK~$\CB^N_k$ becomes,
as we rigorously show in Section~\ref{subsec:NTK_limit},
deterministic and converges to the
\textit{limit NNTK} $\CB^*_0\in\bbR^{M\times M}$ defined as
\begin{equation}
    \label{eq:NNTK*}
    \CB^*_0
    =\frac{1}{\gamma}B^*_0\left(\Id_M + \frac{1}{\gamma}B^*_0 \right)^{-1}
    =B^*_0\left(\gamma \Id_M + B^*_0 \right)^{-1},
\end{equation}
where $B^*_0\in\bbR^{M\times M}$ denotes the limit standard NTK (for gradient descent) with elements
\begin{equation}
    \label{eq:standardNTK}
\begin{split}
    \left(B^*_0\right)_{mn}
    &= \frac{1}{M} \int \sigma(w \cdot x^m+\eta)\sigma(w \cdot x^n+\eta)\\
    &\qquad\quad\, + c^2\sigma'(w \cdot x^m+\eta)\sigma'(w \cdot x^n+\eta) \left(x^m\cdot x^n+1\right) d\mu_0(c,w,\eta),
\end{split}
\end{equation}
which is positive definite due to Assumptions~\ref{def:NN_assumptions} and \ref{def:data_assumptions} \cite{jacot2018neural,sirignano2019scaling,spiliopoulos2025mathematical} (see Lemma~\ref{lem:standardNTKB:spectrum}),
thereby also ensuring the positive definiteness of the NNTK (see Lemma~\ref{lem:secondorderNTKB:spectrum}).
We observe that the limit NNTK $\CB^*_0$ in \eqref{eq:NNTK*} remains constant during training,
resembling the lazy training phenomenon \cite{jacot2018neural,chizat2018global,sirignano2019scaling} observed for certain scalings in \eqref{eq:NN} when trained with gradient descent.
However, due to the use of a second-order optimization algorithm, the structure of the NNTK differs, requiring the development of new mathematical methods to analyze the convergence to the overparameterized limit.

As we will show in Section~\ref{subsec:NTK_limit},
the discrete training trajectory~$(f_{\theta_{k}}^N(\x))_{k=0,\dots,K}$ of \eqref{eq:NNfunction_evolution_prelimit}
converges with high probability 
to the solution~$(f_{k}^*(\x))_{k=0,\dots,K}$ of the deterministic limit equation
\begin{equation}
    \label{eq:NNfunction_evolution_limit}
\begin{split}
    f_{k+1}^*(\x)
    = f_{k}^*(\x) + \alpha \CB^*_0\left(\y - f_k^*(\x)\right)
    \quad\text{ with }\quad
    f_{0}^*(\x) = 0.
\end{split}
\end{equation}
Specifically, Theorem~\ref{thm:NTK_limit} proves that, with high probability, $\max_{k=0,\dots,K} \N{f_{\theta_{k}}^N(\x) \!-\! f_{k}^*(\x)}_2 \rightarrow0$ as $N\rightarrow\infty$.
In the remainder of this section,
we hence investigate the convergence to the target data~$\y$ in the infinite-width limit
as the number of Newton training steps~$k \rightarrow \infty$.
For notational convenience,
\begin{equation}
    \label{eq:NNfunction_evolution_split1}
\begin{split}
    f_{k+1}^*(\x)
    = f_{k}^*(\x) + \alpha \frac{1}{\gamma}B^*_0 \zeta_k^*
    \quad\text{ with }\quad
    \zeta_k^*
    = \left(\Id_M + \frac{1}{\gamma}B^*_0 \right)^{-1}\left(\y - f_k^*(\x)\right)\!.
\end{split}
\end{equation}
In this large neuron limit, we show in Theorem~\ref{thm:convergence} below that, as the number of training steps~$k$ becomes large, the dynamics~\eqref{eq:NNfunction_evolution_limit} converges exponentially fast to zero loss, i.e., a global minimizer of the loss
\begin{equation}
    \label{eq:loss_limit}
    L^*
    = \frac{1}{2M}\N{\y-f^*(\x)}_2^2.
\end{equation} 
The proof is given in Appendix~\ref{sec:app:thm:convergence}.
This shows convergence of the NN function to the target data $\y$ (i.e., a global minimizer) as $k\rightarrow\infty$.
Notably, unlike in the convergence analysis of gradient descent,
the convergence rate is not affected by a poorly-conditioned limit NTK,
but instead is uniform across all frequencies provided a sufficiently small regularizer parameter~$\gamma$ (see Remark~\ref{rem:NNTKB_conv}).

\begin{theorem}[Global convergence to zero loss in the infinite-width limit]
    \label{thm:convergence}
    Let the learning rate $\alpha\in\left(0,2/\!\left(\lambda_{\max}(\CB^*_0)+\lambda_{\min}(\CB^*_0)\right)\right)$.
    Then, for all training steps $k \ge 0$,
    \begin{equation}
        \label{eq:thm:convergence}
        \N{f_k^*(\x) - \y}_2 \leq \left(1 - \alpha \lambda_{\min}(\CB_0^*)\right)^k \N{\y}_2\!.
    \end{equation}
    In particular, $f_k^*(\x) \rightarrow \y$ linearly with rate $1 - \alpha \lambda_{\min}(\CB_0^*)$ as $k \rightarrow \infty$. The loss~$L^*_k$ defined in \eqref{eq:loss_limit} decays exponentially fast (hence, converges linearly) with rate $\left(1 - \alpha \lambda_{\min}(\CB_0^*)\right)^2$.
\end{theorem}

\begin{remark}[{Regularization as a spectral filter}]
    \label{rem:NNTKB_conv}
    Theorem~\ref{thm:convergence} proves the linear convergence (i.e., exponential decay of the error) of the Newton method in the infinite-width limit.
    As can be seen from \eqref{eq:thm:convergence},
    the rate is determined by the smallest eigenvalue~$\lambda_{\min}(\CB^*_0)$ of the NNTK $\CB^*_0$ \eqref{eq:NNTK*}.
    This is analogous to the convergence behavior in case of gradient descent, where the rate is determined by the smallest eigenvalue~$\lambda_{\min}(B^*_0)$ of the standard NTK. In contrast to gradient descent where high-frequency patterns associated with small eigenvalues of the NTK~$B^*_0$ are learned slowly while low-frequency patterns (corresponding to large eigenvalues) are easy to learn, the convergence rate achieved for the Newton's method does not suffer from such spectral bias, as visualized in Figure~\ref{fig:Remark4_Eigenvalues}.
    The eigenvalues of the limit NNTK are $\lambda_{m}(\CB^*_0) = \frac{\lambda_{m}(B^*_0)}{\gamma+\lambda_{m}(B^*_0)}$, where $\lambda_{m}(B^*_0)$ denotes the $m$-th eigenvalue of the standard limit NTK $B^*_0$ for gradient descent (see Lemma~\ref{lem:secondorderNTKB:spectrum}).
    Typically, lower bounds on the smallest eigenvalue of $B^*_0$ depend on a measure of distance between data points \cite{karhadkar2024bounds}, e.g., $\min_{m\not=n} \inf_{\xi\in\bbR}\left\{\N{x^m-\xi x^n}_2\right\}$, which quantifies Assumption~\ref{asm:data_distinct}.
    As the number of data samples $M$ becomes larger, the smallest eigenvalue $\lambda_{\min}(B^*_0)$ becomes small \cite{cao2019towards,bietti2019inductive}.
    This leads to poor conditioning and slow convergence of gradient descent. The eigenvalues of the NNTK $\CB^*_0$, on the other hand, remain of order one if we let the regularization parameter $\gamma\rightarrow0$ suitably as $M\rightarrow\infty$. In the infinite-width limit, one can in principle even choose $\gamma=0$, since $B^*_0$ is positive definite, which would lead to all eigenvalues of $\CB^*_0$ being one.
    We discuss this case in more detail in Remark~\ref{rem:conv_one_step}.
    For finite-width NNs,
    $\gamma>0$ can be chosen arbitrarily small, however this comes at the cost of requiring the number of neurons $N$ being sufficiently large in order for \eqref{eq:thm:asm:Nlargeenough} to be satisfied, which ensures
    the solvability of the Newton update~\eqref{eq:regularizedNewtonStep}, {see also Remark~\ref{rem:influenceCgamma}}.
    This inherent normalization of the spectrum results in favorable and reliable convergence of Newton's method for target functions with high-frequency components.
    {Consequently, the regularizer parameter~$\gamma$ can be interpreted as a spectral filter. Smaller values accelerate the learning of high-frequency components by normalizing the spectrum of the NNTK, whereas larger values suppress these components and retain a stronger spectral bias.
    The preferred choice of~$\gamma$ depends on the application.
    For learning problems in which high-frequency components correspond predominantly to noise, larger $\gamma$ may improve robustness against overfitting.
    Conversely, in applications such as scientific machine learning where high-frequency features form an essential part of the solution, smaller values of~$\gamma$ can be advantageous.
    In practice, this parameter is often adapted dynamically throughout training.}
\end{remark}

    \begin{figure}[ht]
    \centering
    \hphantom{--}%
    \begin{subfigure}{0.43\textwidth}
        \centering
        \includegraphics[width=\linewidth,trim=0ex 3ex 0ex 3ex]{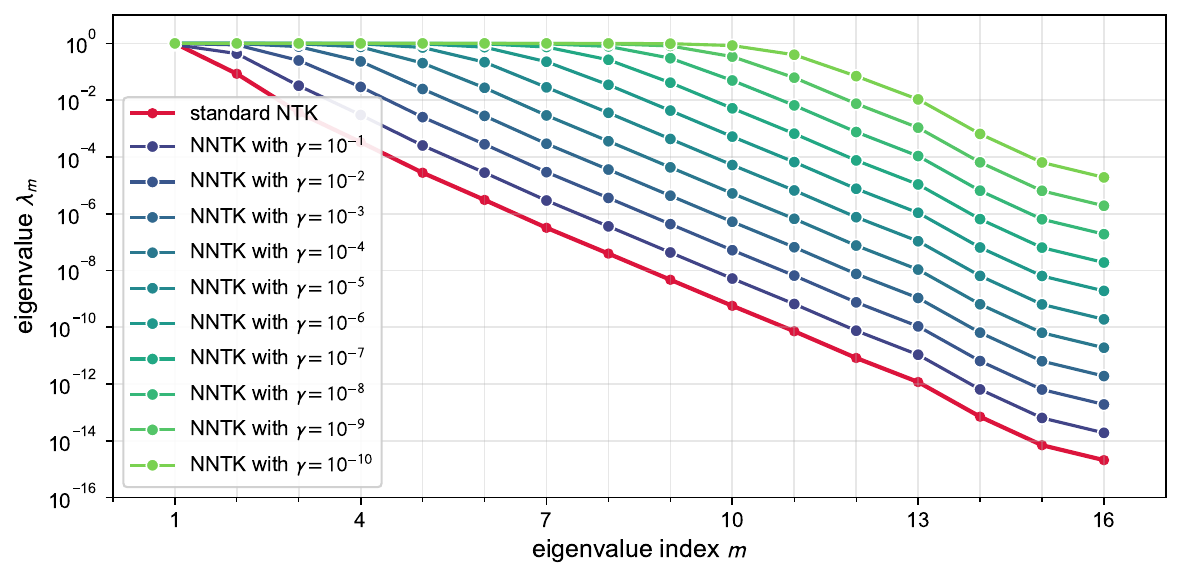}
        \caption{\footnotesize Eigenvalues of standard NTK $B^*_0$ and NNTK~$\CB^*_0$ for different values of the regularizer parameter~$\gamma$.}
        \label{fig:Remark4_Eigenvalues}
    \end{subfigure}
    \hfill
    \begin{subfigure}{0.43\textwidth}
        \centering
        \includegraphics[width=\linewidth,trim=0ex 3ex 0ex 3ex]{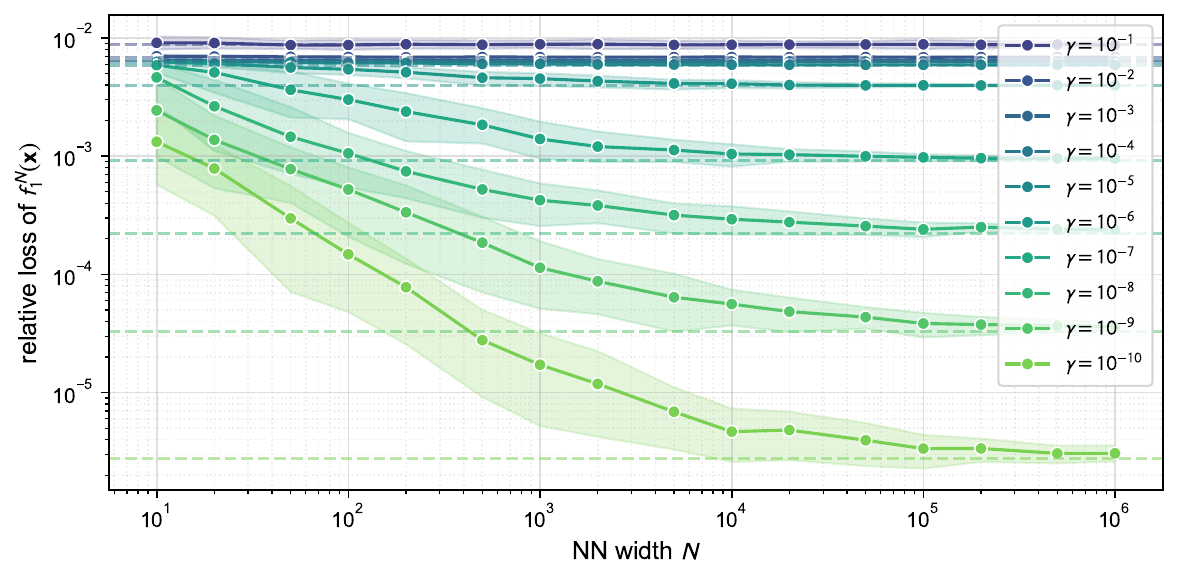}
        \caption{\footnotesize Relative loss $\nicefrac{L^N(\theta_1)}{\|\mathbf{y}\|^2}$ of NNs of width $N$ after a single Newton step with regularizer parameter~$\gamma$.}
        \label{fig:Remark5_OneStepLoss}
    \end{subfigure}%
    \hphantom{--}%
    \caption{Empirical validation of Remarks~\ref{rem:NNTKB_conv} and~\ref{rem:conv_one_step} for a shallow $\tanh$ NN~\eqref{eq:NN} with $\beta=0.52$ on $M=16$ samples of $y=2x + 0.4\sin(5\pi x)$. (Results in (b) are averaged across $100$ runs.)}
    \label{fig:1}
    \end{figure}

\begin{remark}[One-step convergence of the unregularized Newton method in the infinite-width limit]
    \label{rem:conv_one_step}
    In the limit $N\rightarrow\infty$,
    the optimization problem becomes quadratic due to the NN behaving like a linear model around its initialization as can be seen from \eqref{eq:NNfunction_evolution_limit}.
    We therefore expect the learning problem to be solvable with Newton's method without regularization in a single step.
    Formally, for $\gamma=0$ and with step size $\alpha=1$, we observe in \eqref{eq:NNTK*} that $\CB^*_0=\Id_M$,
    confirming that only one Newton step is necessary, since
    $f_{1}^*(\x)
    = f_{0}^*(\x) + \left(\y - f_0^*(\x)\right)
    = \y$ in \eqref{eq:NNfunction_evolution_limit}.
    Figure~\ref{fig:Remark5_OneStepLoss} confirms this empirically.
\end{remark}

\subsection{Convergence of the finite-width NN to its overparameterized limit as \texorpdfstring{$N \rightarrow \infty$}{the NN width tends to infinity}}
\label{subsec:NTK_limit}

The previous section analyzed the convergence properties of the training dynamics \eqref{eq:NNfunction_evolution_limit} of the NN function~$f^*(\x)$ in the infinite-width limit. We will now rigorously prove that the training trajectory~$(f_{\theta_{k}}^N(\x))_{k=0,\dots,K}$ of the finite-width NN function~$f^N_\theta(\x)$ with $N$ hidden units trained with Newton's method, as described in \eqref{eq:NNfunction_evolution_prelimit},
converges in probability, as the number of neurons $N$ tends to infinity, to the trajectory~$(f_k^*(\x))_{k=0,\dots,K}$ of the infinite-width limit.
We present these convergence results in Theorem \ref{thm:NTK_limit} and Corollary \ref{cor:NTK_limit} below, whose proofs are deferred to Appendices~\ref{sec:app:proof_NTK_limit} and \ref{sec:app:proof_cor_NTK_limit}.

\begin{theorem}
    \label{thm:NTK_limit}
    Let $\delta\in(0,1/2]$ and $K<\infty$ be a given number of training steps.
    Fix $C_{R,0}>C_y$.
    Assume that the number of neurons~$N$ is large enough such that 
    \begin{equation}
        \label{eq:thm:asm:Nlargeenough}
        N^{2\beta-1}\geq \frac{C(\sigma,\mu_0)}{(C_{R,0}-C_y)^2}\frac1\delta
        \quad\text{ and }\quad
        \gamma \geq C_\gamma \frac{1}{N^{1-\beta}}.
    \end{equation}
    Then, we have that
    \begin{equation*}
        \bbP\left(\max_{k=0,\dots,K} \N{f_{\theta_{k}}^N(\x) - f_{k}^*(\x)}_2 \leq \frac{C\sqrt{M}}{\delta}\left(\frac{1}{N^{\beta-1/2}} + \frac{1}{N^{1-\beta}}+\frac{\log{(2M)}}{N^{1/2}}\right)\right)
        \geq 1 \!-\! 2\delta
    \end{equation*}
    for a constant $C=C(\alpha,\gamma,\sigma,\CD,\mu_0,C_c,C_R,C_z,K)$ with $C_z=C_z(\alpha,\gamma,\sigma,\CD,\mu_0,C_{R,0},K)$ defined as in \eqref{eq:proof:thm:NTK_limit:aprioribounds}.
\end{theorem}

Theorem \ref{thm:NTK_limit} provides a high-probability bound with a quantitative rate on the deviation of the finite-width NN training trajectory from its infinite-width limit.
Since $\delta$ can be chosen arbitrarily small,
this result directly implies convergence in probability, as formalized in the subsequent corollary.
\begin{corollary}
    \label{cor:NTK_limit}
    Under the assumptions of Theorem~\ref{thm:NTK_limit},
    the training trajectory $(f_{\theta_{k}}^N(\x))_{k=0,\dots,K}$ converges in probability  to $(f_{k}^*(\x))_{k=0,\dots,K}$ as $N\rightarrow\infty$.
\end{corollary}

A key step in the proof of Theorem~\ref{thm:NTK_limit} is Lemma~\ref{prop:NTK_limit_bounds_z} below, which establishes that, with high probability, the Newton updates $z^N_k$ in \eqref{eq:regularizedNewtonStep} vanish neuron-wise as $N\rightarrow\infty$.
This implies that the individual NN parameters when trained with Newton's method as in \eqref{eq:regularizedNewton}--\eqref{eq:regularizedNewtonStep} remain in a neighborhood of their initial conditions, with the radius of the neighborhood converging to zero as $N \rightarrow \infty$, leading to a linearization of the NN training dynamics around the initialization.
Figure~\ref{fig:Lemma9_Parameter_Displacement} confirms this empirically.
\begin{figure}[ht]
    \centering
    \hphantom{--}%
    \begin{subfigure}{0.43\textwidth}
        \centering
        \includegraphics[width=\linewidth,trim=0ex 3ex 0ex 3ex]{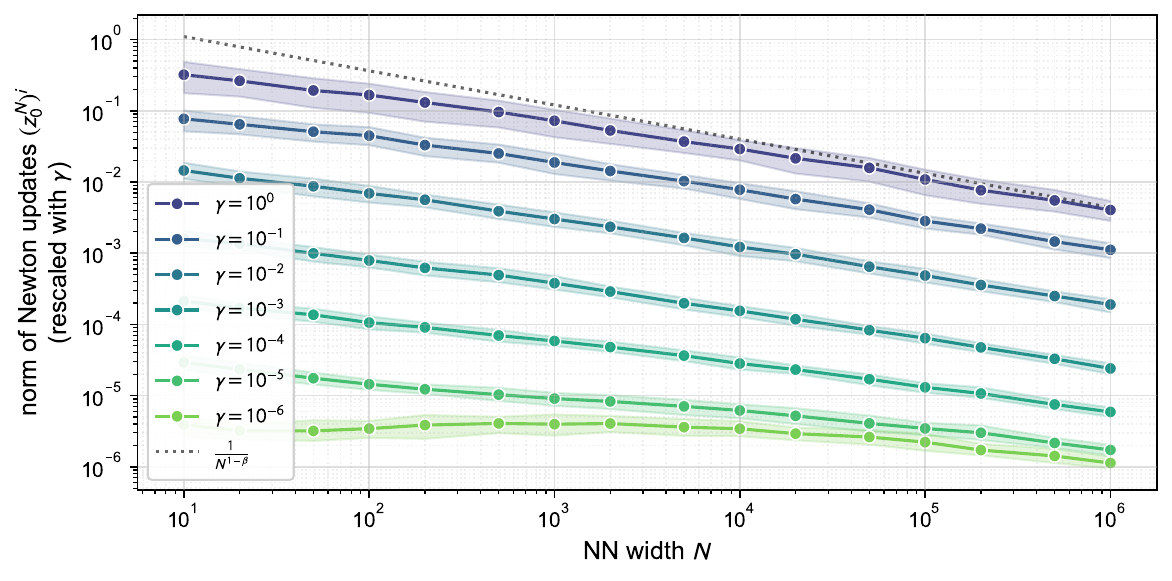}
        \caption{\footnotesize Norm of Newton updates~$(z^N_0)^i$ for different $N$ and $\gamma$. Note, the smaller $\gamma$ is, the later (w.r.t.\@ $N$) the expected rate $\nicefrac{1}{N^{1-\beta}}$ is attained due to condition~\eqref{eq:thm:asm:Nlargeenough}.}
        \label{fig:Lemma9_Parameter_Displacement}
    \end{subfigure}%
    \hfill%
    \begin{subfigure}{0.43\textwidth}
        \centering
        \includegraphics[width=\linewidth,trim=0ex 3ex 0ex 3ex]{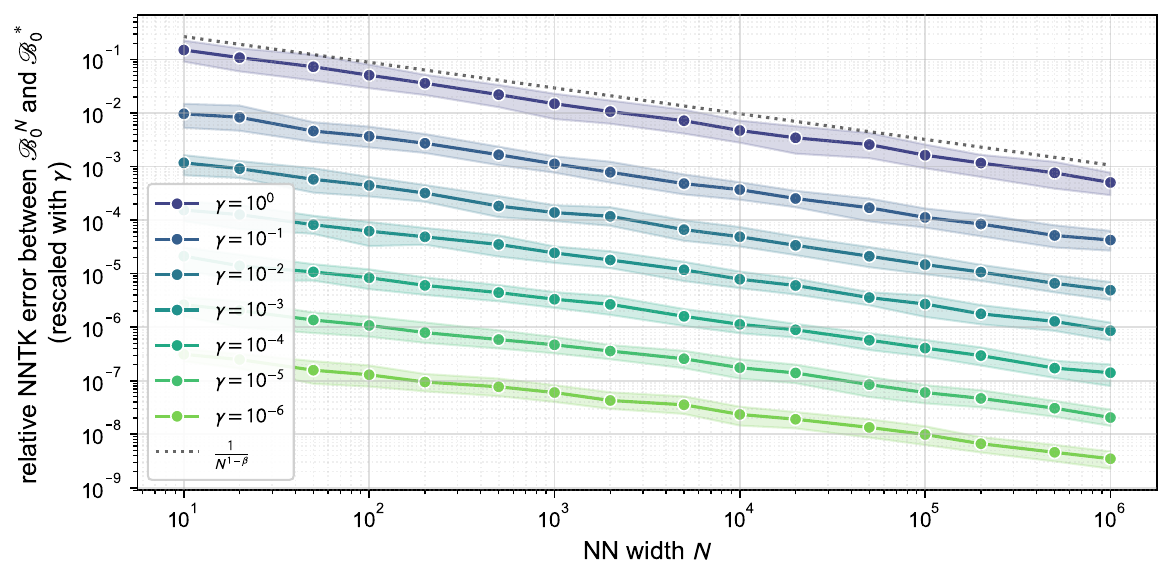}
        \caption{\footnotesize Distance between the finite-width NNTK~$\CB^N_0$ and its infinite-width limit~$\CB^*_0$ for different widths $N$ and regularizer parameters~$\gamma$.}
        \label{fig:Theorem6_Kernel_Convergence}
    \end{subfigure}%
    \hphantom{--}%
    \caption{Empirical validation of Lemma~\ref{prop:NTK_limit_bounds_z} and Theorem~\ref{thm:NTK_limit} for a shallow $\tanh$ NN~\eqref{eq:NN} with $\beta=0.52$ on $M=16$ samples of $y=2x + 0.4\sin(20\pi x)$. (Results are averaged across $100$ runs.)}
\label{fig:2}
\end{figure}
This resembles the overparameterized training phenomenon \cite{jacot2018neural,chizat2018global,sirignano2019scaling} observed for certain scalings in \eqref{eq:NN} when trained with gradient descent. Due to the large number of degrees of freedom in the overparameterized regime, individual parameters are required to move less and less from their initialization to achieve a given magnitude change in the NN output as $N \rightarrow\infty$. The result below further proves, for a finite number of training steps $k=0,\dots,K$, \emph{a priori} bounds for the Newton updates~\eqref{eq:regularizedNewtonStep}, the NN function~\eqref{eq:NN}, and the loss~\eqref{eq:loss}, which are uniform in the number of neurons~$N$ and hold with high probability. These bounds are essential for proving the main theorems.
\begin{lemma}
    \label{prop:NTK_limit_bounds_z}
    Let $\delta\in(0,1/2]$ and $K<\infty$ be a given number of training steps.
    Assume that the number of neurons~$N$ is large enough such that \eqref{eq:thm:asm:Nlargeenough} is satisfied.
    Then it holds that
    \begin{equation*}
    \begin{split}
        &\bbP\bigg(\!\max_{\substack{i=1,\dots,N \\ k=0,\dots,K}}\!\abs{c^i_k}\!\leq\!C_c, \!\;\max_{k=0,\dots,K}\!\Nbig{\y\!-\!f_{\theta_{k}}^N(\x)}_2 \!\leq\! \sqrt{M}C_R, \!\; \max_{\substack{i=1,\dots,N \\ k=0,\dots,K}}\!\N{(z^N_k)^i}_2 \!\leq\! \frac{C_z}{\gamma N^{1-\beta}}\!\bigg) \geq 1\!-\! \delta.
    \end{split}
    \end{equation*}
\end{lemma}

The proof of Lemma \ref{prop:NTK_limit_bounds_z} is presented in Appendix~\ref{sec:app:NTK_limit_bounds_z}.
Using this result, we are able to develop a probabilistic bound on $\N{\CB^N_k-\CB^*_0}_{\mathrm{op}}$ as demonstrated in Lemma~\ref{lem:differenceNNTK_finite_infinite} and confirmed empirically in Figure~\ref{fig:Theorem6_Kernel_Convergence}, which facilitates the proof of Theorem~\ref{thm:NTK_limit}.


\subsection{Convergence in probability of Newton's method for NNs as \texorpdfstring{$k, N \rightarrow \infty$}{the number of training steps and the NN width tend to infinity}}
\label{subsec:convergence_final}

Combining the two main results of the preceding sections, Theorems~\ref{thm:convergence} and \ref{thm:NTK_limit},
leads to the following result regarding the convergence in probability of the finite-width NN function output~$f_\theta^N(\x)$ in \eqref{eq:NN} to the target data $\y$ (i.e., a global minimizer).
Its proof is given in Appendix~\ref{sec:app:convergence_finite}.

\begin{theorem}
    \label{thm:convergence_finite}
    Let $\delta\in(0,1/2]$ and $\alpha\in\left(0,2/\!\left(\lambda_{\max}(\CB^*_0)+\lambda_{\min}(\CB^*_0)\right)\right)$.
    Assume that the number of neurons~$N$ is large enough such that \eqref{eq:thm:asm:Nlargeenough} is satisfied.
    Then it holds that
    \begin{equation*}
    \begin{split}
        &\bbP\left(\N{f_{\theta_{K}}^N(\x) \!-\! \y}_2 \!\leq\! \frac{C\sqrt{M}}{\delta}\!\left(\frac{1}{N^{\beta-1/2}} \!+\! \frac{1}{N^{1-\beta}} \!+\! \frac{\log{(2M)}}{N^{1/2}}\right) \!+\! \left(1\!-\!\alpha \lambda_{\min}(\CB^*_0)\right)^K\!\N{\y}_2 \right) \!\geq\! 1 \!-\! 2\delta
    \end{split}
    \end{equation*}
    for a constant $C=C(\alpha,\gamma,\sigma,\CD,\mu_0,C_c,C_R,C_z,K)$.
\end{theorem}

This implies convergence in probability in the following sense,
as we demonstrate in Appendix~\ref{sec:app:proof_cor_convergence_finite}.

\begin{corollary}
    \label{cor:convergence_finite}    
    Under the assumptions of Theorem~\ref{thm:convergence_finite}, there exists a function $g: \bbN \to \bbN$ dictating the required growth of the network width, such that for any sequence of widths $(N_K)_{K=1}^\infty$ satisfying $N_K \geq g(K)$ for all $K$, the trained NN function output $f_{\theta_{K}}^{N_K}(\x)$ converges in probability to the target data $\y$ (i.e., a global minimizer) as $K \rightarrow \infty$.
\end{corollary}

\begin{remark}
    \label{rem:influenceCgamma}
    
    Theorem~\ref{thm:convergence_finite} combines the approximation error due to the finite network width~$N$ with the optimization error arising from a finite number~$K$ of Newton steps.
    The first term quantifies the discrepancy between the finite-width network and its infinite-width limit (Theorem~\ref{thm:NTK_limit}),
    whereas the second term describes the exponential convergence of the infinite-width dynamics (Theorem~\ref{thm:convergence}).
    The regularizer parameter~$\gamma$ governs the trade-off between these two aspects. While smaller values of~$\gamma$ improve the conditioning of the limiting NNTK (Remark~\ref{rem:NNTKB_conv}), thereby accelerating convergence in the overparameterized limit, they reduce the amount of regularization in the finite-width optimization problem, rendering the latter more challenging by bringing it closer to the original non-convex problem.
    This is reflected in Theorem~\ref{thm:convergence_finite} through the constant~$C$, which becomes larger for smaller values of~$\gamma$.
    Consequently, as also captured by the second condition in \eqref{eq:thm:asm:Nlargeenough}, wider NNs are required to recover the same approximation error as the regularization decreases, highlighting the role of overparameterization as a means of convexifying the optimization problem.
\end{remark}

\section{Numerical experiments}
\label{sec:numerics}

To complement our theoretical analysis, we compare in Figure~\ref{fig:Comparison_GD_Newton} the empirical convergence behavior of gradient descent, Adam, the regularized Gauss--Newton method, and the regularized Newton method on a synthetic regression problem with high frequency components, see Appendix~\ref{app:experiments} for experimental details.
We observe that the regularized Newton method converges considerably faster than first-order methods, both in terms of iteration count and elapsed compute time.
A frequency-resolved comparison in Figure~\ref{fig:Comparison_GD_Newton_FrequenciesLearned} in Appendix~\ref{app:experiments} further demonstrates that Newton's method learns the different frequency components more uniformly than gradient descent, supporting the reduced spectral bias predicted by our theory.

\begin{figure}[ht]
    \centering
    \hphantom{--}%
    \begin{subfigure}{0.43\textwidth}
        \centering
        \includegraphics[width=\linewidth,trim=0ex 3ex 0ex 3ex]{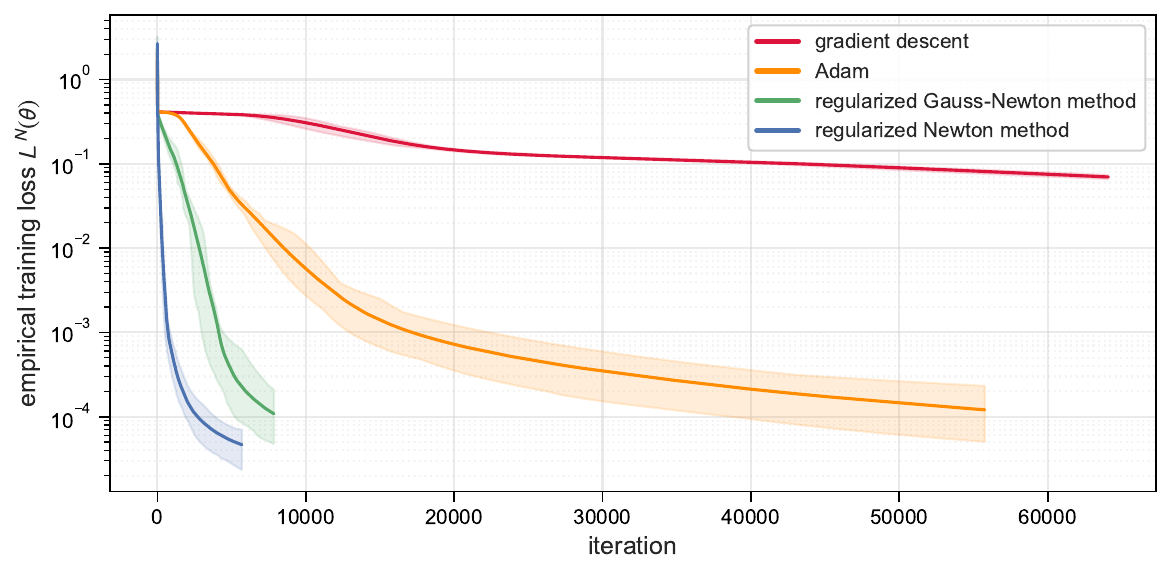}
        \caption{\footnotesize Decay of the loss $L^N(\theta)$ as a function of the iteration count.}
        \label{fig:Comparison_GD_Newton_Iterations}
    \end{subfigure}%
    \hfill%
    \begin{subfigure}{0.43\textwidth}
        \centering
        \includegraphics[width=\linewidth,trim=0ex 3ex 0ex 3ex]{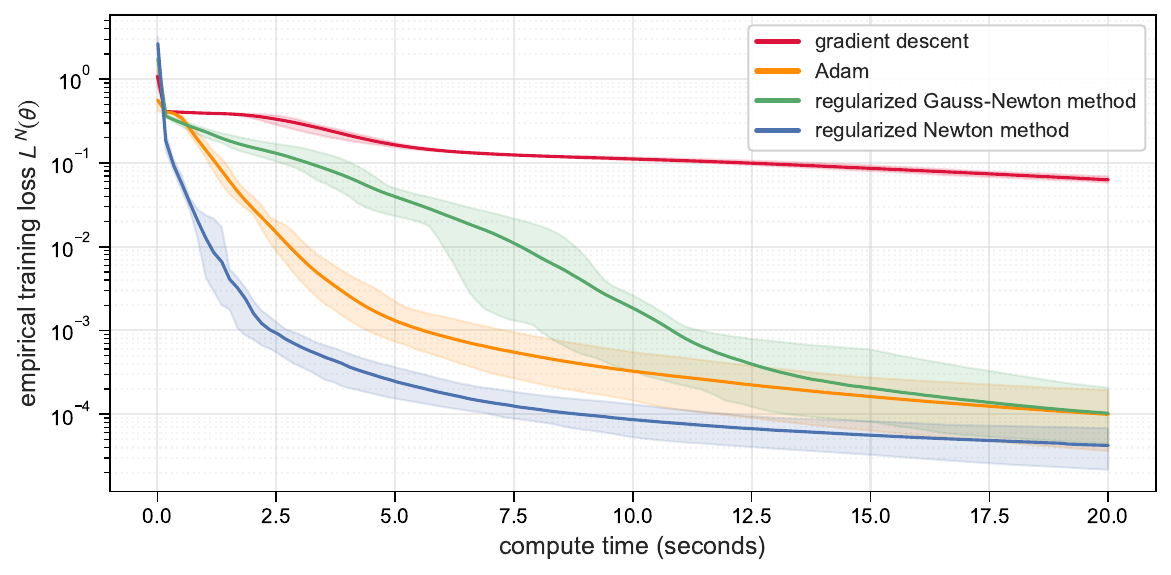}
        \caption{\footnotesize Decay of the loss $L^N(\theta)$ as a function of the elapsed compute time.}
        \label{fig:Comparison_GD_Newton_ComputeTime}
    \end{subfigure}%
    \hphantom{--}%
    \caption{Empirical comparison between gradient descent, Adam, the Gauss--Newton, and the Newton method for a shallow $\tanh$ NN~\eqref{eq:NN} with $N=10^3$ neurons and $\beta=0.52$ on $M=128$ samples of $y=4x \!+\! \sin(2\pi x)  \!+\! 0.8(\sin(8\pi x) \!+\! \sin(10\pi x) \!+\! \sin(16\pi x))$. (Results are averaged across $100$ runs.)}
\label{fig:Comparison_GD_Newton}
\end{figure}

\section{Conclusions and limitations}
\label{sec:conclusions}

This paper developed a convergence analysis for the regularized Newton method for training NNs in the overparameterized limit.
We proved that NNs {of the form \eqref{eq:NN}} trained with Newton's method as in \eqref{eq:regularizedNewton}--\eqref{eq:regularizedNewtonStep} converge to a deterministic limit equation as the number of hidden units $N \rightarrow \infty$ and, as the number of training steps $k \rightarrow \infty$, the empirical loss converges to its global minimum.

In the overparameterized limit, we observe linear convergence of the NN to the target data with a rate that is uniform across the frequency spectrum due to better conditioning of the NNTK compared to its gradient descent counterpart.
Our results, however, do not show the quadratic convergence that one might expect from Newton's method.
This is due to the quadratic rate typically being achieved when the Hessian is not constant, allowing the optimizer to adapt to changing curvature as it approaches the local minimum.
In the overparameterized limit, however, the model behaves as a linearization around the initialization, rendering the learning problem quadratic (hence, the Hessian becoming constant).
We thus instead observe convergence in one training step as $\gamma\rightarrow0$ (for a step size~$\alpha=1$).

\paragraph{Limitations.}
The focus of our analysis was the classical regularized Newton method due to its role as a cornerstone of classical optimization { and mathematical foundation of many approximate second-order methods}.
While the developed techniques may {serve as an analytical framework for methods such as L-BFGS, K-FAC, Shampoo, as well as other quasi-Newton methods}, we leave their study to future research.
{Since these practical variants approximate the curvature differently, they require individual analyses tailored
to their specific implementation.
Nevertheless, by handling a non-trivial term $S_{\theta}^N$ in \eqref{eq:D}, our analysis develops the theoretical machinery for studying them.
In contrast,} our results immediately cover the regularized Gauss--Newton method {(Levenberg--Marquardt algorithm)} by setting $S_{\theta}^N=\mathbf{0}$.
In this case, the analysis simplifies substantially,
highlighting once more that the implicit parameter update of the Newton method with a potentially indefinite Hessian is a central cause for the mathematical challenges.
{Further directions for future work include extensions of our results to adaptive choices of the regularizer parameter $\gamma$, stochastic algorithms, mini-batching, as well as to NNs with unbounded activations~$\sigma$ and deep NNs.
}


\ifarxiv


\fi

\newpage

\appendix

\section*{Supplementary material}

This supplementary material is organized into the following appendices.

\startcontents[main]
\printcontents[main]{}{1}{}

\paragraph{Further Notation.}
By $C$ we denote generic constants, which may vary throughout the different proofs and may change from line to line.
Dependence of such constants on hyperparameters, when applicable, is indicated by notation such as $C(\alpha, \gamma, \mu_0, \CD, \sigma)$, where $\CD$ denotes the dependency on the training data (which includes a dependency on $d$ but not $M$) and $\sigma$ denotes the dependency on properties of the NN nonlinearity $\sigma$.


\section{Proof details for the main results in Section~\ref{sec:main}}
\label{app:proofs_main}

In this appendix,
we present auxiliary technical results and proof details for our main theoretical results, Theorems~\ref{thm:convergence} and~\ref{thm:NTK_limit}.
In particular, we provide a detailed proof of Theorem~\ref{thm:NTK_limit} in Section~\ref{sec:app:proof_NTK_limit}.
Before doing so, we collect and prove several auxiliary statements in Section~\ref{sec:app:aux_res}, and give a proof of Lemma~\ref{prop:NTK_limit_bounds_z} in Section~\ref{sec:app:NTK_limit_bounds_z}.

\subsection{Auxiliary technical results}
\label{sec:app:aux_res}

In Section~\ref{subsec:proof:boundsinitialization}, we provide bounds in expectation for the NN function output~\eqref{eq:NN}, the residual, and the loss~\eqref{eq:loss} at initialization,
while Section~\ref{subsec:proof:boundsgeneric} gives such bounds for generic NN parameters.
Section~\ref{subsec:proof:gradientHessian} then computes the gradient and Hessian of the loss~$L^N$, and introduces several notations.
Auxiliary technical estimates, which will be useful later on, are provided in Section~\ref{subsec:proof:auxtechnicalest}.
Thereafter, in Section~\ref{subsec:proof:matrixbounds}, we establish operator norm bounds for several involved matrices.
In Section~\ref{subsec:proof:zN_reformulation}, we derive Formula~\eqref{eq:zN_reformulation} for the Newton update $z_k^N$ from \eqref{eq:regularizedNewtonStep}.
Section~\ref{subsec:proof:spectra} then investigates the spectral properties of the limit NNTK $\CB^*_0$ as well as the spectral properties of the limit standard NTK $B^*_0$.

\subsubsection{Bounds on the NN function and the loss at initialization}
\label{subsec:proof:boundsinitialization}

Using that the NN parameters associated with different neurons are independent of each other at initialization,
we can obtain the following bounds for the NN function and the loss with NN parameters~$\theta_0$.

\begin{lemma}
    \label{lem:estimate_NNfunction_E} 
    At initialization, it holds for the NN function $f_{\theta_0}^N$, as defined in \eqref{eq:NN}, that
    \begin{equation}
        \bbE\abs{f_{\theta_0}^N(x)}^2 \leq \frac{C}{N^{2\beta-1}}
    \end{equation}
    for a constant $C=C(\sigma,\mu_0)$.
\end{lemma}

\begin{proof}
    Using the independence of the parameters $c_0^i$ from $w_0^i$ and $\eta_0^i$, and when associated with different neurons as of Assumptions~\ref{asm:NN_mu0}\ref{asm:NN_mu0i} and \ref{asm:NN_mu0}, and that the $c_0^i$'s are mean-zero as of Assumption~\ref{asm:NN_mu0}\ref{asm:NN_mu0ii},
    we have
    \begin{equation}
        \bbE\abs{f^N_{\theta_0}(x)}^2
        = \bbE \abs{\frac{1}{N^\beta} \sum_{i=1}^N c_0^i \sigma(w_0^i \cdot x + \eta_0^i)}^2
        = \frac{1}{N^{2\beta}} \bbE\sum_{i=1}^N \abs{c_0^i \sigma(w_0^i \cdot x + \eta_0^i)}^2.
    \end{equation}
    Hence, with Assumptions~\ref{asm:NN_nonlinearity}\ref{asm:NN_sigma} and \ref{asm:NN_mu0}\ref{asm:NN_mu0ii} we conclude
    \begin{equation}
        \bbE\abs{f^N_{\theta_0}(x)}^2
        \leq \frac{1}{N^{2\beta}} N C^2C_\sigma^2
        = \frac{C^2C_\sigma^2}{N^{2\beta-1}}
    \end{equation}
    for a constant $C=C(\mu_0)$.
\end{proof}

\begin{lemma}
    \label{lem:estimate_residual_E}
    At initialization, it holds for the loss~$L^N(\theta_0)$ and the residual $\y-f^N_{\theta_0}(\x)$ that
    \begin{equation}
        \bbE \left[L^N(\theta_0)\right]
        \leq C\left(1+\frac{1}{N^{2\beta-1}}\right) \leq C
    \end{equation}
    and 
    \begin{equation}
        \bbE\N{\y-f_{\theta_0}^N(\x)}_2
        \leq C\sqrt{M}\left(1+\frac{1}{N^{\beta-1/2}}\right)
        \leq C\sqrt{M}
    \end{equation}
    for a constant $C=C(\sigma,\mu_0,\CD)$.
\end{lemma}

\begin{proof}
    Since $\bbE f_{\theta_0}^N(x^m)=0$ due to 
    the independence of the parameters $c_0^i$ from $w_0^i$ and $\eta_0^i$ as of Assumption~\ref{asm:NN_mu0}\ref{asm:NN_mu0i}
    and
    the neurons~$c_0^i$ being mean-zero as of Assumption~\ref{asm:NN_mu0}\ref{asm:NN_mu0ii}, we compute with Lemma~\ref{lem:estimate_NNfunction_E} that
    \begin{equation}
    \begin{split}
        \bbE\N{\y-f_{\theta_0}^N(\x)}_2^2
        &= \bbE \left[\sum_{m=1}^M \left(y^m-f^N_{\theta_0}(x^m)\right)^2\right]
        = \sum_{m=1}^M (y^m)^2 - 2 y^m \bbE f^N_{\theta_0}(x^m) +\bbE\abs{f^N_{\theta_0}(x^m)}^2\\
        &= \sum_{m=1}^M (y^m)^2+\bbE\abs{f^N_{\theta_0}(x^m)}^2
        \leq M\left(C_y^2 + \frac{C}{N^{2\beta-1}}\right).
    \end{split}
    \end{equation}
    Hence, $\bbE \left[L^N(\theta_0)\right] \leq \frac{1}{2}\left(C_y^2+\frac{C}{N^{2\beta-1}}\right)\leq C$,
    where the last inequality is simply due to $N\geq1$ and $\beta\in(1/2,1)$.
    The second statement follows from Jensen's inequality.
\end{proof}

\subsubsection{Bounds on the NN function and the loss}
\label{subsec:proof:boundsgeneric}

Generically, due to the lack of independence once training has started,
we can only obtain the following bounds for the NN function and the loss.
Throughout this and the following sections,
we denote by $\theta$ a generic vector of NN parameters.
It is \textit{assumed} to satisfy $\absnormal{c^i}\leq \widetilde{C}$ for all $i=1,\dots,N$.
As of Assumption~\ref{asm:NN_mu0}\ref{asm:NN_mu0ii},
this condition holds for $\theta_0$ at initialization with $\widetilde{C}=C(\mu_0)$, since $\absnormal{c_0^i}\leq C(\mu_0)$.
For the iterates~$\theta_k$ emerging during training, such a priori bound needs to be first established.

\begin{lemma}
    \label{lem:estimate_NNfunction} 
    Let $\theta=(\theta^i)_{i=1}^N$ with $\theta^i=(c^i,w^i,\eta^i)$ and assume that $\absnormal{c^i}\leq \widetilde{C}$ for all $i=1,\dots,N$.
    Then it holds
    \begin{equation}
        \abs{f_{\theta}^N(x)}^2 \leq CN^{2-2\beta}
    \end{equation}
    for a constant $C=C(\sigma,\widetilde{C})$.
\end{lemma}

\begin{proof}
    With Jensen's inequality it holds
    \begin{equation}
        \abs{f^N_{\theta}(x)}^2
        \!= \abs{\frac{1}{N^\beta} \sum_{i=1}^N c^i \sigma(w^i \cdot x + \eta^i)}^2
        \leq \frac{1}{N^{2\beta-1}} \sum_{i=1}^N \abs{c^i \sigma(w^i \cdot x + \eta^i)}^2
        \!\leq \frac{1}{N^{2\beta-1}} \sum_{i=1}^N \widetilde{C}^2 C_\sigma^2,
    \end{equation}
    which concludes the computation.
\end{proof}

\begin{lemma}
    \label{lem:estimate_residual}
    Let $\theta=(\theta^i)_{i=1}^N$ with $\theta^i=(c^i,w^i,\eta^i)$ and assume that $\absnormal{c^i}\leq \widetilde{C}$ for all $i=1,\dots,N$.
    For the loss~$L^N(\theta)$ and the residual $\y-f_\theta^N(\x)$ it then holds
    \begin{equation}
        L^N(\theta)\leq C(1+N^{2-2\beta})
    \end{equation}
    and 
    \begin{equation}
        \N{\y-f_\theta^N(\x)}_2
        \leq C\sqrt{M}(1+N^{1-\beta})
    \end{equation}
    for a constant $C=C(C_{\sigma},\CD,\widetilde{C})$.
\end{lemma}

\begin{proof}
    We compute with Lemma~\ref{lem:estimate_NNfunction} and using Assumption~\ref{asm:data} that
    \begin{equation}
        \N{\y-f^N_\theta(\x)}_2^2
        = \sum_{m=1}^M \left(y^m-f^N_\theta(x^m)\right)^2
        \leq M\left(C_y + CN^{1-\beta}\right)^2.
    \end{equation}
    Hence, $L^N(\theta) \leq \frac{1}{2}(C_y+CN^{1-\beta})^2$, which concludes the proof.
\end{proof}

\subsubsection[Gradient and Hessian of the loss]{Gradient and Hessian of the Loss~\protect{$L^N(\theta)$}}
\label{subsec:proof:gradientHessian}

We now derive expressions for the gradient and Hessian of the loss~$L^N$.
Throughout this section,
we denote by $\theta$ a generic vector of NN parameters.

\paragraph{Gradient of the loss.}
The gradient~$\nabla_\theta L^N(\theta)$ of the loss~$L^N(\theta)$ from \eqref{eq:loss} is given by
\begin{equation}
    \label{eq:Gradient_loss}
\begin{split}
    \nabla_\theta L^N(\theta)
    &= - \frac{1}{M}\sum_{m=1}^M \left( y^m - f_\theta^N(x^m) \right) \left(J_\theta^N(x^m)\right)^\top \\
    &= - \frac{1}{M} \left(J_\theta^N\right)^\top \left(\y - f_\theta^N(\x)\right),
\end{split}
\end{equation}
where
the Jacobian~$J_\theta^N\in\bbR^{M\times N(d+2)}$ of the NN function~$f_\theta^N$ is given by
\begin{equation}
    \label{eq:Jacobian_f}
    J_\theta^N = \left(J_\theta^N(x^1)^\top,\dots,J_\theta^N(x^M)^\top\right)^\top
\end{equation}
with $J_\theta^N(x)\in\bbR^{1\times N(d+2)}$ defined as
\begin{equation}
    \label{eq:Jacobian_fx}
\begin{split}
    J_\theta^N(x)
    &= \left(\nabla_\theta f_\theta^N(x)\right)^\top 
    = \left(J_{\theta^1}^N(x), \dots, J_{\theta^N}^N(x)\right),
\end{split}
\end{equation}
where $J_{\theta^i}^N(x)\in\bbR^{1\times (d+2)}$ denotes the part of the Jacobian~$J_\theta^N(x)$ corresponding to the $i$-th neuron,
i.e.,
\begin{equation}
    \label{eq:Jacobian_fxi}
    J_{\theta^i}^N(x)
    =\frac{1}{N^\beta}\left(\sigma(w^i\cdot x+\eta^i),c^i\sigma'(w^i\cdot x+\eta^i)x^\top,c^i\sigma'(w^i\cdot x+\eta^i)\right).
\end{equation}
Following this notation, we further denote by $J_{\theta^i}^N\in\bbR^{M\times (d+2)}$ the part of the Jacobian~$J_\theta^N$ corresponding to the $i$-th neuron, i.e., $J_{\theta^i}^N = \big(J_{\theta^i}^N(x^1)^\top,\dots,J_{\theta^i}^N(x^M)^\top\big)^\top$.
A graphical illustration of these definitions is provided in Figure~\ref{fig:jacobian_cuts}.
\begin{figure}[ht!]
\centering
\begin{tikzpicture}[font=\small]
  \def\M{4}
  \def\N{7}
  \def\blockW{1.3}
  \def\blockH{0.5}

  \def\jrow{2}   
  \def\icol{2}   

  \foreach \r in {0,...,\numexpr\M-1\relax} {
    \fill[OXROYAL!60] ({(\icol-1)*\blockW},{-\r*\blockH})
      rectangle ++(\blockW,-\blockH);
  }

  \foreach \c in {0,...,\numexpr\N-1\relax} {
    \fill[OXGREEN!60] (\c*\blockW,{-(\jrow-1)*\blockH})
      rectangle ++(\blockW,-\blockH);
  }

  \fill[OXMIX!60]
    ({(\icol-1)*\blockW},{-(\jrow-1)*\blockH})
    rectangle ++(\blockW,-\blockH);

  \foreach \r in {0,...,\numexpr\M-1\relax} {
    \foreach \c in {0,...,\numexpr\N-1\relax} {
      \draw[thin] (\c*\blockW,-\r*\blockH)
        rectangle ++(\blockW,-\blockH);
    }
  }
  \draw[thick] (0,0) rectangle ++(\N*\blockW,-\M*\blockH);

  \draw[decorate,decoration={brace,amplitude=5pt,mirror},font=\tiny]
(\N*\blockW+0.1,-\M*\blockH) -- ++(0,\M*\blockH)
node[midway,right=4pt] {$M$ rows};

  \draw[decorate,decoration={brace,amplitude=5pt},font=\tiny]
(\N*\blockW,-\M*\blockH-0.65) -- ++(-\N*\blockW,0)
node[midway,below=4pt] {$N(d\!+\!2)$ columns};

  \draw[decorate,decoration={brace,amplitude=4pt,mirror},font=\tiny]
({(\icol-1)*\blockW}, {- \M*\blockH - 0.1})
-- ++(\blockW,0)
node[midway,below=4pt] {$(d\!+\!2)$ columns};

  \node[OXGREEN,anchor=east,font=\small]
at (-0.07,{-(\jrow-1)*\blockH - 0.5*\blockH})
{$J_\theta^N(x^m)$};

  \node[OXROYAL,anchor=south,font=\small]
    at ({(\icol-1)*\blockW + 0.5*\blockW},{0.02})
    {$J_{\theta^i}^N$};

  \node[font=\small]
at ({(\icol-1)*\blockW + 0.5*\blockW},
    {-(\jrow-1)*\blockH - 0.5*\blockH})
{$J_{\theta^i}^N(x^m)$};

\end{tikzpicture}
\caption{The Jacobian $J_\theta^N\in\bbR^{M\times N(d+2)}$ and its different slicings.}
\label{fig:jacobian_cuts}
\end{figure}
Recalling the notations from the notations paragraph in the introduction,
we note that it hold
$J_{\theta}^N = \frac{1}{N^{\beta}}s_{\theta}^N$,
$J_{\theta}^N(x) = \frac{1}{N^{\beta}}s_{\theta}^N(x)$,
$J_{\theta^i}^N = \frac{1}{N^{\beta}}s_{\theta^i}^N$, and
$J_{\theta^i}^N(x) = \frac{1}{N^{\beta}}s_{\theta^i}^N(x)$.
    
\paragraph{Hessian of the loss.}
The Hessian~$\nabla^2_\theta L^N(\theta)$ of the loss~$L^N(\theta)$ from \eqref{eq:loss} is given by
\begin{equation}
    \label{eq:Hessian_loss}
\begin{split}
    \nabla^2_\theta L^N(\theta)
    &= \frac{1}{M}\sum_{m=1}^M \left(J_\theta^N(x^m)\right)^\top J_\theta^N(x^m) - \frac{1}{M}\sum_{m=1}^M \left(y^m - f_\theta^N(x^m)\right) H_\theta^N(x^m) \\
    &= \underbrace{\frac{1}{M}\left(J_\theta^N\right)^\top J_\theta^N}_{=:G_\theta^N} \ \underbrace{- \ \frac{1}{M} \sum_{m=1}^M \left(y^m - f_\theta^N(x^m)\right) H_\theta^N(x^m)}_{=:S_\theta^N},
\end{split}
\end{equation}
where the Hessian~$H^N_\theta\in\bbR^{M\times N(d+2)\times N(d+2)}$ of the NN function~$f^N_\theta$ is given by
\begin{equation}
    \label{eq:Hessian_f}
    H^N_\theta=\left(H^N_\theta(x^1),\dots,H^N_\theta(x^M)\right),
\end{equation}
with $H^N_\theta(x)\in\bbR^{N(d+2)\times N(d+2)}$ defined as
\begin{equation}
    \label{eq:Hessian_fx}
\begin{split}
    H^N_\theta(x)
    &= \nabla_\theta^2 f^N_\theta(x) 
    = 
    \begin{pmatrix}
        H^N_{\theta^1}(x) & \mathbf{0} & \cdots & \mathbf{0} \\
        \mathbf{0} & H^N_{\theta^2}(x) & \cdots & \mathbf{0} \\
        \vdots & \vdots & \ddots & \vdots \\
        \mathbf{0} & \mathbf{0} & \cdots & H^N_{\theta^N}(x)
    \end{pmatrix},
\end{split}
\end{equation}
where $H^N_{\theta^i}(x)\in\bbR^{(d+2)\times (d+2)}$ denotes the part of the Hessian~$H^N_\theta(x)$ corresponding to the $i$-th neuron,
i.e.,
\begin{equation}
    \label{eq:Hessian_fxi}
    H^N_{\theta^i}(x) = \frac{1}{N^\beta}\begin{pmatrix}0 & \sigma'(w^i \cdot x + \eta^i)x^\top & \sigma'(w^i \cdot x + \eta^i) \\
    \sigma'(w^i \cdot x + \eta^i)x & c^i\sigma''(w^i \cdot x + \eta^i)xx^\top & c^i\sigma''(w^i \cdot x + \eta^i)x \\
    \sigma'(w^i \cdot x + \eta^i) & c^i\sigma''(w^i \cdot x + \eta^i)x^\top & c^i\sigma''(w^i \cdot x + \eta^i) \end{pmatrix}.
\end{equation}
Recalling the notations from the notations paragraph in the introduction,
we note that it holds
$H_{\theta^i}^N(x) = \frac{1}{N^{\beta}}h_{\theta^i}^N(x)$.

\paragraph{Decomposition of the Hessian.}
To simplify notations throughout the paper,
let us introduce the Gram-Gauss--Newton matrix~$G_\theta^N$ as 
\begin{equation}
    \label{eq:GaussNewton}
    G_\theta^N = \frac{1}{M}\left(J_\theta^N\right)^\top J_\theta^N,
\end{equation}
and the non-convex contribution~$S_\theta^N$ of the Hessian as
\begin{equation}
    \label{eq:secondOrderPartHessian}
    S_\theta^N = - \frac{1}{M} \sum_{m=1}^M (y^m - f^N_\theta(x^m)) H_\theta^N(x^m).
\end{equation}
Hence, the Hessian in \eqref{eq:Hessian_loss} decomposes as $\nabla^2_\theta L^N(\theta) = G_\theta^N + S_\theta^N$.

\subsubsection{Auxiliary technical estimates}
\label{subsec:proof:auxtechnicalest}

In this section, we prove several auxiliary estimates, which are useful throughout later sections.

\begin{lemma}
    \label{lem:bound_s}
    Let $\theta^i=(c^i,w^i,\eta^i)$ and assume that $\absnormal{c^i}\leq \widetilde{C}$.
    Then it holds
    \begin{equation}
        \N{s^N_{\theta^i}}_{\mathrm{op}} \leq C\sqrt{M}
    \end{equation}
    for a constant $C=C(\sigma,\CD,\widetilde{C})$.
\end{lemma}

\begin{proof}
    With the operator norm being bounded by the Frobenius norm,
    we can estimate
        \begin{equation}
        \begin{split}
            \N{s^N_{\theta^i}}_{\mathrm{op}}^2 &\leq \N{s^N_{\theta^i}}_{\mathrm{F}}^2 = \sum_{m=1}^M \N{s^N_{\theta^i}(x^m)}_{2}^2 \\
            &= \sum_{m=1}^M \N{\left(\sigma(w^i \cdot x^m+\eta^i), c^i\sigma'(w^i \cdot x^m+\eta^i)(x^m)^\top,c^i\sigma'(w^i \cdot x^m+\eta^i)\right)}_{2}^2 \\
            &\leq M \left(C_\sigma^2 + \widetilde{C}^2C_{\sigma'}^2(C_x^2+1)\right),
        \end{split}
        \end{equation}
        where we used Assumptions~\ref{asm:NN_nonlinearity}\ref{asm:NN_sigma} and \ref{asm:NN_nonlinearity}\ref{asm:NN_sigma'} together with Assumption~\ref{asm:data} to obtain the last step.
    \end{proof}
    
\begin{lemma}
    \label{lem:bound_s_Lipschitz}
    Let $\theta^i=(c^i,w^i,\eta^i)$ and $\tilde\theta^i=(\tilde{c}^i,\widetilde{w}^i,\tilde{\eta}^i)$, and assume that $\absnormal{c^i}\leq \widetilde{C}$ and $\absnormal{\tilde{c}^i}\leq \widetilde{C}$.
    Then it holds
    \begin{equation}
        \N{s^N_{\theta^i}-s^N_{\tilde\theta^i}}_{\mathrm{op}}
        \leq C\sqrt{M}\Nbig{\theta^i-\widetilde{\theta}^i}_2
    \end{equation}
    for a constant $C=C(\sigma,\CD,\widetilde{C})$.
\end{lemma}

\begin{proof}
    With the operator norm being bounded by the Frobenius norm,
    we can estimate
    \begin{equation}
    \begin{split}
        \N{s^N_{\theta^i}-s^N_{\tilde\theta^i}}_{\mathrm{op}}^2
        &\leq \N{s^N_{\theta^i}-s^N_{\tilde\theta^i}}_{\mathrm{F}}^2
        = \sum_{m=1}^M \N{s^N_{\theta^i}(x^m)-s^N_{\tilde\theta^i}(x^m)}_{2}^2 \\
        &= \sum_{m=1}^M \left(\sigma(w^i \cdot x^m+\eta^i)-\sigma(\widetilde{w}^i \cdot x^m+\tilde{\eta}^i)\right)^2 \\
        &\qquad\quad\,+ \left(c^i\sigma'(w^i \cdot x^m+\eta^i)-\tilde{c}^i\sigma'(\widetilde{w}^i \cdot x^m+\tilde{\eta}^i)\right)^2\left(\N{x^m}_2^2+1\right) \\
        &\leq 2\sum_{m=1}^M L_\sigma^2 \left(\N{w^i-\widetilde{w}^i}_2^2 \N{x^m}_2^2 + \abs{\eta^i-\tilde{\eta}^i}^2\right) \\
        &\qquad\quad\,+ \left(C_{\sigma'}^2\abs{c^i-\tilde{c}^i}^2 + 2\widetilde{C}^2L_{\sigma'}^2\left(\N{w^i-\widetilde{w}^i}_2^2 \N{x^m}_2^2 + \abs{\eta^i-\tilde{\eta}^i}^2\right)\right)\cdot\\
        &\qquad\quad\,\qquad\,\cdot\left(\N{x^m}_2^2+1\right) \\
        &\leq CM \left(\abs{c^i-\tilde{c}^i}^2 + \N{w^i-\widetilde{w}^i}_2^2 + \abs{\eta^i-\tilde{\eta}^i}^2\right) \\
        &= CM \Nbig{\theta^i-\widetilde{\theta}^i}_2^2,
    \end{split}
    \end{equation}
    where we used Assumptions~\ref{asm:NN_nonlinearity}\ref{asm:NN_sigma} and \ref{asm:NN_nonlinearity}\ref{asm:NN_sigma'}  to obtain the next-to-last step together with Assumption~\ref{asm:data} in the last step.
\end{proof}

\begin{lemma}
    \label{lem:bound_h}
    Let $\theta^i=(c^i,w^i,\eta^i)$ and assume that $\absnormal{c^i}\leq \widetilde{C}$.
    Then, for $h^N_{\theta^i}(x^m)$, as defined in \eqref{eq:h}, it holds
    \begin{equation}
        \N{h^N_{\theta^i}(x^m)}_{\mathrm{op}} \leq C
    \end{equation}
    for a constant $C=C(\sigma,\CD,\widetilde{C})$.
\end{lemma}

\begin{proof}
    With the operator norm being bounded by the Frobenius norm,
    we can estimate
    \begin{equation}
    \begin{split}
        \N{h^N_{\theta^i}(x^m)}_{\mathrm{op}}^2
        &= \N{\begin{pmatrix}0 \!\!&\!\! \sigma'(w^i \!\cdot\! x^m \!\!+\! \eta^i)(x^m)^\top \!\!&\!\! \sigma'(w^i \!\cdot\! x^m \!\!+\! \eta^i) \\
        \sigma'(w^i \!\cdot\! x^m \!\!+\! \eta^i)x^m \!\!\!&\!\! c^i\sigma''(w^i \!\cdot\! x^m \!\!+\! \eta^i)x^m(x^m)^\top \!\!\!&\!\! c^i\sigma''(w^i \!\cdot\! x^m \!\!+\! \eta^i)x^m \\
        \sigma'(w^i \!\cdot\! x^m \!\!+\! \eta^i) \!\!&\!\! c^i\sigma''(w^i \!\cdot\! x^m \!\!+\! \eta^i)(x^m)^\top \!\!\!&\!\! c^i\sigma''(w^i \!\cdot\! x^m \!\!+\! \eta^i) \end{pmatrix}}_{\mathrm{op}}^2 \\
        &\leq \N{\begin{pmatrix}0 \!\!&\!\! \sigma'(w^i \!\cdot\! x^m \!\!+\! \eta^i)(x^m)^\top \!\!&\!\! \sigma'(w^i \!\cdot\! x^m \!\!+\! \eta^i) \\
        \sigma'(w^i \!\cdot\! x^m \!\!+\! \eta^i)x^m \!\!\!&\!\! c^i\sigma''(w^i \!\cdot\! x^m \!\!+\! \eta^i)x^m(x^m)^\top \!\!\!&\!\! c^i\sigma''(w^i \!\cdot\! x^m \!\!+\! \eta^i)x^m \\
        \sigma'(w^i \!\cdot\! x^m \!\!+\! \eta^i) \!\!&\!\! c^i\sigma''(w^i \!\cdot\! x^m \!\!+\! \eta^i)(x^m)^\top \!\!\!&\!\! c^i\sigma''(w^i \!\cdot\! x^m \!\!+\! \eta^i) \end{pmatrix}}_{\mathrm{F}}^2 \\
        &\leq 2C_{\sigma'}^2\N{x^m}_2^2 + 2C_{\sigma'}^2 + \widetilde{C}^2C_{\sigma''}^2\N{x^m(x^m)^\top}_{\mathrm{F}}^2 + 2\widetilde{C}^2C_{\sigma''}^2\N{x^m}_2^2 + \widetilde{C}^2C_{\sigma''}^2 \\
        &\leq 2C_{\sigma'}^2C_x^2 + 2C_{\sigma'}^2 + \widetilde{C}^2C_{\sigma''}^2C_x^4 + 2\widetilde{C}^2C_{\sigma''}^2C_x^2 + \widetilde{C}^2C_{\sigma''}^2,
    \end{split}
    \end{equation}
    where we used Assumptions~\ref{asm:NN_nonlinearity}\ref{asm:NN_sigma'} and \ref{asm:NN_nonlinearity}\ref{asm:NN_sigma''} to obtain the next-to-last step together with Assumption~\ref{asm:data} in the last step.
\end{proof}

\begin{lemma}
    \label{lem:estimate_S_operatornorm_i}
    Let $\theta^i=(c^i,w^i,\eta^i)$ and assume that $\absnormal{c^i}\leq \widetilde{C}$.
    Then
    it holds
    \begin{equation}
        \N{\frac{1}{M} \sum_{m=1}^M (y^m - f_{\theta}^N(x^m)) h_{\theta^i}^N(x^m)}_{\mathrm{op}}
        \leq C_S \sqrt{L^N(\theta)}
    \end{equation}
    for a constant $C_S=C_S(\sigma,\CD,\widetilde{C})$.
\end{lemma}

\begin{proof}
   We compute with triangle inequality in the first and Cauchy-Schwarz inequality in the second step that
    \begin{equation}
    \begin{split}
        &\N{\frac{1}{M} \sum_{m=1}^M (y^m - f_{\theta}^N(x^m)) h_{\theta^i}^N(x^m)}_{\mathrm{op}}
        \leq \frac{1}{M} \sum_{m=1}^M \abs{y^m - f_\theta^N(x^m)}\N{h_{\theta^i}^N(x^m)}_{\mathrm{op}} \\
        &\qquad\quad\, \leq \left(\frac{1}{M} \sum_{m=1}^M \left(y^m - f_\theta^N(x^m)\right)^2\right)^{1/2} \left(\frac{1}{M} \sum_{m=1}^M\N{h_{\theta^i}^N(x^m)}_{\mathrm{op}}^2\right)^{1/2} \\
        &\qquad\quad\, \leq C \sqrt{2L^N(\theta)},
    \end{split}
    \end{equation}
    where we used Lemma~\ref{lem:bound_h} to obtain the last inequality.
    The definition $C_S=\sqrt{2}C$ concludes the proof.
\end{proof}

Let us further denote by $d_{\theta^i}^N\in\bbR^{(d+2)\times(d+2)}$ the matrix
    \begin{equation}
        \label{eq:def_d}
    \begin{split}
        d_{\theta^i}^N
        &= \gamma \Id_{d+2} - \frac{1}{N^{1-\beta}} \frac{1}{M} \sum_{m=1}^M (y^m - f_{\theta}^N(x^m)) h_{\theta^i}^N(x^m).
    \end{split}
    \end{equation}

\begin{lemma}
    \label{lem:estimate_Dinv_operatornorm_i}
    Let $\theta^i=(c^i,w^i,\eta^i)$ and assume that $\absnormal{c^i}\leq \widetilde{C}$.
    Let $\gamma > \frac{C_S}{N^{1-\beta}} \sqrt{L^N(\theta)}$ with $C_S=C_S(\sigma,\CD,\widetilde{C})$ as in Lemma~\ref{lem:estimate_S_operatornorm_i}.
    Then the matrix $d_{\theta^i}^N$ is positive definite (hence invertible) and it holds
    \begin{equation}
        \N{\left(d_{\theta^i}^N\right)^{-1}}_{\mathrm{op}}
        \leq \frac{1}{\gamma - \frac{C_S}{N^{1-\beta}} \sqrt{L^N(\theta)}}.
    \end{equation}
\end{lemma}


\begin{proof}
    For the minimal eigenvalue of $d_{\theta^i}^N=\gamma\Id_{d+2}  - \frac{1}{N^{1-\beta}} \frac{1}{M} \sum_{m=1}^M (y^m - f_{\theta}^N(x^m)) h_{\theta^i}^N(x^m)$, we can estimate that
    \begin{equation}
        \label{eq:proof:lem:estimate_Dinv_operatornorm_i:10}
    \begin{split}
        \lambda_{\min}\left(d_{\theta^i}^N\right)
        &=\lambda_{\min}\left(\gamma\Id_{d+2}  - \frac{1}{N^{1-\beta}} \frac{1}{M} \sum_{m=1}^M (y^m - f_{\theta}^N(x^m)) h_{\theta^i}^N(x^m)\right) \\
        &=\gamma - \lambda_{\max}\left( \frac{1}{N^{1-\beta}} \frac{1}{M} \sum_{m=1}^M (y^m - f_{\theta}^N(x^m)) h_{\theta^i}^N(x^m)\right) \\
        &\geq \gamma - \frac{1}{N^{1-\beta}} \N{\frac{1}{M} \sum_{m=1}^M (y^m - f_{\theta}^N(x^m)) h_{\theta^i}^N(x^m)}_{\mathrm{op}} \\
        &\geq \gamma - \frac{C_S}{N^{1-\beta}} \sqrt{L^N(\theta)},
    \end{split}
    \end{equation}
    where the last inequality is due to Lemma~\ref{lem:estimate_S_operatornorm_i}.
    Since by assumption $\gamma > \frac{C_S}{N^{1-\beta}} \sqrt{L^N(\theta)}$, the right-hand side of \eqref{eq:proof:lem:estimate_Dinv_operatornorm_i:10} is strictly positive and thus
    the matrix $d_{\theta^i}^N$ is positive definite and hence invertible.
    The operator norm of its inverse can be bounded as
    \begin{equation}
    \begin{split}
        \N{\left(d_{\theta^i}^N\right)^{-1}}_{\mathrm{op}}
        = \abs{\lambda_{\max}\left(\left(d_{\theta^i}^N\right)^{-1}\right)} 
        = \abs{\lambda_{\min}\left(d_{\theta^i}^N\right)}^{-1} 
        \leq \frac{1}{\gamma - \frac{C_S}{N^{1-\beta}} \sqrt{L^N(\theta)}},
    \end{split}
    \end{equation}
    where we used \eqref{eq:proof:lem:estimate_Dinv_operatornorm_i:10} to obtain the last step.
\end{proof}

\subsubsection[Preliminary estimates for operator norms]{Preliminary estimates for operator norms of $G_\theta^N$, $S_\theta^N$, and $D_\theta^N$}
\label{subsec:proof:matrixbounds}

We now establish preliminary estimates for the operator norms of the matrices $G_\theta^N$, $S_\theta^N$, and $D_\theta^N$.

The entries of the Jacobian~$J_\theta^N$ of the NN function $f_\theta^N$ can be bounded as follows.
\begin{lemma}
    \label{lem:estimate_J_entries}
    Let $\theta=(\theta^i)_{i=1}^N$ with $\theta^i=(c^i,w^i,\eta^i)$ and assume that $\absnormal{c^i}\leq \widetilde{C}$ for all $i=1,\dots,N$.
    Then, for the Jacobian~$J_\theta^N$ of the NN function $f_\theta^N$, as defined in \eqref{eq:Jacobian_f},
    it holds
    \begin{equation}
        \abs{(J_\theta^N)_{mj}} \leq \frac{C}{N^{\beta}}
    \end{equation}
    for a constant $C=C(\sigma,\CD,\widetilde{C})$.
\end{lemma}

\begin{proof}
    Using Assumptions~\ref{asm:NN_nonlinearity}\ref{asm:NN_sigma} and \ref{asm:NN_nonlinearity}\ref{asm:NN_sigma'}, Assumption~\ref{asm:data}, as well as the assumption that $\absnormal{c^i}\leq \widetilde{C}$ for all $i=1,\dots,N$,
    we can upper bound the individual entries of the Jacobian~$J^N_\theta$ as
    \begin{equation}
        \abs{\frac{1}{N^\beta}\sigma(w^i \cdot x + \eta^i)} \leq \frac{C_\sigma}{N^{\beta}},
        \qquad
        \abs{\frac{1}{N^\beta}c^i\sigma'(w^i \cdot x + \eta^i)x_j} \leq \frac{\widetilde{C}C_{\sigma'}C_x}{N^{\beta}}
    \end{equation}
    or
    \begin{equation}
        \abs{\frac{1}{N^\beta}c^i\sigma'(w^i \cdot x + \eta^i)} \leq \frac{\widetilde{C}C_{\sigma'}}{N^{\beta}},
    \end{equation}
    respectively, which concludes the proof.
\end{proof}

The operator norm of the Gauss--Newton matrix $G_\theta^N=\frac{1}{M}(J_\theta^N)^\top J_\theta^N$ can be bounded as follows.
\begin{lemma}
    \label{lem:op_norm:GaussNewton}
    Let $\theta=(\theta^i)_{i=1}^N$ with $\theta^i=(c^i,w^i,\eta^i)$ and assume that $\absnormal{c^i}\leq \widetilde{C}$ for all $i=1,\dots,N$.
    Then, for the Gauss--Newton matrix~$G_\theta^N$ of the NN function $f_\theta^N$, as defined in \eqref{eq:GaussNewton}, it holds
    \begin{equation}
        \N{G_\theta^N}_{\mathrm{op}}
        \leq \frac{C}{N^{2\beta-1}}
    \end{equation}
    for a constant $C=C(\sigma,\CD,\widetilde{C})$.
\end{lemma}

\begin{proof}
    With the definition of the operator norm in the first line, triangle inequality in the third, and Cauchy-Schwarz inequality in the fifth line,
    we estimate
    \begin{equation}
    \begin{split}
        \N{G_\theta^N}_{\mathrm{op}}
        =\N{\frac{1}{M}\left(J_\theta^N\right)^\top J_\theta^N}_{\mathrm{op}}
        &= \sup_{v\in\bbR^{N(d+2)}:\N{v}_2=1} \N{\frac{1}{M} \left(J_\theta^N\right)^\top J_\theta^Nv}_2 \\
        &=\sup_{v\in\bbR^{N(d+2)}:\N{v}_2=1} \N{\frac{1}{M} \sum_{m=1}^M \left(J_\theta^N(x^m)\right)^\top\left(J_\theta^N(x^m)v\right)}_2 \\
        &\leq \sup_{v\in\bbR^{N(d+2)}:\N{v}_2=1} \frac{1}{M} \sum_{m=1}^M \N{\left(J_\theta^N(x^m)\right)^\top(J_\theta^N(x^m)v)}_2 \\
        &= \sup_{v\in\bbR^{N(d+2)}:\N{v}_2=1} \frac{1}{M} \sum_{m=1}^M \abs{J_\theta^N(x^m)v} \N{J_\theta^N(x^m)}_2\\
        &\leq \sup_{v\in\bbR^{N(d+2)}:\N{v}_2=1} \frac{1}{M} \sum_{m=1}^M \N{J_\theta^N(x^m)}_2\N{v}_2 \N{J_\theta^N(x^m)}_2\\
        &= \frac{1}{M} \sum_{m=1}^M \N{J_\theta^N(x^m)}_2^2
        = \frac{1}{M} \sum_{m=1}^M \sum_{j=1}^{N(d+2)} \abs{(J_\theta^N)_{mj}}^2 \\
        &\leq CN(d+2) \frac{1}{N^{2\beta}}
        = C(d+2) \frac{1}{N^{2\beta-1}},
    \end{split}
    \end{equation}
    where we employed Lemma~\ref{lem:estimate_J_entries} to obtain the bound in the last line.
\end{proof}

The operator norm of the Hessian~$H^N_\theta$ of the NN function $f^N_\theta$ can be bounded as follows.
\begin{lemma}
    \label{lem:op_norm:Hessian}
    Let $\theta=(\theta^i)_{i=1}^N$ with $\theta^i=(c^i,w^i,\eta^i)$ and assume that $\absnormal{c^i}\leq \widetilde{C}$ for all $i=1,\dots,N$.
    Then, for the Hessian~$H^N_\theta$ of the NN function $f^N_\theta$, as defined in \eqref{eq:Hessian_f} and \eqref{eq:Hessian_fx}, it holds
    \begin{equation}
        \N{H^N_\theta(x^m)}_{\mathrm{op}}
        \leq \frac{C}{N^{\beta}}
    \end{equation}
    for a constant $C=C(\sigma,\CD,\widetilde{C})$.
    Moreover, $\N{H^N_{\theta^i}(x^m)}_{\mathrm{op}}
        \leq \frac{C}{N^{\beta}}$ for all $i=1,\dots,N$.
\end{lemma}

\begin{proof}
    Noticing that the Hessian $H^N_\theta(x)$ is a block-diagonal matrix with $N$ blocks of size $(d+2)\times(d+2)$
    and recalling that its operator norm equals the maximum of the operator norms of its diagonal blocks,
    we have
    \begin{equation}
    \begin{split}
        \N{H^N_\theta(x^m)}_{\mathrm{op}}
        &= \max_{i=1,\dots,N}
           \N{H^N_{\theta^i}(x^m)}_{\mathrm{op}}
        = \max_{i=1,\dots,N}
           \N{\frac{1}{N^\beta}
           h^N_{\theta^i}(x^m)}_{\mathrm{op}} \\
        &= \frac{1}{N^\beta}
           \max_{i=1,\dots,N}
           \N{h^N_{\theta^i}(x^m)}_{\mathrm{op}} \\
        &\leq \frac{C}{N^\beta},
    \end{split}
    \end{equation}
    where the inequality in the last line follows from Lemma~\ref{lem:bound_h}.
\end{proof}

The operator norm of the non-convex contribution~$S_\theta^N$ of the Hessian $\nabla^2_\theta L^N(\theta)$ can be bounded as follows.

\begin{lemma}
    \label{lem:estimate_S_operatornorm}
    Let $\theta=(\theta^i)_{i=1}^N$ with $\theta^i=(c^i,w^i,\eta^i)$ and assume that $\absnormal{c^i}\leq \widetilde{C}$ for all $i=1,\dots,N$.
    Then, for the matrix~$S_\theta^N$, as defined in \eqref{eq:secondOrderPartHessian}, it holds
    \begin{equation}
        \N{S_\theta^N}_{\mathrm{op}}
        \leq \frac{C_S}{N^{\beta}} \sqrt{L^N(\theta)}
    \end{equation}
    for the same constant $C_S=C_S(\sigma,\CD,\widetilde{C})$ as in Lemma~\ref{lem:estimate_S_operatornorm_i}.
\end{lemma}

\begin{proof}
   We compute with triangle inequality in the second and Cauchy-Schwarz inequality in the third line that
    \begin{equation}
    \begin{split}
        \N{S_\theta^N}_{\mathrm{op}}
        &= \N{\frac{1}{M} \sum_{m=1}^M \left(y^m - f_\theta^N(x^m)\right) H_\theta^N(x^m)}_{\mathrm{op}} \\
        &\leq \frac{1}{M} \sum_{m=1}^M \abs{y^m - f_\theta^N(x^m)}\N{H_\theta^N(x^m)}_{\mathrm{op}} \\
        &\leq \left(\frac{1}{M} \sum_{m=1}^M \left(y^m - f_\theta^N(x^m)\right)^2\right)^{1/2} \left(\frac{1}{M} \sum_{m=1}^M\N{H_\theta^N(x^m)}_{\mathrm{op}}^2\right)^{1/2} \\
        & \leq \sqrt{2L^N(\theta)} \frac{C}{N^{\beta}} \leq \frac{C}{N^{\beta}} \sqrt{L^N(\theta)},
    \end{split}
    \end{equation}
    where we used Lemma~\ref{lem:op_norm:Hessian} to obtain the last inequality.
    (Alternatively, we could have used Lemma~\ref{lem:estimate_S_operatornorm_i} for the proof.)
\end{proof}

Let us recall from \eqref{eq:D} the definition $D_{\theta}^N = \gamma^N \Id_{N(d+2)} + S_{\theta}^N$.
Since $S_{\theta}^N$ is a sum of $M$ block-diagonal (of the same shape) matrices $H^N_\theta(x^m)$, $m=1,\dots,M$,
$S_{\theta}^N$ and hence $D_{\theta}^N$ are block-diagonal of the same form as \eqref{eq:Hessian_fx}.
We denote the individual blocks by $D_{\theta^i}^N$, i.e.,
$D_{\theta^i}^N\in\bbR^{(d+2)\times (d+2)}$ denotes the part of the matrix~$D_{\theta}^N$ corresponding to the $i$-th neuron.

\begin{lemma}
    \label{lem:estimate_Dinv_operatornorm}
    Let $\theta=(\theta^i)_{i=1}^N$ with $\theta^i=(c^i,w^i,\eta^i)$ and assume that $\absnormal{c^i}\leq \widetilde{C}$ for all $i=1,\dots,N$.
    Let $\gamma > \frac{C_S}{N^{1-\beta}} \sqrt{L^N(\theta)}$ with $C_S=C_S(\sigma,\CD,\widetilde{C})$ as in Lemma~\ref{lem:estimate_S_operatornorm_i}.
    Then the matrix $D_\theta^N$ is positive definite (hence invertible) and it holds
    \begin{equation}
        \N{\left(D_\theta^N\right)^{-1}}_{\mathrm{op}}
        \leq \frac{N^{\beta}}{\gamma N^{1-\beta} - C_S \sqrt{L^N(\theta)}}.
    \end{equation}
\end{lemma}

\begin{proof}
    We can estimate for the minimal eigenvalue of $D_\theta^N=\gamma^N\Id_{N(d+2)} + S_\theta^N$ that
    \begin{equation}
        \label{eq:proof:lem:estimate_Dinv_operatornorm:10}
    \begin{split}
        \lambda_{\min}\left(D_\theta^N\right)
        =\lambda_{\min}\left(\gamma^N\Id_{N(d+2)} + S_\theta^N\right)
        &=\gamma^N + \lambda_{\min}\left(S_\theta^N\right) \\
        &\geq \gamma^N - \N{S_\theta^N}_{\mathrm{op}} \\
        &\geq \gamma^N - \frac{C_S}{N^{\beta}} \sqrt{L^N(\theta)}
        = \frac{\gamma}{N^{2\beta-1}} - \frac{C_S}{N^{\beta}} \sqrt{L^N(\theta)},
    \end{split}
    \end{equation}
    where the last inequality is due to Lemma~\ref{lem:estimate_S_operatornorm}.
    Since by assumption $\gamma > \frac{C_S}{N^{1-\beta}} \sqrt{L^N(\theta)}$, the right-hand side of \eqref{eq:proof:lem:estimate_Dinv_operatornorm:10} is strictly positive and thus
    the matrix $D_\theta^N$ is positive definite and hence invertible.
    The operator norm of its inverse can be bounded as
    \begin{equation}
    \begin{split}
        \N{\left(D_\theta^N\right)^{-1}}_{\mathrm{op}}
        = \abs{\lambda_{\max}\left(\left(D_\theta^N\right)^{-1}\right)} 
        = \abs{\lambda_{\min}\left(D_\theta^N\right)}^{-1} 
        \leq \frac{N^{\beta}}{\gamma N^{1-\beta} - C_S \sqrt{L^N(\theta)}},
    \end{split}
    \end{equation}
    where we used \eqref{eq:proof:lem:estimate_Dinv_operatornorm:10} to obtain the last step.
    (Alternatively, we could have used Lemma~\ref{lem:estimate_Dinv_operatornorm_i} for the proof.)
\end{proof}

\subsubsection{Proof of Formula~(\ref{eq:zN_reformulation})}
\label{sec:proof:eq:zN_reformulation}
\label{subsec:proof:zN_reformulation}

We now have available all necessary tools to derive Formula~\eqref{eq:zN_reformulation} for the Newton update $z_k^N$ from \eqref{eq:regularizedNewtonStep}.
Throughout this section, we denote by $\theta$ a generic vector of NN parameters.

\begin{proof}
    With the definition of the matrix $D_{\theta}^N$, see, e.g., \eqref{eq:D},
    we can write \eqref{eq:regularizedNewtonStep} as 
    \begin{equation}
        \left(D_{\theta}^N + G_{\theta}^N\right) z^N
        = \frac{1}{M} \left(J_\theta^N\right)^\top \left(\y - f_\theta^N(\x)\right),
    \end{equation}
    where we used \eqref{eq:Hessian_loss} for the left-hand side and \eqref{eq:Gradient_loss} for the right-hand side.
    Noticing that the matrix on the left-hand side is the sum of the block-diagonal matrix~$D_{\theta}^N$ and the low-rank matrix~$G_{\theta}^N$,
    we can use the Woodbury matrix identity (push-through identity) in the third line to obtain the representation
    \begin{equation}
        \label{proof:eq:zN_representatio}
    \begin{split}
        z^N
        &= \frac{1}{M}\left(D_{\theta}^N + G_{\theta}^N\right)^{-1}  \left(J_\theta^N\right)^\top \left(\y - f_\theta^N(\x)\right) \\
        &= \frac{1}{M}\left(D_{\theta}^N + \frac{1}{M}\left(J_\theta^N\right)^\top J_\theta^N\right)^{-1} \left(J_\theta^N\right)^\top \left(\y - f_\theta^N(\x)\right) \\
        &= \frac{1}{M}\left(D_{\theta}^N\right)^{-1} \left(J_\theta^N\right)^\top \left(\Id_M + \frac{1}{M}J_\theta^N\left(D_{\theta}^N\right)^{-1}\left(J_\theta^N\right)^\top \right)^{-1}\left(\y - f_\theta^N(\x)\right) \\
        &= \frac{1}{M}\left(D_{\theta}^N\right)^{-1}\! \left(J_\theta^N\right)^\top\! \zeta^N 
        \quad\text{ with }\quad
        \zeta^N\! \coloneqq \left(\Id_M + \frac{1}{M}J_\theta^N\left(D_{\theta}^N\right)^{-1}\!\left(J_\theta^N\right)^\top \right)^{-1}\!\!\left(\y - f_\theta^N(\x)\right)\!,
    \end{split}
    \end{equation}
    which concludes the proof by giving \eqref{eq:zN_reformulation}.
    To convince the reader of the matrix identity
    \begin{equation}
        \label{proof:eq:matrixIdentity}
    \begin{split}
        \left(D_{\theta}^N + \frac{1}{M}\left(J_\theta^N\right)^\top J_\theta^N\right)^{-1} \left(J_\theta^N\right)^\top
        = \left(D_{\theta}^N\right)^{-1} \left(J_\theta^N\right)^\top \left(\Id_M + \frac{1}{M}J_\theta^N\left(D_{\theta}^N\right)^{-1}\left(J_\theta^N\right)^\top \right)^{-1}
    \end{split}
    \end{equation}
    used in the third line of \eqref{proof:eq:zN_representatio},
    first note that the expressions are well-defined in the sense that the involved matrices are indeed invertible:
    The matrix $D_{\theta}^N + \frac{1}{M}(J_\theta^N)^\top J_\theta^N$ is positive definite due to $D_\theta^N$ being positive definite by assumption and due to $(J_\theta^N)^\top J_\theta^N$ being positive semi-definite.
    For the matrix $\Id_M + \frac{1}{M}J_\theta^N(D_\theta^N)^{-1}(J_\theta^N)^\top$ simply notice that $D_\theta^N$ being positive definite implies $(D_\theta^N)^{-1}$ being positive definite and hence $J_\theta^N(D_\theta^N)^{-1}(J_\theta^N)^\top$ being positive semi-definite.
    To validate \eqref{proof:eq:matrixIdentity},
    let us multiply the identity \eqref{proof:eq:matrixIdentity} from the left with $D_\theta^N + \frac{1}{M}(J_\theta^N)^\top J_\theta^N$ and from the right with $\Id_M + \frac{1}{M}J_\theta^N(D_\theta^N)^{-1}(J_\theta^N)^\top$.
    Hence, \eqref{proof:eq:matrixIdentity} is equivalent to
    \begin{equation}
    \begin{split}
        \left(J_\theta^N\right)^\top \left(\Id_M + \frac{1}{M}J_\theta^N\left(D_\theta^N\right)^{-1}\left(J_\theta^N\right)^\top \right)
        = \left(D_\theta^N + \frac{1}{M}\left(J_\theta^N\right)^\top J_\theta^N\right)\left(D_\theta^N\right)^{-1} \left(J_\theta^N\right)^\top,
    \end{split}
    \end{equation}
    which holds trivially once multiplying out the terms.
\end{proof}

\subsubsection[Spectral properties of the standard NTK and NNTK]{Spectral properties of the standard NTK $B^*_0$ and NNTK $\CB^*_0$}
\label{subsec:proof:spectra}

Let us first recall that the eigenvalues of the limit standard NTK $B^*_0$ are strictly positive, see, e.g., \cite{jacot2018neural}, \cite[Corollary~1.4]{sirignano2019scaling}, or \cite[Corollary~19.6]{spiliopoulos2025mathematical}.
For the reader's convenience, we repeat the argument.

\begin{lemma}
    \label{lem:standardNTKB:spectrum}
    The limit standard NTK $B^*_0$, as defined in \eqref{eq:standardNTK}, is positive definite,
    i.e., its eigenvalues $\{\lambda_{m}(B^*_0)\}_{m=1}^M$ are strictly positive.
\end{lemma}

\begin{proof}
    Using the definition~\eqref{eq:standardNTK} of the limit standard NTK $B^*_0$,
    we can compute for any vector $v\in\bbR^M$ that
    \begin{equation}
        \label{eq:proof:lem:Ginfty_pos_def:1}
    \begin{split}
        v^\top B^*_0 v
        &= \sum_{m=1}^M v_m \sum_{n=1}^M \frac{1}{M}\int \sigma(w \cdot x^m+\eta)\sigma(w \cdot x^n+\eta) \, d\mu_0(c,w,\eta) v_n \\
        &\quad\,+\sum_{m=1}^M v_m \sum_{n=1}^M \frac{1}{M}\int c^2\sigma'(w \cdot x^m+\eta)\sigma'(w \cdot x^n+\eta) \left(x^m\cdot x^n+1\right) \, d\mu_0(c,w,\eta) v_n \\
        &\geq \frac{1}{M}\int \left(\sum_{m=1}^M \sigma(w \cdot x^m+\eta)v_m \right)^2 d\mu_0(c,w,\eta) \\
        &\geq0.
    \end{split}
    \end{equation}
    This verifies the non-negativity of the eigenvalues~$\lambda_{m}(B^*_0)$ of the limit standard NTK $B^*_0$.\\
    Let us further show that $v^\top B^*_0 v=0$ if and only if $v=0$.
    The ``if'' direction is immediate.
    For the ``only if'' direction, assume by contradiction that there was a vector $v\not=0$ with $v^\top B^*_0 v=0$.
    This would imply due to \eqref{eq:proof:lem:Ginfty_pos_def:1} that
    \begin{equation}
        \label{eq:proof:lem:Ginfty_pos_def:2}
        \sum_{m=1}^M \sigma(w \cdot x^m+\eta) v_m=0
        \quad\text{for all } w\in \bbR^d, \eta\in \bbR,
    \end{equation}
    since the distribution $\mu_0$ assigns positive probability to every set with positive Lebesgue measure as of Assumption~\ref{asm:NN_mu0}\ref{asm:NN_mu0iv} and continuity of the integrand w.r.t.\@ the NN parameters $w$ and $\eta$.
    Since the activation function $\sigma$ is non-polynomial and bounded (thus tempered as a distribution) as of Assumption~\ref{asm:NN_nonlinearity}\ref{asm:NN_sigma},
    and since the data samples~$x^m$ are in distinct directions for $m=1,\dots,M$ as of Assumption~\ref{asm:data_distinct},
    the functions $w\mapsto\sigma(w\cdot x^m+\eta)$, $m=1,\dots,M$, are linearly independent
    according to \cite[Remark~3.1]{ito1996nonlinearity}.
    Therefore, by the definition of linear independence, \eqref{eq:proof:lem:Ginfty_pos_def:2} implies $v_m=0$ for all $m=1,\dots,M$.
    Since this is a contradiction,
    $v^\top B^*_0 v>0$ for all $v\in\bbR^M$, ensuring that the eigenvalues~$\lambda_{m}(B^*_0)$ are indeed strictly positive.
\end{proof}

We further show that the operator norm of the limit standard NTK $B^*_0$ is bounded.

\begin{lemma}
    \label{lem:standardNTKB:operatornorm}
    The operator norm of the limit standard NTK $B^*_0$, as defined in \eqref{eq:standardNTK}, is bounded, i.e., it holds $\N{B^*_0}_{\mathrm{op}}\leq C$
    for a constant $C=C(\sigma,\CD,\mu_0)$.
\end{lemma}

\begin{proof}
    First recall from Lemma~\ref{lem:standardNTKB:spectrum} that the limit standard NTK $B^*_0$ is positive definite.
    Since its operator norm (being the largest eigenvalue) is bounded by the trace (being the sum of all eigenvalues),
    it holds
    \begin{equation}
        \label{eq:proof:lem:standardNTKB:operatornorm:1}
    \begin{split}
        \N{B^*_0}_{\mathrm{op}}
        \leq \mathrm{tr}\left(B^*_0\right) 
        &=\sum_{m=1}^M \left(B^*_0\right)_{mm} \\
        &= \sum_{m=1}^M \frac{1}{M} \int \left(\sigma(w \!\cdot\! x^m\!+\!\eta)\right)^2 + c^2\left(\sigma'(w \!\cdot\! x^m\!+\!\eta)\right)^2 \left(\N{x^m}_2^2\!+\!1\right) d\mu_0(c,w,\eta) \\
        &\leq \left(C_\sigma^2 + C_{\mu_0}^2C_{\sigma'}^2(C_x^2+1)\right),
    \end{split}
    \end{equation}
    where we used Assumptions~\ref{asm:NN_nonlinearity}\ref{asm:NN_sigma}, \ref{asm:NN_nonlinearity}\ref{asm:NN_sigma'}, and \ref{asm:NN_mu0},\ref{asm:NN_mu0ii} together with Assumption~\ref{asm:data} in the last step.
\end{proof}

From the eigenvalues of the limit standard NTK $B^*_0$,
we can directly infer the eigenvalues of the limit NNTK $\CB^*_0 = B^*_0\left(\gamma\Id_M + B^*_0 \right)^{-1}$, as defined in \eqref{eq:NNTK*}.

\begin{lemma}
    \label{lem:secondorderNTKB:spectrum}
    Let $\{\lambda_{m}(B^*_0)\}_{m=1}^M$ denote the eigenvalues of the limit standard NTK $B^*_0$.
    Then the eigenvalues of the limit NNTK $\CB^*_0$, as defined in \eqref{eq:NNTK*}, are given by $\{\lambda_{m}(\CB^*_0)\}_{m=1}^M$ with
    \begin{equation}
        \lambda_{m}(\CB^*_0) = \frac{\lambda_{m}(B^*_0)}{\gamma +\lambda_{m}(B^*_0)}
    \end{equation}
    for $m=1,\dots,M$.
\end{lemma}

\begin{proof}
    Let $(\lambda_{m}(B^*_0),v_m)$ denote an eigenpair of the standard NTK $B^*_0$.
    Noticing that $\left(\gamma \Id_M + B^*_0 \right)v_m = (\gamma +\lambda_{m}(B^*_0))v_m$ and hence $\left(\gamma \Id_M + B^*_0 \right)^{-1}v_m = \frac{1}{\gamma +\lambda_{m}(B^*_0)}v_m$, it holds
    \begin{equation}
    \begin{split}
        \CB^*_0 v_m
        &= B^*_0\left(\gamma \Id_M + B^*_0 \right)^{-1}v_m 
        = \frac{1}{\gamma +\lambda_{m}(B^*_0)}B^*_0 v_m = \frac{\lambda_{m}(B^*_0)}{\gamma +\lambda_{m}(B^*_0)} v_m,
    \end{split}
    \end{equation}
    proving that $\big(\frac{\lambda_{m}(B^*_0)}{\gamma +\lambda_{m}(B^*_0)},v_m\big)$ denotes the corresponding eigenpair of $\CB^*_0$.
\end{proof}

\begin{remark}
    \label{rem:secondorderNTKB:spectrum}
    Since
    \begin{equation}
        \lambda_{m}(\CB^*_0) = \frac{\lambda_{m}(B^*_0)}{\gamma +\lambda_{m}(B^*_0)} = 1-\frac{\gamma}{\gamma +\lambda_{m}(B^*_0)}
    \end{equation}
    is monotonically increasing in $\lambda_{m}(B^*_0)$,
    we have that $\lambda_{\max}(\CB^*_0) = \lambda_{\max}(B^*_0)/(\gamma+\lambda_{\max}(B^*_0))$ and analogously $\lambda_{\min}(\CB^*_0) = \lambda_{\min}(B^*_0)/(\gamma +\lambda_{\min}(B^*_0))$.
\end{remark}

\subsection{Proof of Theorem~\ref{thm:convergence}}
\label{sec:app:thm:convergence}

We first provide a detailed proof of the convergence in the infinite-width limit in Theorem~\ref{thm:convergence}.

\begin{proof}[Proof of Theorem~\ref{thm:convergence}]
    Let $r_k^*(\x) = \y - f_{k}^*(\x)\in\bbR^M$ be the residual at iteration $k$. We can obtain from \eqref{eq:NNfunction_evolution_limit} the recursion
    \begin{equation}
    r_{k+1}^*(\x)
        =\y \!-\! f_{k+1}^*(\x)
        = \left(\Id_M \!-\! \alpha \CB^*_0\right) \left(\y \!-\! f_k^*(\x)\right)
        = \left(\Id_M \!-\! \alpha \CB^*_0\right) r_{k}^*(\x)
    \end{equation}
    and hence $r_{k}^*(\x) = \left(\Id_M - \alpha \CB^*_0\right)^{k} r_{0}^*(\x)$.
    Using
    Lemma~\ref{lem:secondorderNTKB:spectrum} and Remark~\ref{rem:secondorderNTKB:spectrum}, we can bound for a step size $\alpha\in\left(0,2/\!\left(\lambda_{\max}(\CB^*_0)+\lambda_{\min}(\CB^*_0)\right)\right)$ the norm of the residual as
    \begin{equation*}
    \begin{split}
        \N{r_{k}^*(\x)}_2
        &\leq \N{\Id_M \!-\! \alpha \CB^*_0}_{\mathrm{op}}^{k}\! \N{r_{0}^*(\x)}_2 \\
        &\leq \max\left\{\abs{\lambda_{\min}\!\left(\Id_M \!-\! \alpha \CB^*_0\right)},\abs{\lambda_{\max}\!\left(\Id_M \!-\! \alpha \CB^*_0\right)}\right\}^k \!\N{r_{0}^*(\x)}_2 \\
        &= \max\{\abs{1 - \alpha \lambda_{\min}\!\left(\CB^*_0\right)}, \abs{1 - \alpha \lambda_{\max}\left(\CB^*_0\right)}\}^k \N{r_{0}^*(\x)}_2 \\
        &\leq \abs{1 - \alpha \lambda_{\min}\left(\CB^*_0\right)}^k \N{r_{0}^*(\x)}_2,
    \end{split}
    \end{equation*}
    where the last step uses the condition on $\alpha$.
    This concludes the proof since $r_{0}^*(\x) = \y$.
    To see the last inequality in the proof, notice that
    $1 - \alpha \lambda_{\max} \leq 1 - \alpha \lambda_{\min}$ since $\lambda_{\max}\geq\lambda_{\min}$.
    Due to the condition on $\alpha$ it holds
    $\alpha \lambda_{\max} + \alpha \lambda_{\min} < 2$ and 
    hence $1 - \alpha \lambda_{\max} > -(1 - \alpha \lambda_{\min})$.
    Combining those inequalities shows
    $-(1 - \alpha \lambda_{\min}) < 1 - \alpha \lambda_{\max} \leq 1 - \alpha \lambda_{\min}$.
    Multiplying it with $-1$ and rearranging it further shows $-(1 - \alpha \lambda_{\min}) \leq -(1 - \alpha \lambda_{\max}) < 1 - \alpha \lambda_{\min}$.
    Moreover, the condition on $\alpha$ and $\lambda_{\max}\geq\lambda_{\min}>0$ imply $\alpha\lambda_{\min}<1$, and hence $1-\alpha\lambda_{\min}>0$.
    Hence,
    $\abs{1 - \alpha \lambda_{\max}} \leq \abs{1 - \alpha \lambda_{\min}}$.
\end{proof}

\subsection{Proof of Lemma~\ref{prop:NTK_limit_bounds_z}}
\label{sec:app:NTK_limit_bounds_z}

We can now give a detailed proof of the central auxiliary result, Lemma~\ref{prop:NTK_limit_bounds_z}.

\begin{proof}[Proof of Lemma~\ref{prop:NTK_limit_bounds_z}]
    Using an inductive argument,
    we derive, for a finite number of training steps $k=0,\dots,K$, a priori bounds for the Newton updates~\eqref{eq:regularizedNewtonStep}, the NN function~\eqref{eq:NN}, the residual, and the loss~\eqref{eq:loss}, which are uniform in the number of neurons~$N$ and hold with high probability.
    We further prove that, with high probability, the Newton updates $z^N_k$ in \eqref{eq:regularizedNewtonStep} vanish as $N\rightarrow\infty$.\\
    \textit{Step 1: Definition of the high-probability set.}
    Since, as of Assumption~\ref{asm:NN_mu0}\ref{asm:NN_mu0ii}, $\abs{c^i_0}\leq C_{c,0}(\mu_0)$ for all $i=1,\dots,N$,
    Lemma~\ref{lem:estimate_NNfunction_E} shows that $\bbE\N{f_{\theta_0}^N(\x)}_2^2 = \sum_{m=1}^M \bbE\abs{f_{\theta_0}^N(x^m)}^2 \leq \frac{MC(\sigma,\mu_0)}{N^{2\beta-1}}$.
    By Markov's inequality,
    we thus have that for any $C_{f,0}>0$ it holds
    \begin{equation}
        \label{eq:proof:thm:NTK_limit:whp1}
        \bbP\left(\N{f_{\theta_0}^N(\x)}_2 \geq \sqrt{M}C_{f,0} \right)
        \leq \frac{\bbE \N{f_{\theta_0}^N(\x)}_2^2}{MC_{f,0}^2}
        = \frac{C(\sigma,\mu_0)}{C_{f,0}^2}\frac{1}{N^{2\beta-1}} 
    \end{equation}
    and hence, since $\N{\y-f_{\theta_0}^N(\x)}_2 \leq \N{\y}_2 + \N{f_{\theta_0}^N(\x)}_2 \leq \sqrt{M}C_y + \N{f_{\theta_0}^N(\x)}_2$ with Assumption~\ref{asm:data} used in the last step, also for any $C_{R,0}>C_y$ that
    \begin{equation}
        \label{eq:proof:thm:NTK_limit:whp2}
    \begin{split}
        \bbP\left(L^N(\theta_0)
        \leq \frac{C_{R,0}^2}{2} \right)
        &= \bbP\left(\N{\y-f_{\theta_0}^N(\x)}_2 \leq \sqrt{M}C_{R,0} \right) \\
        &\geq \bbP\left(\N{\y}_2 + \N{f_{\theta_0}^N(\x)}_2 \leq \sqrt{M}C_{R,0} \right) \\
        &\geq \bbP\left(\N{f_{\theta_0}^N(\x)}_2 \leq \sqrt{M}\left(C_{R,0}-C_y\right) \right) \\
        &\geq 1- \frac{C(\sigma,\mu_0)}{\left(C_{R,0}-C_y\right)^2}\frac{1}{N^{2\beta-1}} \geq 1-\delta,
    \end{split}
    \end{equation}
    where the last step is by choice of $N$ (since by assumption, $N^{2\beta-1}\geq \frac{C(\sigma,\mu_0)}{(C_{R,0}-C_y)^2}\frac1\delta$).
    Let us now denote by $\Omega_R$ the event where it holds $\N{\y-f_{\theta_0}^N(\x)}_2 \leq \sqrt{M} C_{R,0}$.
    Conditionally on this set, all subsequent computations can be performed.
    Due to \eqref{eq:proof:thm:NTK_limit:whp2}, $\bbP\left(\Omega_R\right) \geq 1-\delta$. \\
    \textit{Step 2: Initialization (induction start $k=0$).}
    By the choice of $N$ (since by assumption,  $\gamma \geq C_\gamma \frac{1}{N^{1-\beta}}$) and using that on the set $\Omega_R$ it holds $\sqrt{2L^N(\theta_0)} = \frac{1}{\sqrt{M}} \N{\y-f_{\theta_0}^N(\x)}_2 \leq C_{R,0}$,
    we have
    \begin{equation}
        \label{eq:proof:thm:NTK_limit:condition_gamma0}
    \begin{split}
        \gamma>\frac\gamma2 = \frac12 \gamma
        &\geq \frac12 \sqrt{2}C_{R,0}C_S(\sigma,\CD,\mu_0)\frac{1}{N^{1-\beta}}\\
        &\geq \frac{C_S(\sigma,\CD,\mu_0)}{N^{1-\beta}} \sqrt{L^N(\theta_0)}.
    \end{split}
    \end{equation}
    Therefore, Lemma~\ref{lem:estimate_Dinv_operatornorm} applies and ensures that $D_{\theta_0}^N$ is positive definite (hence invertible)
    with
    \begin{equation}
        \label{eq:proof:thm:NTK_limit:bound_D0inv}
    \begin{split}
        \N{\left(D_{\theta_0}^N\right)^{-1}}_{\mathrm{op}}
        &\leq \frac{N^{\beta}}{\gamma N^{1-\beta} - C_S(\sigma,\CD,\mu_0) \sqrt{L^N(\theta_0)}} \\
        &= \frac{N^{\beta}}{\frac12\gamma N^{1-\beta} + \left(\frac12\gamma N^{1-\beta} - C_S(\sigma,\CD,\mu_0) \sqrt{L^N(\theta_0)}\right)} \\
        &\leq \frac{2N^{\beta}}{\gamma N^{1-\beta}} = \frac{2}{\gamma}N^{2\beta-1},
    \end{split}
    \end{equation}
    where the penultimate inequality follows from $\frac12\gamma 
    \geq\frac{C_S(\sigma,\CD,\mu_0)}{N^{1-\beta}} \sqrt{L^N(\theta_0)}$, as established in \eqref{eq:proof:thm:NTK_limit:condition_gamma0}.
    Leveraging this, we can establish neuron-wise bounds on the norm of the Newton update $z^N_0$, i.e., bounds on $(z^N_0)^i$ for all $i=1,\dots,N$.
    Using Formula~\eqref{eq:zN_reformulation} together with the observation that the neuron-wise updates separate due to the block-diagonal structure of the matrix $D_{\theta_0}^N$ (and hence its inverse) in the first step, we compute
    \begin{allowdisplaybreaks}
    \begin{align}
        \N{(z^N_0)^i}_2^2
        &=  \N{\frac{1}{M}\left(D_{\theta_0^i}^N\right)^{-1} \left(J_{\theta_0^i}^N\right)^\top \zeta_0^N}_2^2 \notag\\
        &\leq \frac{1}{M^2}\N{\left(D_{\theta_0^i}^N\right)^{-1}}_{\mathrm{op}}^2\N{ \left(J_{\theta_0^i}^N\right)^\top \zeta_0^N}_2^2
        = \frac{1}{M^2}\N{\left(D_{\theta^i_0}^N\right)^{-1}}_{\mathrm{op}}^2\sum_{j=1}^{d+2} \abs{\left(J^N_{\theta^i_0}\right)_{:,j}^\top\zeta^N_0}_2^2 \notag\\
        &\leq \frac{1}{M^2}\N{\left(D_{\theta^i_0}^N\right)^{-1}}_{\mathrm{op}}^2\sum_{j=1}^{d+2} \N{\left(J^N_{\theta^i_0}\right)_{:,j}}_2^2\N{\zeta^N_0}_2^2 \notag\\
        &= \frac{1}{M^2}\N{\left(D_{\theta_0^i}^N\right)^{-1}}_{\mathrm{op}}^2\sum_{j=1}^{d+2} \N{\left(J^N_{\theta^i_0}\right)_{:,j}}_2^2\cdot \notag\\
        &\qquad\quad\, \cdot\N{\left(\Id_M + \frac{1}{M}J_{\theta_0}^N\left(D_{\theta_0}^N\right)^{-1}\left(J_{\theta_0}^N\right)^\top \right)^{-1}\left(\y - f_{\theta_0}^N(\x)\right)}_2^2 \notag\\
        &\leq \frac{1}{M^2} \N{\left(D_{\theta_0^i}^N\right)^{-1}}_{\mathrm{op}}^2\sum_{j=1}^{d+2} \N{\left(J^N_{\theta^i_0}\right)_{:,j}}_2^2\N{\y - f_{\theta_0}^N(\x)}_2^2 \notag\\
        &\leq \frac{C(\sigma,\CD,\mu_0)(d+2)}{M} \N{\left(D_{\theta_0^i}^N\right)^{-1}}_{\mathrm{op}}^2\N{\y - f_{\theta_0}^N(\x)}_2^2 \frac{1}{N^{2\beta}} \notag\\
        &\leq \frac{C(\sigma,\CD,\mu_0)(d+2)}{M\!\left(\gamma N^{1-\beta}\right)^2} \N{\y - f_{\theta_0}^N(\x)}_2^2 \notag\\
        &\leq \frac{C(\sigma,\CD,\mu_0,C_{R,0})}{\left(\gamma N^{1-\beta}\right)^2},
        \label{eq:proof:thm:NTK_limit:bound_z0^2}
    \end{align}
    \end{allowdisplaybreaks}%
    where the individual steps hold as explained in what follows.
    The inequality in the second line uses the definition of the operator norm.
    The step thereafter merely writes out the $\ell_2$-norm of the $(d+2)$-dimensional vector $(J_{\theta_0^i}^N)^\top \zeta_0^N$ before using Cauchy-Schwarz inequality for each of its entries in the third line.
    In the fourth line (split across two lines), we use the definition~\eqref{eq:zN_reformulation} of $\zeta^N_0$.
    The estimate thereafter follows noting that $J_{\theta_0}^N\left(D_{\theta_0}^N\right)^{-1}\left(J_{\theta_0}^N\right)^\top$ is positive semi-definite due to $\left(D_{\theta_0}^N\right)^{-1}$ being positive definite as of Lemma~\ref{lem:estimate_Dinv_operatornorm}.
    The third-to-last line uses that, as of Lemma~\ref{lem:estimate_J_entries}, each entry of $J^N_{\theta_0^i}\in\bbR^{M\times (d+2)}$ is of order $\mathcal{O}(\frac{1}{N^{\beta}})$ and hence $\Nbig{\big(J^N_{\theta^i_0}\big)_{:,j}}_2^2 \leq MC(\sigma,\CD,\mu_0)\frac{1}{N^{2\beta}}$, since $\big(J^N_{\theta^i_0}\big)_{:,j}$ is an $M$-dimensional vector.
    Eventually, the penultimate line follows from \eqref{eq:proof:thm:NTK_limit:bound_D0inv}, before using in the last inequality that $\N{\y-f_{\theta_0}^N(\x)}_2 \leq \sqrt{M} C_{R,0}$ holds.
    From \eqref{eq:proof:thm:NTK_limit:bound_z0^2}, we conclude that for all $i=1,\dots,N$ it holds
    \begin{equation}
        \label{eq:proof:thm:NTK_limit:bound_z0}
        \N{(z^N_0)^i}_2
        \leq \frac{C_{z,0}(\sigma,\CD,\mu_0,C_{R,0})}{\gamma N^{1-\beta}}.
    \end{equation}
    This proves the boundedness of the Newton updates for the initialization step $k=0$.

    \textit{Step 3: Iteration step (induction $k\mapsto k+1$).}
    The remainder of the proof goes by induction.
    
    Induction assumptions.
    Let us assume that it holds $\abs{c^i_k}\leq C_{c,k}(\alpha,\gamma,\sigma,\CD,\mu_0,C_{c,k-1},C_{z,k-1})$ for all $i=1,\dots,N$,
    that $\Nbig{\y-f_{\theta_{k}}^N(\x)}_2\leq \sqrt{M}C_{R,k}(\alpha,\gamma,\sigma,\CD,C_{c,k-1},C_{R,k-1},C_{z,k-1})$,
    and that it holds
    \begin{equation}
        \label{eq:proof:thm:NTK_limit:bound_zk_inductionassumption}
        \N{(z^N_{k})^i}_2
        \leq \frac{C_{z,k}(\sigma,\CD,\mu_0,C_{c,k},C_{R,k})}{\gamma N^{1-\beta}}
    \end{equation}
    for all $i=1,\dots,N$.
    As we have proven in Step~2,
    this is certainly true for the induction start~$k=0$ as we have proven in the previous step with $C_{c,0}=C_{c,0}(\mu_0)$, $C_{R,0}$, and $C_{z,0}$ as in \eqref{eq:proof:thm:NTK_limit:bound_z0}.

    Induction step. We first show that the NN parameters~$c^i$ remain bounded.
    With triangle inequality, the NN parameter update~\eqref{eq:regularizedNewton}, and the induction assumption~\eqref{eq:proof:thm:NTK_limit:bound_zk_inductionassumption} in the next-to-last step,
    we can bound
    \begin{equation}
    \begin{split}
        \abs{c^i_{k+1}}
        \leq \abs{c^i_k} \!+\! \abs{c^i_{k+1}-c^i_k}
        &\leq \abs{c^i_k} \!+\! \alpha\N{(z^N_k)^i}_2 \\
        &\leq C_{c,k}(\alpha,\gamma,\sigma,\CD,\mu_0,C_{c,k-1},C_{z,k-1}) \!+\! \frac{\alpha C_{z,k}(\sigma,\CD,\mu_0,C_{c,k},C_{R,k})}{\gamma N^{1-\beta}} \\
        &= C_{c,k+1}(\alpha,\gamma,\sigma,\CD,\mu_0,C_{c,k},C_{z,k}),
    \end{split}
    \end{equation}
    where $C_{c,k+1}$ is defined implicitly in the last step after using that $N\geq1$ and $\beta\in(1/2,1)$.
    Next, we show the boundedness of the NN function.
    With triangle inequality in the second line and using the boundedness and Lipschitz continuity of the activation function~$\sigma$ as of Assumptions~\ref{asm:NN_nonlinearity}\ref{asm:NN_sigma} together with Cauchy-Schwarz inequality in the third line,
    it is straightforward to establish by using that the training data is bounded due to Assumption~\ref{asm:data} the estimate
    \begin{allowdisplaybreaks}
    \begin{align}
        &\abs{f_{\theta_{k+1}}^N(x^m) - f_{\theta_{k}}^N(x^m)}
        = \abs{\frac{1}{N^\beta} \sum_{i=1}^N c_{k+1}^i \sigma(w_{k+1}^i \cdot x^m + \eta_{k+1}^i) - \frac{1}{N^\beta} \sum_{i=1}^N c_{k}^i \sigma(w_{k}^i \cdot x^m + \eta_{k}^i)} \notag\\
        &\qquad\quad\,\leq \frac{1}{N^\beta} \!\sum_{i=1}^N \abs{c_{k+1}^i\!-\!c_{k}^i} \!\abs{\sigma(w_{k+1}^i \!\cdot\! x^m \!+\! \eta_{k+1}^i)} \!+\! \abs{c_{k}^i} \!\abs{\sigma(w_{k+1}^i \!\cdot\! x^m \!+\! \eta_{k+1}^i)\!-\!\sigma(w_{k}^i \!\cdot\! x^m \!+\!\eta_{k}^i)} \notag\\
        &\qquad\quad\,\leq \frac{1}{N^\beta} \sum_{i=1}^N C_\sigma \abs{c_{k+1}^i-c_{k}^i}  + L_\sigma \abs{c_{k}^i} \left( \N{w_{k+1}^i-w_{k}^i}_2 \N{x^m}_2 + \abs{\eta_{k+1}^i - \eta_{k}^i}\right) \notag\\
        &\qquad\quad\,\leq \frac{\alpha C(\sigma,\CD)}{N^\beta} \sum_{i=1}^N \left(1+\abs{c_{k}^i}\right)\N{(z^N_{k})^i}_2 \notag\\
        &\qquad\quad\,\leq \frac{\alpha C(\sigma,\CD,C_{c,k})}{N^\beta} \sum_{i=1}^N \N{(z^N_{k})^i}_2 \leq \frac{\alpha C(\sigma,\CD,C_{c,k})}{N^\beta} \sum_{i=1}^N \frac{C_{z,k}(\sigma,\CD,\mu_0,C_{c,k},C_{R,k})}{\gamma N^{1-\beta}} \notag\\
        &\qquad\quad\,= \frac{\alpha C(\sigma,\CD,C_{c,k},C_{z,k})}{\gamma},
    \end{align}
    \end{allowdisplaybreaks}%
    where we used the induction assumptions to obtain the bounds in the next-to-last line.
    Hence,
    \begin{equation}
        \label{eq:proof:thm:NTK_limit:bound_res_inductionstep}
    \begin{split}
        \Nbig{\y-f_{\theta_{k+1}}^N(\x)}_2
        &\leq \Nbig{\y-f_{\theta_{k}}^N(\x)}_2 + \Nbig{f_{\theta_{k+1}}^N(\x)-f_{\theta_{k}}^N(\x)}_2 \\
        &\leq \sqrt{M}C_{R,k}(\alpha,\gamma,\sigma,\CD,C_{c,k-1},C_{R,k-1},C_{z,k-1}) \\
        &\quad\,+ \sqrt{M}\frac{\alpha C(\sigma,\CD,C_{c,k},C_{z,k})}{\gamma} \\
        &= \sqrt{M}C_{R,k+1}(\alpha,\gamma,\sigma,\CD,C_{c,k},C_{R,k},C_{z,k}),
    \end{split}
    \end{equation}
    where the second step is once more due to the induction assumption and
    where we define $C_{R,k+1}$ implicitly in the last step.
    It remains to bound the Newton update $z^N_{k+1}$.
    By the choice of $N$ and using that with \eqref{eq:proof:thm:NTK_limit:bound_res_inductionstep} it holds $\sqrt{2L^N(\theta_{k+1})} = \frac{1}{\sqrt{M}} \Nbig{\y-f_{\theta_{k+1}}^N(\x)}_2 \leq C_{R,k+1}(\alpha,\gamma,\sigma,\CD,C_{c,k},C_{R,k},C_{z,k})$, we have
    \begin{equation}\label{eq:Gamma_bound}
    \begin{split}
        \gamma>\frac\gamma2=\frac12\gamma
        &\geq \frac12\sqrt{2} C_{R,k+1}(\alpha,\gamma,\sigma,\CD,C_{c,k},C_{R,k},C_{z,k})C_S(\sigma,\CD,C_{c,k+1})\frac{1}{N^{1-\beta}} \\
        &\geq \frac{C_S(\sigma,\CD,C_{c,k+1})}{N^{1-\beta}}\sqrt{L^N(\theta_{k+1})}.
    \end{split}
    \end{equation}
    Therefore, Lemma~\ref{lem:estimate_Dinv_operatornorm} applies and ensures that $D_{\theta_{k+1}}^N$ is positive definite (hence invertible)
    with
    \begin{equation}
        \label{eq:proof:thm:NTK_limit:bound_Dkinv}
    \begin{split}
        \N{\left(D_{\theta_{k+1}}^N\right)^{-1}}_{\mathrm{op}}
        &\leq \frac{N^{\beta}}{\gamma N^{1-\beta} - C_S(\sigma,\CD,C_{c,k+1}) \sqrt{L^N(\theta_{k+1})}} \\
        &= \frac{N^{\beta}}{\frac12\gamma N^{1-\beta} + \left(\frac12\gamma N^{1-\beta} - C_S(\sigma,\CD,C_{c,k+1})  \sqrt{L^N(\theta_{k+1})}\right)} \\
        &\leq \frac{2N^{\beta}}{\gamma N^{1-\beta}}
        = \frac{2}{\gamma}N^{2\beta-1},
    \end{split}
    \end{equation}
    where the penultimate inequality follows from $\frac12{\gamma}\geq\frac{C_S(\sigma,\CD,C_{c,k+1})}{N^{1-\beta}}\sqrt{L^N(\theta_{k+1})}$, as established in \eqref{eq:Gamma_bound}.
    Leveraging this, we can now establish neuron-wise bounds on the norm of the Newton update $z^N_{k+1}$, i.e., bounds on $(z^N_{k+1})^i$ for all $i=1,\dots,N$.
    Using the same arguments as in \eqref{eq:proof:thm:NTK_limit:bound_z0^2} and noticing in particular that $J_{\theta_{k+1}}^N\big(D_{\theta_{k+1}}^N\big)^{-1}\big(J_{\theta_{k+1}}^N\big)^\top$ is positive semi-definite due to $\big(D_{\theta_{k+1}}^N\big)^{-1}$ being positive definite as of Lemma~\ref{lem:estimate_Dinv_operatornorm},
    we compute
    \begin{allowdisplaybreaks}
    \begin{align}
        \N{(z^N_{k+1})^i}_2^2
        &\leq \frac{1}{M^2}\N{\left(D_{\theta_{k+1}^i}^N\right)^{-1}}_{\mathrm{op}}^2\sum_{j=1}^{d+2} \N{\left(J^N_{\theta^i_{k+1}}\right)_{:,j}}_2^2\N{\zeta_{k+1}^N}_2^2 \notag\\
        &= \frac{1}{M^2}\N{\left(D_{\theta_{k+1}^i}^N\right)^{-1}}_{\mathrm{op}}^2\sum_{j=1}^{d+2} \N{\left(J^N_{\theta^i_{k+1}}\right)_{:,j}}_2^2\cdot \notag\\
        &\qquad\quad\, \cdot\N{\left(\Id_M + \frac{1}{M}J_{\theta_{k+1}}^N\left(D_{\theta_{k+1}}^N\right)^{-1}\left(J_{\theta_{k+1}}^N\right)^\top \right)^{-1}\left(\y - f_{\theta_{k+1}}^N(\x)\right)}_2^2 \notag\\
        &\leq \frac{1}{M^2} \N{\left(D_{\theta_{k+1}^i}^N\right)^{-1}}_{\mathrm{op}}^2\sum_{j=1}^{d+2} \N{\left(J^N_{\theta^i_{k+1}}\right)_{:,j}}_2^2\N{\y - f_{\theta_{k+1}}^N(\x)}_2^2 \notag\\
        &\leq \frac{C(\sigma,\CD,C_{c,k+1})(d+2)}{M} \N{\left(D_{\theta_{k+1}^i}^N\right)^{-1}}_{\mathrm{op}}^2\N{\y - f_{\theta_{k+1}}^N(\x)}_2^2 \frac{1}{N^{2\beta}} \notag\\
        &\leq \frac{C(\sigma,\CD,C_{c,k+1})}{M\!\left(\gamma N^{1-\beta}\right)^2}\N{\y - f_{\theta_{k+1}}^N(\x)}_2^2 \notag\\
        &\leq \frac{C(\sigma,\CD,C_{c,k+1},C_{R,k+1})}{\left(\gamma N^{1-\beta}\right)^2}.\label{eq:proof:thm:NTK_limit:bound_zk^2}
    \end{align}
    \end{allowdisplaybreaks}%
    From \eqref{eq:proof:thm:NTK_limit:bound_zk^2}, we have that for all $i=1,\dots,N$ it holds
    \begin{equation}
        \label{eq:proof:thm:NTK_limit:bound_zk}
        \N{(z^N_{k+1})^i}_2
        \leq \frac{C_{z,k+1}(\sigma,\CD,C_{c,k+1},C_{R,k+1})}{\gamma N^{1-\beta}},
    \end{equation}
    where $C_{z,k+1}$ is defined implicitly in the last step.
    This concludes the induction proof.

    \textit{Step 4: Conclusions.}
    \textit{(a) A priori bounds.}
    Let us now define in retrospect the quantities
    \begin{equation}
        \label{eq:proof:thm:NTK_limit:aprioribounds}
        C_c \coloneqq \sup_{k=0,\dots,K} C_{c,k},
        \quad
        C_R \coloneqq \sup_{k=0,\dots,K} C_{R,k},
        \quad\text{and}\quad
        C_z \coloneqq \sup_{k=0,\dots,K} C_{z,k}.
    \end{equation}
    Clearly, by tracing back these quantities through the iterations, we obtain the parameter dependencies
    $C_c=C_c(\alpha,\gamma,\sigma,\CD,\mu_0,C_{R,0},K)$,
    $C_R=C_R(\alpha,\gamma,\sigma,\CD,\mu_0,C_{R,0},K)$, as well as 
    $C_z=C_z(\alpha,\gamma,\sigma,\CD,\mu_0,C_{R,0},K)$.
    With the definition of these constants,
    it holds conditionally on the set $\Omega_R$ that
    \begin{equation}
        \label{eq:proof:thm:NTK_limit:aprioribounds_c}
    \begin{split}
        \abs{c^i_k}\leq C_c \quad\text{ for all }i=1,\dots,N \text{ and all } k=0,\dots,K
    \end{split}
    \end{equation}
    as well as
    \begin{equation}
        \label{eq:proof:thm:NTK_limit:aprioribounds_res}
    \begin{split}
        \Nbig{\y-f_{\theta_{k}}^N(\x)}_2 \leq \sqrt{M}C_R \quad\text{ for all }k=0,\dots,K.
    \end{split}
    \end{equation}

    \textit{(b) Vanishing Newton updates.}
    Moreover, since \eqref{eq:proof:thm:NTK_limit:bound_z0} and \eqref{eq:proof:thm:NTK_limit:bound_zk} hold conditionally on the event $\Omega_R$, we further have on this set that 
    \begin{equation}
        \label{eq:proof:thm:NTK_limit:vanishingz}
    \begin{split}
        &\N{(z^N_k)^i}_2 \leq \frac{C_z(\alpha,\gamma,\sigma,\CD,\mu_0,C_{R,0},K)}{\gamma N^{1-\beta}} \quad\text{ for all }i=1,\dots,N \text{ and all } k=0,\dots,K.
    \end{split}
    \end{equation}

    \textit{(c) Summary.}
    Denoting by $\Omega_B$ the set where
    \begin{equation}
        \max_{\substack{i=1,\dots,N \\ k=0,\dots,K}}\abs{c^i_k}\leq C_c, \;\max_{k=0,\dots,K}\Nbig{\y-f_{\theta_{k}}^N(\x)}_2 \leq \sqrt{M}C_R, \;\text{ and} \; \max_{\substack{i=1,\dots,N \\ k=0,\dots,K}}\N{(z^N_k)^i}_2 \leq \frac{C_z}{\gamma N^{1-\beta}}
    \end{equation}
    hold, we have proven that
    \begin{equation}
    \begin{split}
        \bbP\left(\Omega_B \big|\Omega_R\right)
        =1,
    \end{split}
    \end{equation}
    since \eqref{eq:proof:thm:NTK_limit:aprioribounds_c}, \eqref{eq:proof:thm:NTK_limit:aprioribounds_res}, and \eqref{eq:proof:thm:NTK_limit:vanishingz} hold conditionally on the set $\Omega_R$.
    Recalling the definition of the set $\Omega_R$,
    it holds $\bbP\left(\Omega_R\right) \geq 1-\delta$ due to \eqref{eq:proof:thm:NTK_limit:whp2}.
    Hence, the statement follows from the law of total probability since
    \begin{equation}
    \begin{split}
        \bbP\left(\Omega_B\right)
        &= \bbP\left(\Omega_B \big|\Omega_R\right) \bbP\left(\Omega_R\right) + \bbP\left(\Omega_B \big|\Omega_R^c\right) \bbP\left(\Omega_R^c\right) \\
        &\geq 1-\delta,
    \end{split}
    \end{equation}
    concluding the proof.
\end{proof}

\subsection{Proof of Theorem~\ref{thm:NTK_limit}}
\label{sec:app:proof_NTK_limit}

Let us first prove the following auxiliary result.

\begin{lemma}
    \label{lem:differenceNNTK_finite_infinite}
    Let $K<\infty$ be a given number of training steps.
    Assume that the number of neurons~$N$ is large enough such that \eqref{eq:thm:asm:Nlargeenough} is satisfied.
    Then it holds that
    \begin{equation}
    \begin{split}
        &\bbE\left[\N{\CB^N_k - \CB^*_0}_{\mathrm{op}}\Big| \,\Omega_R\right]\\
        &\qquad\quad\,\leq \left(1+\frac{C(\sigma,\CD,\mu_0)}{\gamma}\right) \frac{C(\sigma,\CD,\mu_0,C_c,C_R,C_z)(K+1)}{\gamma}\left(\frac{1}{N^{1-\beta}}+\frac{\log{(2M)}}{N^{1/2}}\right)
    \end{split}
    \end{equation}
    for all $k=0,\dots,K$.
\end{lemma}

\begin{proof}
    We estimate the discrepancy between the finite-width NNTK and its infinite-width limit, i.e., the difference between $\CB^N_k$ and $\CB^*_0$ as in \eqref{eq:NNTK} and \eqref{eq:NNTK*}, respectively.

    By Lemma~\ref{prop:NTK_limit_bounds_z}, the a priori bounds $\abs{c^i_k}\leq C_c$ for all $i=1,\dots,N$ and all $k=0,\dots,K$, $\Nbig{\y-f_{\theta_{k}}^N(\x)}_2 \leq \sqrt{M}C_R$ for all $k=0,\dots,K$, as well as $\N{(z^N_k)^i}_2 \leq \frac{C_z}{\gamma N^{1-\beta}}$ for all $i=1,\dots,N$ and all $k=0,\dots,K$ hold conditionally on a set $\Omega_R$, which is such that $\bbP(\Omega_R)\geq1-\delta$, with constants $C_c, C_R, C_z$ as in \eqref{eq:proof:thm:NTK_limit:aprioribounds}, see also \eqref{eq:proof:thm:NTK_limit:aprioribounds_c}, \eqref{eq:proof:thm:NTK_limit:aprioribounds_res}, as well as \eqref{eq:proof:thm:NTK_limit:vanishingz}.
    Such set $\Omega_R$ may be as defined after \eqref{eq:proof:thm:NTK_limit:whp2} in the proof of Lemma~\ref{prop:NTK_limit_bounds_z}.
    All subsequent computations are performed conditionally on this set.\\
    By inserting a mixed term, we first note the equality
    \begin{equation}
        \label{eq:proof:lem:differenceNNTK_finite_infinite:1}
    \begin{split}
        &\CB^N_k - \CB^*_0\\
        &\quad\,=
        \frac{1}{MN}\!\sum_{i=1}^N s^N_{\theta_k^i} \!\left(d_{\theta^i_k}^N\right)^{-1} \!\!\left(s^N_{\theta_k^i}\right)^{\!\top}\!\left(\Id_M \!+\! \frac{1}{MN}\!\sum_{i=1}^N \!s^N_{\theta_k^i} \!\left(d_{\theta^i_k}^N\right)^{\!\!-1}\!\! \left(s^N_{\theta_k^i}\right)^{\!\top} \right)^{-1} \!\!-\! \frac{1}{\gamma}B^*_0\!\left(\Id_M \!+\! \frac{1}{\gamma}B^*_0 \right)^{\!-1} \\
        &\quad\,=
        \left(\frac{1}{MN}\sum_{i=1}^N s^N_{\theta_k^i} \left(d_{\theta^i_k}^N\right)^{-1}\left(s^N_{\theta_k^i}\right)^\top - \frac{1}{\gamma}B^*_0\right)\left(\Id_M \!+\! \frac{1}{MN}\sum_{i=1}^N \!s^N_{\theta_k^i} \!\left(d_{\theta^i_k}^N\right)^{\!\!-1}\!\! \left(s^N_{\theta_k^i}\right)^{\!\top} \right)^{-1} \\
        &\quad\,\quad\,+ \frac{1}{\gamma}B^*_0 \left(\left(\Id_M \!+\! \frac{1}{MN}\!\sum_{i=1}^N \!s^N_{\theta_k^i} \!\left(d_{\theta^i_k}^N\right)^{\!\!-1}\!\! \left(s^N_{\theta_k^i}\right)^{\!\top} \right)^{-1} \!\!-\!\left(\Id_M \!+\! \frac{1}{\gamma}B^*_0 \right)^{\!-1}\right).
    \end{split}
    \end{equation}
    Let us first recall that the standard limit NTK $B^*_0$ is positive-definite as of Lemma~\ref{lem:standardNTKB:spectrum} and that
    $\frac{1}{MN}\sum_{i=1}^N s^N_{\theta_k^i} \big(d_{\theta^i_k}^N\big)^{-1} \big(s^N_{\theta_k^i}\big)^\top = \frac{1}{M}J_{\theta_k}^N\left(D_{\theta_k}^N\right)^{-1}\left(J_{\theta_k}^N\right)^\top$ is positive semi-definite as elaborated on already in the proof of Lemma~\ref{prop:NTK_limit_bounds_z}, see \eqref{eq:proof:thm:NTK_limit:bound_Dkinv}.
    Using this together with the fact (based on the sub-multiplicativity of the operator norm) that it holds for arbitrary invertible matrices~$A$ and ~$\widetilde{A}$ that $\Nbig{A^{-1}-\widetilde{A}^{-1}}_{\mathrm{op}} \leq \Nbig{\widetilde{A}^{-1}}_{\mathrm{op}}\Nbig{A-\widetilde{A}}_{\mathrm{op}}\N{A^{-1}}_{\mathrm{op}}$ in the second inequality,
    we can derive the bound
    \begin{equation}
        \label{eq:proof:lem:differenceNNTK_finite_infinite:2}
    \begin{split}
        &\N{\CB^N_k - \CB^*_0}_{\mathrm{op}}\\
        &\quad\,\leq
        \N{\frac{1}{MN}\sum_{i=1}^N s^N_{\theta_k^i} \left(d_{\theta^i_k}^N\right)^{-1}\left(s^N_{\theta_k^i}\right)^\top - \frac{1}{\gamma}B^*_0}_{\mathrm{op}}\N{\left(\Id_M \!+\! \frac{1}{MN}\sum_{i=1}^N \!s^N_{\theta_k^i} \!\left(d_{\theta^i_k}^N\right)^{\!\!-1}\!\! \left(s^N_{\theta_k^i}\right)^{\!\top} \right)^{-1}}_{\mathrm{op}} \\
        &\quad\,\quad\,+ \N{\frac{1}{\gamma}B^*_0}_{\mathrm{op}} \N{\left(\Id_M \!+\! \frac{1}{MN}\!\sum_{i=1}^N \!s^N_{\theta_k^i} \!\left(d_{\theta^i_k}^N\right)^{\!\!-1}\!\! \left(s^N_{\theta_k^i}\right)^{\!\top} \right)^{-1} \!\!-\!\left(\Id_M \!+\! \frac{1}{\gamma}B^*_0 \right)^{\!-1}}_{\mathrm{op}} \\
        &\quad\,\leq
        \left(1+\frac{C(\sigma,\CD,\mu_0)}{\gamma}\right)\N{\frac{1}{MN}\sum_{i=1}^N s^N_{\theta_k^i} \left(d_{\theta^i_k}^N\right)^{-1}\left(s^N_{\theta_k^i}\right)^\top - \frac{1}{\gamma}B^*_0}_{\mathrm{op}},
    \end{split}
    \end{equation}
    where the bound in the last step holds due to Lemma~\ref{lem:standardNTKB:operatornorm}.
    It remains to estimate the last term in \eqref{eq:proof:lem:differenceNNTK_finite_infinite:2}, for which we conduct the following auxiliary computations.

    \textit{Auxiliary computations.}
    By inserting mixed terms and using triangle inequality in the first two steps,
    we can bound
    \begin{allowdisplaybreaks}
    \begin{align}
        &\N{\frac{1}{MN}\sum_{i=1}^N s^N_{\theta_k^i} \left(\gamma\left(d_{\theta^i_k}^N\right)^{-1}\right) \left(s^N_{\theta_k^i}\right)^\top - B^*_0}_{\mathrm{op}}\notag \\
        &\qquad\quad\,\leq \N{\frac{1}{MN}\sum_{i=1}^N s^N_{\theta_k^i} \left(\gamma\left(d_{\theta^i_k}^N\right)^{-1}-\Id_{d+2}\right) \left(s^N_{\theta_k^i}\right)^\top}_{\mathrm{op}} + \N{\frac{1}{MN}\sum_{i=1}^N s^N_{\theta_k^i} \left(s^N_{\theta_k^i}\right)^\top - B^*_0}_{\mathrm{op}} \notag\\
        &\qquad\quad\,\leq \N{\frac{1}{MN}\sum_{i=1}^N s^N_{\theta_k^i} \left(\gamma\left(d_{\theta^i_k}^N\right)^{-1}-\Id_{d+2}\right) \left(s^N_{\theta_k^i}\right)^\top}_{\mathrm{op}} \notag\\
        &\qquad\quad\,\quad\,+ \N{\frac{1}{MN}\sum_{i=1}^Ns^N_{\theta_k^i} \left(s^N_{\theta_k^i}\right)^\top - s^N_{\theta_0^i} \left(s^N_{\theta_0^i}\right)^\top}_{\mathrm{op}} + \N{\frac{1}{MN}\sum_{i=1}^N s^N_{\theta_0^i} \left(s^N_{\theta_0^i}\right)^\top - B^*_0}_{\mathrm{op}}.\label{eq:proof:thm:NTK_limit:discrepancy_10:aux:0}
    \end{align}
    \end{allowdisplaybreaks}%
    In what follows, we address and establish bounds for each of the three terms on the right-hand side of \eqref{eq:proof:thm:NTK_limit:discrepancy_10:aux:0} separately.
    For the first of which,
    we bound with Jensen's inequality and using that the operator norm is sub-multiplicative in the first step, the fact (based on the sub-multiplicativity of the operator norm) that $\Nbig{A^{-1}-\widetilde{A}^{-1}}_{\mathrm{op}} = \Nbig{\widetilde{A}^{-1}(\widetilde{A}-A)A^{-1}}_{\mathrm{op}} \leq \Nbig{\widetilde{A}^{-1}}_{\mathrm{op}}\Nbig{A-\widetilde{A}}_{\mathrm{op}}\N{A^{-1}}_{\mathrm{op}}$ in the second step, and using in the equality thereafter the definition of $d_{\theta^i_k}^N$ from \eqref{eq:def_d},
    that
    \begin{allowdisplaybreaks}
    \begin{align}
        &\N{\frac{1}{MN}\sum_{i=1}^N s^N_{\theta_k^i} \left(\gamma\left(d_{\theta^i_k}^N\right)^{-1}-\Id_{d+2}\right) \left(s^N_{\theta_k^i}\right)^\top}_{\mathrm{op}} \notag\\
        &\qquad\quad\,\leq \frac{1}{MN}\sum_{i=1}^N \N{s^N_{\theta_k^i}}_{\mathrm{op}}^2 \N{\gamma\left(d_{\theta^i_k}^N\right)^{-1}-\Id_{d+2}}_{\mathrm{op}} \notag\\
        &\qquad\quad\,\leq \frac{1}{MN}\sum_{i=1}^N \N{s^N_{\theta_k^i}}_{\mathrm{op}}^2 \N{\gamma\left(d_{\theta^i_k}^N\right)^{-1}}_{\mathrm{op}}\N{\frac{1}{\gamma}d_{\theta^i_k}^N-\Id_{d+2}}_{\mathrm{op}} \notag\\
        &\qquad\quad\,= \frac{1}{MN}\sum_{i=1}^N \N{s^N_{\theta_k^i}}_{\mathrm{op}}^2 \N{\gamma\left(d_{\theta^i_k}^N\right)^{-1}}_{\mathrm{op}}\N{\frac{1}{\gamma}\frac{1}{N^{1-\beta}} \frac{1}{M} \sum_{m=1}^M (y^m - f_{\theta_{k}}^N(x^m)) h_{\theta^i_k}^N(x^m)}_{\mathrm{op}} \notag\\
        &\qquad\quad\,\leq \frac{1}{MN}\sum_{i=1}^N \N{s^N_{\theta_k^i}}_{\mathrm{op}}^2 \N{\gamma\left(d_{\theta^i_k}^N\right)^{-1}}_{\mathrm{op}}\frac{1}{\gamma}\frac{1}{N^{1-\beta}} \frac{1}{M} \sum_{m=1}^M \abs{y^m - f_{\theta_{k}}^N(x^m)} \N{h_{\theta^i_k}^N(x^m)}_{\mathrm{op}} \notag\\
        &\qquad\quad\,\leq \frac{1}{MN}\sum_{i=1}^N \N{s^N_{\theta_k^i}}_{\mathrm{op}}^2 \N{\gamma\left(d_{\theta^i_k}^N\right)^{-1}}_{\mathrm{op}}\frac{1}{\gamma}\frac{1}{N^{1-\beta}}  \frac{1}{\sqrt{M}}\Nbig{\y-f_{\theta_{k}}^N(\x)}_2\cdot \notag\\
        &\qquad\quad\,\quad\qquad\qquad\quad\,\cdot\left(\frac{1}{M}\sum_{m=1}^M \N{h_{\theta^i_k}^N(x^m)}_{\mathrm{op}}^2\right)^{1/2} \notag\\
        &\qquad\quad\,\leq \frac{C(\sigma,\CD,C_c,C_R)}{N}\sum_{i=1}^N \N{\gamma\left(d_{\theta^i_k}^N\right)^{-1}}_{\mathrm{op}}\frac{1}{\gamma}\frac{1}{N^{1-\beta}} \notag\\
        &\qquad\quad\,\leq \frac{C(\sigma,\CD,C_c,C_R)}{\gamma}\frac{1}{N^{1-\beta}} \label{eq:proof:thm:NTK_limit:discrepancy_10:aux:1}
    \end{align}
    \end{allowdisplaybreaks}%
    where we used triangle and Cauchy-Schwarz inequality in the fourth and fifth step, respectively,
    before employing Lemmas~\ref{lem:bound_s} and \ref{lem:bound_h} in the sixth step
    while recalling the a priori bounds described at the beginning of the proof and valid due to Lemma~\ref{prop:NTK_limit_bounds_z} together with the computations being performed on the set $\Omega_R$.
    For the last step, note that
    by the choice of $N$ and using that it holds $\sqrt{2L^N(\theta_{k})} = \frac{1}{\sqrt{M}} \N{\y-f_{\theta_k}^N(\x)}_2 \leq C_{R}$ due to the a priori bounds described at the beginning of the proof and valid due to Lemma~\ref{prop:NTK_limit_bounds_z}, we have
    \begin{equation}
    \begin{split}
        \gamma
        >\frac\gamma2 = \frac12 \gamma
        &\geq \frac12\sqrt{2}C_RC_S(\sigma,\CD,C_c)\frac{1}{N^{1-\beta}} \\
        &\geq \frac{C_S(\sigma,\CD,C_c)}{N^{1-\beta}}\sqrt{L^N(\theta_{k})}
    \end{split}
    \end{equation}
    Therefore, Lemma~\ref{lem:estimate_Dinv_operatornorm_i} applies and ensures that $d_{\theta_k^i}^N$ is positive definite (hence invertible)
    with
    \begin{equation}
        \gamma \N{\left(d_{\theta_k^i}^N\right)^{-1}}_{\mathrm{op}}
        \leq \frac{\gamma}{\gamma - \frac{C_S}{N^{1-\beta}} \sqrt{L^N(\theta_k)}}
        = \frac{\gamma}{\frac{\gamma}{2} + \left(\frac{\gamma}{2}- \frac{C_S}{N^{1-\beta}} \sqrt{L^N(\theta_k)}\right)}
        \leq 2.
    \end{equation}
    For the second term on the right-hand side of \eqref{eq:proof:thm:NTK_limit:discrepancy_10:aux:0},
    by inserting a mixed term, using triangle inequality, and again the sub-multiplicativity of the operator norm,
    and by employing Lemmas~\ref{lem:bound_s} and \ref{lem:bound_s_Lipschitz} together with the a priori bounds described at the beginning of the proof and valid due to Lemma~\ref{prop:NTK_limit_bounds_z} to obtain the fifth step,
    we derive the estimate
    \begin{equation}
        \label{eq:proof:thm:NTK_limit:discrepancy_10:aux:2}
    \begin{split}
        &\N{\frac{1}{MN}\sum_{i=1}^N s^N_{\theta_k^i} \left(s^N_{\theta_k^i}\right)^\top - s^N_{\theta_0^i} \left(s^N_{\theta_0^i}\right)^\top}_{\mathrm{op}} \\
        &\qquad\quad\,= \N{\frac{1}{MN}\sum_{i=1}^N \left(s^N_{\theta_k^i}-s^N_{\theta_0^i}\right) \left(s^N_{\theta_k^i}\right)^\top + s^N_{\theta_0^i} \left(s^N_{\theta_k^i} - s^N_{\theta_0^i}\right)^\top }_{\mathrm{op}} \\
        &\qquad\quad\,\leq \frac{1}{MN}\sum_{i=1}^N \N{\left(s^N_{\theta_k^i}-s^N_{\theta_0^i}\right) \left(s^N_{\theta_k^i}\right)^\top + s^N_{\theta_0^i} \left(s^N_{\theta_k^i} - s^N_{\theta_0^i}\right)^\top }_{\mathrm{op}} \\
        &\qquad\quad\,\leq \frac{1}{MN}\sum_{i=1}^N \N{s^N_{\theta_k^i}-s^N_{\theta_0^i}}_{\mathrm{op}} \N{s^N_{\theta_k^i}}_{\mathrm{op}} + \N{s^N_{\theta_0^i}}_{\mathrm{op}} \N{s^N_{\theta_k^i} - s^N_{\theta_0^i}}_{\mathrm{op}} \\
        &\qquad\quad\,= \frac{1}{MN}\sum_{i=1}^N \left(\N{s^N_{\theta_k^i}}_{\mathrm{op}} + \N{s^N_{\theta_0^i}}_{\mathrm{op}}\right)\N{s^N_{\theta_k^i}-s^N_{\theta_0^i}}_{\mathrm{op}} \\
        &\qquad\quad\,\leq \frac{C(\sigma,\CD,\mu_0,C_c)}{N}\sum_{i=1}^N \Nbig{\theta_k^i-\theta_0^i}_2 \\
        &\qquad\quad\,\leq \frac{C(\sigma,\CD,\mu_0,C_c)}{N}\sum_{i=1}^N \sum_{\ell=0}^{k-1} \Nbig{\theta_{\ell+1}^i-\theta_{\ell}^i}_2\\
        &\qquad\quad\,= \frac{\alpha C(\sigma,\CD,\mu_0,C_c)}{N}\sum_{i=1}^N \sum_{\ell=0}^{k-1} \Nbig{(z^N_{\ell})^i}_2 \\
        &\qquad\quad\,\leq \frac{\alpha C(\sigma,\CD,\mu_0,C_c,C_z)k}{\gamma} \frac{1}{N^{1-\beta}}\\
        &\qquad\quad\,\leq \frac{\alpha C(\sigma,\CD,\mu_0,C_c,C_z)K}{\gamma} \frac{1}{N^{1-\beta}},
    \end{split}
    \end{equation}
    where we used Lemma \ref{lem:bound_s_Lipschitz} in the fifth step, a telescopic sum argument together with triangle inequalities in the sixth step,
    the definition of the Newton update \eqref{eq:regularizedNewton} in the step thereafter,
    before concluding in the next-to-last line
    by using that $\N{(z^N_{\ell})^i}_2 \leq \frac{C_z}{\gamma N^{1-\beta}}$ for all $i=1,\dots,N$ and all $\ell=0,\dots,K$ as described at the beginning of the proof and valid due to Lemma~\ref{prop:NTK_limit_bounds_z}.\\
    In order to tackle the last term on the right-hand side of \eqref{eq:proof:thm:NTK_limit:discrepancy_10:aux:0},
    let us introduce the matrix-valued random variables
    \begin{equation}
        Z^i = \frac{1}{M}s^N_{\theta_0^i} \big(s^N_{\theta_0^i}\big)^\top - B^*_0.
    \end{equation}
    With the NN parameters $\theta_0^i$ per neuron being as of Assumption~\ref{asm:NN_mu0} i.i.d.\@ according to the measure $\mu_0$,
    the random variables $Z^i$ are i.i.d.\@ symmetric random matrices.
    Since $\bbE \big[\frac{1}{M}s^N_{\theta_0^i} (s^N_{\theta_0^i})^\top\big] = B^*_0$, clearly, $\bbE Z^i = 0$.
    Moreover, with triangle inequality and the sub-multiplicativity of the operator norm,
    we can bound
    \begin{equation}
    \begin{split}
        \lambda_{\max}(Z^i)
        \leq \N{Z^i}_{\mathrm{op}}
        &= \N{\frac{1}{M}s^N_{\theta_0^i} \big(s^N_{\theta_0^i}\big)^\top - B^*_0}_{\mathrm{op}} \\
        &\leq \frac{1}{M}\Nbig{s^N_{\theta_0^i}}_{\mathrm{op}}^2 + \N{B^*_0}_{\mathrm{op}}
        \leq C(\sigma,\CD,\mu_0).
    \end{split}
    \end{equation}
    Here, we employed Lemma~\ref{lem:bound_s} (with $\widetilde{C}=C(\mu_0)$ since $\abs{c_0^i}\leq C(\mu_0)$ as of Assumption~\ref{asm:NN_mu0}\ref{asm:NN_mu0ii}) and Lemma~\ref{lem:standardNTKB:operatornorm} to obtain the estimate in the last step.
    With similar arguments,
    using triangle inequality and Jensen's inequality in the first,
    triangle inequality and the sub-multiplicativity of the operator norm in the next-to-last,
    and again Lemmas~\ref{lem:bound_s} and~\ref{lem:standardNTKB:operatornorm} in the last line,
    we can estimate
    \begin{equation}
    \begin{split}
        &\N{\sum_{i=1}^N\bbE(Z^i)^2}_{\mathrm{op}}
        \leq \sum_{i=1}^N \bbE\N{(Z^i)^2}_{\mathrm{op}}
        = N \bbE\N{(Z^1)^2}_{\mathrm{op}}
        = N \bbE\N{\left(\frac{1}{M}s^N_{\theta_0^1} \big(s^N_{\theta_0^1}\big)^\top - B^*_0\right)^2}_{\mathrm{op}} \\
        &\qquad\quad\,= N \bbE\N{\frac{1}{M^2}s^N_{\theta_0^1} \big(s^N_{\theta_0^1}\big)^\top s^N_{\theta_0^1} \big(s^N_{\theta_0^1}\big)^\top - \frac{1}{M}B^*_0 s^N_{\theta_0^1} \big(s^N_{\theta_0^1}\big)^\top - \frac{1}{M}s^N_{\theta_0^1} \big(s^N_{\theta_0^1}\big)^\top B^*_0 + B^*_0B^*_0}_{\mathrm{op}} \\
        &\qquad\quad\,\leq N \bbE\left[\frac{1}{M^2}\Nbig{s^N_{\theta_0^1}}_{\mathrm{op}}^4 + 2\frac{1}{M}\N{B^*_0}_{\mathrm{op}}\Nbig{s^N_{\theta_0^1}}_{\mathrm{op}}^2 + \N{B^*_0}_{\mathrm{op}}^2\right] \\
        &\qquad\quad\,\leq C(\sigma,\CD,\mu_0)N.
    \end{split}
    \end{equation}
    An application of the matrix Bernstein inequality (recall that $\bbE Z^i = 0$)~\cite[Theorem~6.1.1]{tropp2015introduction} now shows
    \begin{equation}
        \label{eq:proof:thm:NTK_limit:discrepancy_10:aux:3}
    \begin{split}
        \bbE\N{\frac{1}{MN}\sum_{i=1}^N s^N_{\theta_0^i} \left(s^N_{\theta_0^i}\right)^\top - B^*_0}_{\mathrm{op}}
        &= \bbE\N{\frac{1}{N}\sum_{i=1}^N Z^i}_{\mathrm{op}} \\
        &\leq \sqrt{2\frac{C(\sigma,\CD,\mu_0)}{N}\log{(2M)}} + \frac{1}{3}\frac{C(\sigma,\CD,\mu_0)}{N}\log{(2M)} \\
        &\leq C(\sigma,\CD,\mu_0)\frac{\log{(2M)}}{N^{1/2}}.
    \end{split}
    \end{equation}
    Inserting \eqref{eq:proof:thm:NTK_limit:discrepancy_10:aux:1}, \eqref{eq:proof:thm:NTK_limit:discrepancy_10:aux:2}, and \eqref{eq:proof:thm:NTK_limit:discrepancy_10:aux:3} into \eqref{eq:proof:thm:NTK_limit:discrepancy_10:aux:0} after taking the conditional expectation in \eqref{eq:proof:thm:NTK_limit:discrepancy_10:aux:0} w.r.t.\@ the set $\Omega_R$,
    allows us to prove
    \begin{allowdisplaybreaks}
    \begin{align}
        &\bbE\left[\N{\frac{1}{MN}\sum_{i=1}^N s^N_{\theta_k^i} \left(d_{\theta^i_k}^N\right)^{-1} \left(s^N_{\theta_k^i}\right)^\top - \frac{1}{\gamma}B^*_0}_{\mathrm{op}} \Bigg| \,\Omega_R\right] \notag\\
        &\qquad\quad\,\leq \bbE\left[\N{\frac{1}{MN}\sum_{i=1}^N s^N_{\theta_k^i} \left(\left(d_{\theta^i_k}^N\right)^{-1}-\frac{1}{\gamma}\Id\right) \left(s^N_{\theta_k^i}\right)^\top}_{\mathrm{op}}\Bigg| \,\Omega_R\right] \notag\\
        &\qquad\quad\,\quad\,+ \bbE\left[\N{\frac{1}{MN}\sum_{i=1}^N \frac{1}{\gamma} s^N_{\theta_k^i} \left(s^N_{\theta_k^i}\right)^\top - \frac{1}{\gamma}s^N_{\theta_0^i} \left(s^N_{\theta_0^i}\right)^\top}_{\mathrm{op}}\Bigg| \,\Omega_R\right] \notag\\
        &\qquad\quad\,\quad\,+ \bbE\left[\N{\frac{1}{MN}\sum_{i=1}^N \frac{1}{\gamma}s^N_{\theta_0^i} \left(s^N_{\theta_0^i}\right)^\top - \frac{1}{\gamma}B^*_0}_{\mathrm{op}}\Bigg| \,\Omega_R\right] \notag\notag\\
        &\qquad\quad\,\leq \frac{C(\sigma,\CD,C_c,C_R)}{\gamma}\frac{1}{N^{1-\beta}} + \frac{\alpha C(\sigma,\CD,\mu_0,C_c,C_z)K}{\gamma} \frac{1}{N^{1-\beta}} + \frac{C(\sigma,\CD,\mu_0)}{\gamma}\frac{\log{(2M)}}{N^{1/2}} \notag\\
        &\qquad\quad\,\leq \frac{C(\sigma,\CD,\mu_0,C_c,C_R,C_z)(K+1)}{\gamma}\left(\frac{1}{N^{1-\beta}}+\frac{\log{(2M)}}{N^{1/2}}\right).
        \label{eq:proof:thm:NTK_limit:discrepancy_10:aux:final}
    \end{align}
    \end{allowdisplaybreaks}%
    For the second step of \eqref{eq:proof:thm:NTK_limit:discrepancy_10:aux:final}, we recall that the computations leading to \eqref{eq:proof:thm:NTK_limit:discrepancy_10:aux:1} and \eqref{eq:proof:thm:NTK_limit:discrepancy_10:aux:2} were directly performed conditionally on the set $\Omega_R$.
    For the third term, on the other hand, we have used \eqref{eq:proof:thm:NTK_limit:discrepancy_10:aux:3} together with the fact that for a non-negative random variable~$X$ it holds $\bbE\left[X\big|\Omega_R\right] \leq \frac{1}{\bbP(\Omega_R)}\bbE X$, which follows from the law of total expectation
    $\bbE X = \bbE\left[X\big|\Omega_R\right]\bbP(\Omega_R) + \bbE\left[X\big|\Omega_R^c\right]\bbP(\Omega_R^c) \geq \bbE\left[X\big|\Omega_R\right]\bbP(\Omega_R)$ thanks to the non-negativity of $X$.
    Noting further that with $\delta\in(0,1/2]$ it holds $\bbP(\Omega_R)\geq 1-\delta \geq 1/2$, gives the bound.

    \textit{Conclusion.}
    After taking the conditional expectation in \eqref{eq:proof:lem:differenceNNTK_finite_infinite:2} w.r.t.\@ the set $\Omega_R$,
    we eventually arrive, using the auxiliary estimate \eqref{eq:proof:thm:NTK_limit:discrepancy_10:aux:final},
    at
    \begin{equation}
    \begin{split}
        &\bbE\left[\N{\CB^N_k - \CB^*_0}_{\mathrm{op}}\Big| \,\Omega_R\right]\\
        &\quad\,\leq
        \left(1+\frac{C(\sigma,\CD,\mu_0)}{\gamma}\right)\bbE\left[\N{\frac{1}{MN}\sum_{i=1}^N s^N_{\theta_k^i} \left(d_{\theta^i_k}^N\right)^{-1} \left(s^N_{\theta_k^i}\right)^\top - \frac{1}{\gamma}B^*_0}_{\mathrm{op}} \Bigg| \,\Omega_R\right] \\
        &\quad\,\leq \left(1+\frac{C(\sigma,\CD,\mu_0)}{\gamma}\right) \frac{C(\sigma,\CD,\mu_0,C_c,C_R,C_z)(K+1)}{\gamma}\left(\frac{1}{N^{1-\beta}}+\frac{\log{(2M)}}{N^{1/2}}\right),
    \end{split}
    \end{equation}
    which concludes the proof.
\end{proof}

We now have all necessary tools to provide a detailed proof of the convergence to the infinite-width limit in Theorem~\ref{thm:NTK_limit}.

\begin{proof}[Proof of Theorem~\ref{thm:NTK_limit}]
    We estimate the discrepancy between the finite-width regime and the infinite-width limit for the NN function, i.e., the difference between $f_{\theta}^N$ and $f^*$ as in \eqref{eq:NN} and \eqref{eq:NNfunction_evolution_limit}, respectively.

    By Lemma~\ref{prop:NTK_limit_bounds_z}, the a priori bounds $\abs{c^i_k}\leq C_c$ for all $i=1,\dots,N$ and all $k=0,\dots,K$, $\Nbig{\y-f_{\theta_{k}}^N(\x)}_2 \leq \sqrt{M}C_R$ for all $k=0,\dots,K$, as well as $\N{(z^N_k)^i}_2 \leq \frac{C_z}{\gamma N^{1-\beta}}$ for all $i=1,\dots,N$ and all $k=0,\dots,K$ hold conditionally on a set $\Omega_R$, which is such that $\bbP(\Omega_R)\geq1-\delta$, with constants $C_c, C_R, C_z$ as in \eqref{eq:proof:thm:NTK_limit:aprioribounds}, see also \eqref{eq:proof:thm:NTK_limit:aprioribounds_c}, \eqref{eq:proof:thm:NTK_limit:aprioribounds_res}, as well as \eqref{eq:proof:thm:NTK_limit:vanishingz}.
    Such set $\Omega_R$ may be as defined after \eqref{eq:proof:thm:NTK_limit:whp2} in the proof of Lemma~\ref{prop:NTK_limit_bounds_z}.
    All subsequent computations are performed conditionally on this set.\\
    \textit{Step 1: Training-time evolution of the NN function~$f_{\theta}^N$.}
    With a Taylor series expansion,
    we can write for the NN function output
    \begin{equation}
        \label{eq:proof:thm:NTK_limit:evolution_fk+1_10}
        f_{\theta_{k+1}}^N(x^m)
        = f_{\theta_{k}}^N(x^m) + \nabla_\theta f^N_{\theta_k}(x^m) \cdot (\theta_{k+1}-\theta_k) + \frac{1}{2} (\theta_{k+1}-\theta_k)^\top H^N_{\tilde{\theta}_k^m}(x^m) (\theta_{k+1}-\theta_k),
    \end{equation}
    where $\tilde{\theta}_k^m$ is a point on the line segment $[\theta_{k},\theta_{k+1}]$, which may differ for different $x^m$.
    For notational simplicity, we denote in what follows the remainder term by
    \begin{equation}
        R_k^N(x^m)
        = \frac{1}{2} (\theta_{k+1}-\theta_k)^\top H^N_{\tilde{\theta}_k^m}(x^m) (\theta_{k+1}-\theta_k)
        = \frac{\alpha^2}{2} \left(z^N_k\right)^\top H^N_{\tilde{\theta}_k^m}(x^m)z^N_k,
    \end{equation}
    where we used the definition of the Newton update \eqref{eq:regularizedNewton} in the last equality.
    With the definition of the Jacobian \eqref{eq:Jacobian_fx} and the definition of the Newton update \eqref{eq:regularizedNewton},
    we can write \eqref{eq:proof:thm:NTK_limit:evolution_fk+1_10} for all training data samples in vector-form as
    \begin{equation}
        \label{eq:proof:thm:NTK_limit:evolution_fk+1_20}
    \begin{split}
        f_{\theta_{k+1}}^N(\x)
        &= f_{\theta_{k}}^N(\x) + \alpha J_{\theta_k}^N z^N_k + R^N_k(\x) \\
        &= f_{\theta_{k}}^N(\x) + \alpha \sum_{i=1}^N J_{\theta^i_k}^N (z^N_k)^i + R^N_k(\x) \\
        &= f_{\theta_{k}}^N(\x) + \alpha \frac{1}{M}\sum_{i=1}^N J_{\theta^i_k}^N \left(D_{\theta^i_k}^N\right)^{-1} \left(J_{\theta^i_k}^N\right)^\top \zeta_k^N + R^N_k(\x),
    \end{split}
    \end{equation}
    where we used Formula~\eqref{eq:zN_reformulation} 
    together with the observation that the neuron-wise updates separate due to the block-diagonal structure of the matrix $D_{\theta}^N$ (and hence its inverse) in the last step.
    Here,
    \begin{equation}
        \label{eq:proof:thm:NTK_limit:evolution_fk+1_21}
    \begin{split}
        \zeta_k^N
        &= \left(\Id_M + \frac{1}{M}J_{\theta_k}^N\left(D_{\theta_k}^N\right)^{-1}\left(J_{\theta_k}^N\right)^\top \right)^{-1}\left(\y - f_{\theta_k}^N(\x)\right) \\
        &= \left(\Id_M + \frac{1}{M}\sum_{i=1}^N J_{\theta^i_k}^N \left(D_{\theta^i_k}^N\right)^{-1} \left(J_{\theta^i_k}^N\right)^\top \right)^{-1}\left(\y - f_{\theta_k}^N(\x)\right),
    \end{split}
    \end{equation}
    again using the block-diagonal structure of the matrix $D_{\theta}^N$ (and hence its inverse).
    Notice now that with the choice of the regularizer $\gamma^N=\frac{\gamma}{N^{2\beta-1}}$, we can write
    \begin{equation}
    \begin{split}
        \frac{1}{M}\sum_{i=1}^N J_{\theta^i_k}^N \left(D_{\theta^i_k}^N\right)^{-1} \left(J_{\theta^i_k}^N\right)^\top
        &= \frac{1}{MN}\sum_{i=1}^N \left(N^\beta J_{\theta^i_k}^N\right) \left(N^{2\beta-1}D_{\theta^i_k}^N\right)^{-1} \left(N^\beta J_{\theta^i_k}^N\right)^\top \\
        &= \frac{1}{MN}\sum_{i=1}^N s^N_{\theta_k^i} \left(d_{\theta^i_k}^N\right)^{-1} \left(s^N_{\theta_k^i}\right)^\top,
    \end{split}
    \end{equation}
    where we refer to the notations paragraph in the introduction for the notation of~$s^N_{\theta_k^i}$,
    and to \eqref{eq:def_d} for the notation of~$d^N_{\theta_k^i}$.
    Leveraging these reformulations,
    we can re-write \eqref{eq:proof:thm:NTK_limit:evolution_fk+1_20}--\eqref{eq:proof:thm:NTK_limit:evolution_fk+1_21} as
    \begin{equation}
        \label{eq:proof:thm:NTK_limit:evolution_fk+1_40}
    \begin{split}
        f_{\theta_{k+1}}^N(\x)
        &= f_{\theta_{k}}^N(\x) + \alpha \frac{1}{MN}\sum_{i=1}^N s^N_{\theta_k^i} \left(d_{\theta^i_k}^N\right)^{-1} \left(s^N_{\theta_k^i}\right)^\top \zeta_k^N + R^N_k(\x)
    \end{split}
    \end{equation}
    with 
    \begin{equation}
        \label{eq:proof:thm:NTK_limit:evolution_fk+1_41}
    \begin{split}
        \zeta_k^N
        &= \left(\Id_M + \frac{1}{MN}\sum_{i=1}^N s^N_{\theta_k^i} \left(d_{\theta^i_k}^N\right)^{-1} \left(s^N_{\theta_k^i}\right)^\top \right)^{-1}\left(\y - f_{\theta_k}^N(\x)\right),
    \end{split}
    \end{equation}
    which can be, using the definition of the NNTK~\eqref{eq:NNTK}, re-written as in \eqref{eq:NNfunction_evolution_prelimit}.
    That is,
    \begin{equation}
            \label{eq:proof:thm:NTK_limit:evolution_fk+1_41rewr}
        \begin{split}
            f_{\theta_{k+1}}^N(\x)
            &= f_{\theta_{k}}^N(\x) + \alpha \CB^N_k\left(\y - f_{\theta_{k}}^N(\x)\right)  + R^N_k(\x).
        \end{split}
    \end{equation}

    \textit{Step 2: Bound on NN function discrepancy between finite-width regime~$f_{\theta}^N$ and infinite-width limit~$f^*$.}
    Using the recursive expressions \eqref{eq:proof:thm:NTK_limit:evolution_fk+1_41rewr} (a.k.a.\@ \eqref{eq:NNfunction_evolution_prelimit}) and \eqref{eq:NNfunction_evolution_limit} for $f_{\theta_{k}}^N(\x)$ and $f_{k}^*(\x)$, respectively,
    we can bound for all $k=0,\dots,K-1$ with triangle inequality
    \begin{equation}
        \label{eq:proof:thm:NTK_limit:discrepancy_10}
    \begin{split}
        \N{f_{\theta_{k+1}}^N(\x) - f_{k+1}^*(\x)}_2
        &\leq \N{f_{\theta_{k}}^N(\x)-f_{k}^*(\x)}_2 \\
        &\quad\, + \alpha \N{\CB^N_k\left(\y - f_{\theta_{k}}^N(\x)\right) - \CB^*_0\left(\y - f_k^*(\x)\right)}_2 \\
        &\quad\,+ \N{R^N_k(\x)}_2 \\
        &\leq \N{f_{\theta_{k}}^N(\x)-f_{k}^*(\x)}_2 \\
        &\quad\, + \alpha \left(\N{\left(\CB^N_k-\CB^*_0\right)\left(\y - f_{\theta_{k}}^N(\x)\right)}_2 + \N{\CB^*_0\left(f_{\theta_{k}}^N(\x)-f_k^*(\x)\right)}_2\right) \\
        &\quad\,+ \N{R^N_k(\x)}_2,\\
        &\leq \N{f_{\theta_{k}}^N(\x)-f_{k}^*(\x)}_2 \\
        &\quad\, + \alpha \left(\N{\CB^N_k-\CB^*_0}_{\mathrm{op}}\N{\y - f_{\theta_{k}}^N(\x)}_2 + \N{\CB^*_0}_{\mathrm{op}}\N{f_{\theta_{k}}^N(\x)-f_k^*(\x)}_2\right) \\
        &\quad\,+ \N{R^N_k(\x)}_2 \\
        &\leq \left(1+\alpha C(\sigma,\CD,\mu_0)\right)\N{f_{\theta_{k}}^N(\x)-f_{k}^*(\x)}_2 \\
        &\quad\, + \alpha \N{\CB^N_k-\CB^*_0}_{\mathrm{op}}\N{\y - f_{\theta_{k}}^N(\x)}_2 + \N{R^N_k(\x)}_2,
    \end{split}
    \end{equation}
    where the bound in the last step holds due to Lemma~\ref{lem:standardNTKB:operatornorm}.
    It remains to estimate the last two terms on the right-hand side of \eqref{eq:proof:thm:NTK_limit:discrepancy_10} in order to be able to conclude recursively.

    For the first of which, after taking the conditional expectation w.r.t.\@ the set $\Omega_R$,
    and recalling that on this set it holds $\Nbig{\y-f_{\theta_{k}}^N(\x)}_2 \leq \sqrt{M}C_R$ for all $k=0,\dots,K$ for the residual, 
    we obtain with Lemma~\ref{lem:differenceNNTK_finite_infinite} used in the second step that
    \begin{equation}
        \label{eq:proof:thm:NTK_limit:discrepancy_10:1:0_final}
    \begin{split}
        &\bbE\left[\N{\CB^N_k-\CB^*_0}_{\mathrm{op}}\N{\y - f_{\theta_{k}}^N(\x)}_2\Big| \,\Omega_R\right] \\
        &\qquad\,
        \leq\sqrt{M}C_R\bbE\left[\N{\CB^N_k-\CB^*_0}_{\mathrm{op}}\Big| \,\Omega_R\right] \\
        &\qquad\,
        \leq \left(1+\frac{C(\sigma,\CD,\mu_0)}{\gamma}\right) \frac{C(\sigma,\CD,\mu_0,C_c,C_R,C_z)\sqrt{M}(K+1)}{\gamma}\left(\frac{1}{N^{1-\beta}}+\frac{\log{(2M)}}{N^{1/2}}\right).
    \end{split}
    \end{equation}
    For the remaining term on the right-hand side of \eqref{eq:proof:thm:NTK_limit:discrepancy_10}, we estimate
    \begin{equation}
        \label{eq:proof:thm:NTK_limit:discrepancy_10:2:0}
    \begin{split}
        \N{R^N_k(\x)}_2^2
        &= \frac{\alpha^4}{4}\sum_{m=1}^M \abs{ \left(z^N_k\right)^\top H^N_{\tilde{\theta}_k^m}(x^m)z^N_k}^2 = \frac{\alpha^4}{4}\sum_{m=1}^M \abs{\sum_{i=1}^N\left((z^N_k)^i\right)^\top H^N_{(\tilde{\theta}_k^m)^i}(x^m)(z^N_k)^i}^2 \\
        &\leq N\frac{\alpha^4}{4}\sum_{m=1}^M \sum_{i=1}^N\abs{\left((z^N_k)^i\right)^\top H^N_{(\tilde{\theta}_k^m)^i}(x^m)(z^N_k)^i}^2 \\
        &\leq N\frac{\alpha^4}{4}\sum_{m=1}^M \sum_{i=1}^N \N{H^N_{(\tilde{\theta}_k^m)^i}(x^m)}_{\mathrm{op}}^2\N{(z^N_k)^i}_2^4 \\
        &\leq N\frac{\alpha^4}{4}\sum_{m=1}^M \sum_{i=1}^N C(\sigma,\CD,C_c)N^{-2\beta}\frac{C_z^4}{\left(\gamma N^{1-\beta}\right)^4} \\
        &\leq C(\sigma,\CD,C_c,C_z)M\frac{\alpha^4}{\gamma^4} \frac{1}{N^{2(1 - \beta)}},
    \end{split}
    \end{equation}
    where we recall the block-diagonal structure of the matrix $H^N_{\tilde{\theta}_k^m}(x^m)$ to obtain the second equality in the first line,
    while employing Cauchy-Schwarz inequality and the definition of the operator norm in the second line.
    Moreover, we used Lemma~\ref{lem:op_norm:Hessian} to obtain the inequality in the fourth line after noticing that with $\tilde{\theta}_k^m\in[\theta_{k},\theta_{k+1}]$ it holds 
    $\absnormal{(\tilde{c}_k^m)^i}\leq C_c$ for all $i=1,\dots,N$
    due to $\absnormal{c_{k}^i}\leq C_c$ and $\absnormal{c_{k+1}^i}\leq C_c$ being valid due to the a priori bounds described at the beginning of the proof.
    Moreover, we used that $\N{(z^N_k)^i}_2 \leq \frac{C_z}{\gamma N^{1-\beta}}$ for all $i=1,\dots,N$ and all $k=0,\dots,K$ due to the a priori bounds described at the beginning of the proof and valid due to Lemma~\ref{prop:NTK_limit_bounds_z}.
    
    Combining the estimates~\eqref{eq:proof:thm:NTK_limit:discrepancy_10:1:0_final} and \eqref{eq:proof:thm:NTK_limit:discrepancy_10:2:0},
    and inserting them into \eqref{eq:proof:thm:NTK_limit:discrepancy_10} after taking the conditional expectation w.r.t.\@ the set $\Omega_R$,
    yields    
    \begin{allowdisplaybreaks}
    \begin{align}
        &\bbE\left[\N{f_{\theta_{k+1}}^N(\x) - f_{k+1}^*(\x)}_2\Big| \,\Omega_R\right] \notag\\
        &\qquad\quad\,\leq \left(1+\alpha C(\sigma,\CD,\mu_0)\right)\bbE\left[\N{f_{\theta_{k}}^N(\x)-f_{k}^*(\x)}_2\Big| \,\Omega_R\right]\notag \\
        &\qquad\quad\,\quad\, + \alpha \left(1+\frac{C(\sigma,\CD,\mu_0)}{\gamma}\right)\frac{C(\sigma,\CD,\mu_0,C_c,C_R,C_z)\sqrt{M}(K+1)}{\gamma}\left(\frac{1}{N^{1-\beta}}+\frac{\log{(2M)}}{N^{1/2}}\right)  \notag\\
        &\qquad\quad\,\quad\,+ C(\sigma,\CD,C_c,C_z)\sqrt{M}\frac{\alpha^2}{\gamma^2} \frac{1}{N^{1 - \beta}} \notag\\
        &\qquad\quad\,\leq \left(1+\alpha C(\sigma,\CD,\mu_0)\right)\bbE\left[\N{f_{\theta_{k}}^N(\x)-f_{k}^*(\x)}_2\Big| \,\Omega_R\right]\notag \\
        &\qquad\quad\,\quad\, + \frac{\alpha}{\gamma}\left(1+\frac{1}{\gamma}+\frac{\alpha}{\gamma}\right)C(\sigma,\CD,\mu_0,C_c,C_R,C_z)\sqrt{M}(K+1)\left(\frac{1}{N^{1-\beta}}+\frac{\log{(2M)}}{N^{1/2}}\right).
        \label{eq:proof:thm:NTK_limit:discrepancy_10_final}
    \end{align}
    \end{allowdisplaybreaks}%

    \textit{Step 3: Conclusion.}
    Abbreviating $e_k \coloneqq \bbE\big[\!\N{f_{\theta_{k}}^N(\x)-f_{k}^*(\x)}_2\big| \,\Omega_R\big]$, we can re-write \eqref{eq:proof:thm:NTK_limit:discrepancy_10_final} in the form $e_{k+1} \leq (1+c_1) e_k + c_2$,
    for which one can easily obtain recursively
    \begin{equation}
        e_k
        \leq (1+c_1)^k e_0 + c_2 \sum_{\ell=0}^{k-1}(1+c_1)^{\ell}
        \leq (1+c_1)^k e_0 + c_2 \frac{(1+c_1)^k-1}{c_1}
        \leq (1+c_1)^k \left(e_0 + \frac{c_2}{c_1}\right),
    \end{equation}
    where we used that $c_1$ and $c_2$ are positive to obtain the last step.
    Since by Lemma~\ref{lem:estimate_NNfunction_E} it holds 
    \begin{equation}
        \label{eq:proof:thm:NTK_limit:final_0}
    \begin{split}
        \bbE \N{f_{\theta_{0}}^N(\x)-f_{0}^*(\x)}_2
        &= \bbE \N{f_{\theta_{0}}^N(\x)}_2 \\
        &\leq \sqrt{\bbE \N{f_{\theta_{0}}^N(\x)}_2^2}
        = \sqrt{\sum_{m=1}^M \bbE \abs{f_{\theta_{0}}^N(x^m)}^2}
        \leq C(\sigma,\mu_0) \sqrt{M} \frac{1}{N^{\beta-1/2}}
    \end{split}
    \end{equation}
    and hence 
    \begin{equation}
    \begin{split}
        \bbE \left[\N{f_{\theta_{0}}^N(\x)-f_{0}^*(\x)}_2\Big| \,\Omega_R\right]
        &\leq \frac{\bbE \N{f_{\theta_{0}}^N(\x)-f_{0}^*(\x)}_2}{\bbP\left(\Omega_R\right)} \\
        &\leq 2\bbE \N{f_{\theta_{0}}^N(\x)-f_{0}^*(\x)}_2 \leq C(\sigma,\mu_0) \sqrt{M} \frac{1}{N^{\beta-1/2}},
    \end{split}
    \end{equation}
    where the first inequality follows from the fact that for a non-negative random variable~$X$ it holds $\bbE\left[X\big|\Omega_R\right] \leq \frac{1}{\bbP(\Omega_R)}\bbE X$, which follows from the law of total expectation
    $\bbE X = \bbE\left[X\big|\Omega_R\right]\bbP(\Omega_R) + \bbE\left[X\big|\Omega_R^c\right]\bbP(\Omega_R^c) \geq \bbE\left[X\big|\Omega_R\right]\bbP(\Omega_R)$ thanks to the non-negativity of $X$.
    For the second inequality we note further that with $\delta\in(0,1/2]$ it holds $\bbP(\Omega_R)\geq 1-\delta \geq 1/2$.
    Combining this estimate with the recursion above, we obtain
    \begin{equation}
    \label{eq:proof:thm:NTK_limit:final}
    \begin{split}
        &\bbE\left[
        \N{f_{\theta_k}^N(\x)-f_k^*(\x)}_2
        \Big|\,\Omega_R\right] \\
        &\qquad\quad
        \leq
        \left(1+\alpha C(\sigma,\CD,\mu_0)\right)^k
        \left[
            1+\frac{1}{\gamma}
            \left(
                1+\frac{1}{\gamma}
                +\frac{\alpha}{\gamma}
            \right)
        \right]
        C(\sigma,\CD,\mu_0,C_c,C_R,C_z)
        \sqrt{M}(K+1)\cdot\\
        &\qquad\quad\qquad\cdot
        \left(
            \frac{1}{N^{\beta-1/2}}
            +\frac{1}{N^{1-\beta}}
            +\frac{\log(2M)}{N^{1/2}}
        \right)\\
        &\qquad\quad
        \leq
        C(\sigma,\CD,\mu_0,C_c,C_R,C_z,K)
        \sqrt{M}(1+\alpha)^K
        \left[
            1+\frac{1}{\gamma}
            \left(
                1+\frac{1}{\gamma}
                +\frac{\alpha}{\gamma}
            \right)
        \right]\cdot\\
        &\qquad\quad\qquad\cdot
        \left(
            \frac{1}{N^{\beta-1/2}}
            +\frac{1}{N^{1-\beta}}
            +\frac{\log(2M)}{N^{1/2}}
        \right)\\
        &\qquad\quad
        \leq
        C(\sigma,\CD,\mu_0,C_c,C_R,C_z,K)
        \sqrt{M}(1+\alpha)^{K+1}
        \left(1+\frac{1}{\gamma}\right)^2\cdot\\
        &\qquad\quad\qquad\cdot
        \left(
            \frac{1}{N^{\beta-1/2}}
            +\frac{1}{N^{1-\beta}}
            +\frac{\log(2M)}{N^{1/2}}
        \right)\\
        &\qquad\quad\, \leq C(\alpha,\gamma,\sigma,\CD,\mu_0,C_c,C_R,C_z,K)\sqrt{M} \left(\frac{1}{N^{\beta-1/2}} + \frac{1}{N^{1-\beta}}+\frac{\log{(2M)}}{N^{1/2}}\right).
    \end{split}
    \end{equation}
    Since all random variables involved are non-negative, it follows from \eqref{eq:proof:thm:NTK_limit:final} that
    \begin{equation}
        \label{eq:proof:thm:NTK_limit:final2}
    \begin{split}
        &\bbE\left[\max_{k=0,\dots,K}\N{f_{\theta_{k}}^N(\x) - f_{k}^*(\x)}_2\Big| \,\Omega_R\right]  \\
        &\qquad\quad\, \leq \sum_{k=0}^K\bbE\left[\N{f_{\theta_{k}}^N(\x) - f_{k}^*(\x)}_2\Big| \,\Omega_R\right] \\
        &\qquad\quad\, \leq (K+1)C(\alpha,\gamma,\sigma,\CD,\mu_0,C_c,C_R,C_z,K)\sqrt{M} \left(\frac{1}{N^{\beta-1/2}} + \frac{1}{N^{1-\beta}}+\frac{\log{(2M)}}{N^{1/2}}\right) \\
        &\qquad\quad\, \leq C(\alpha,\gamma,\sigma,\CD,\mu_0,C_c,C_R,C_z,K)\sqrt{M} \left(\frac{1}{N^{\beta-1/2}} + \frac{1}{N^{1-\beta}}+\frac{\log{(2M)}}{N^{1/2}}\right).
    \end{split}
    \end{equation}
    Let us denote by $\Omega_N$ the event where
    \begin{equation}
        \max_{k=0,\dots,K}\N{f_{\theta_{k}}^N(\x) - f_{k}^*(\x)}_2 > \frac{C\sqrt{M}}{\delta}\left(\frac{1}{N^{\beta-1/2}} + \frac{1}{N^{1-\beta}}+\frac{\log{(2M)}}{N^{1/2}}\right)
    \end{equation}
    with $C=C(\alpha,\gamma,\sigma,\CD,\mu_0,C_c,C_R,C_z,K)$ denoting the constant from \eqref{eq:proof:thm:NTK_limit:final2}.
    For the probability of this set,
    it holds with conditional Markov's inequality in the third line that
    \begin{equation}
    \begin{split}
        \bbP\left(\Omega_N\right)
        = \bbP\left(\Omega_N \cap \Omega_R\right) + \bbP\left(\Omega_N \cap \Omega_R^c\right)
        &\leq \bbP\left(\Omega_N \cap \Omega_R\right) + \bbP\left(\Omega_R^c\right) \\
        &\leq \bbP\left(\Omega_N \mid \Omega_R\right)\bbP\left(\Omega_R\right) + \bbP\left(\Omega_R^c\right)
        \leq \bbP\left(\Omega_N \mid \Omega_R\right) + \delta \\
        &\leq \frac{\bbE\left[\max_{k=0,\dots,K}\N{f_{\theta_{k}}^N(\x) - f_{k}^*(\x)}_2 \mid \Omega_R\right]}{\frac{C\sqrt{M}}{\delta}\left(\frac{1}{N^{\beta-1/2}} + \frac{1}{N^{1-\beta}}+\frac{\log{(2M)}}{N^{1/2}}\right)} + \delta \\
        &\leq 2\delta,
    \end{split}
    \end{equation}
    where we recall for the second line that $\bbP\left(\Omega_R^c\right)\leq \delta$ since $\bbP\left(\Omega_R\right)\geq 1-\delta$.
    The last step follows using
    \eqref{eq:proof:thm:NTK_limit:final2},
    which concludes the proof.
\end{proof}

\subsection{Proof of Corollary~\ref{cor:NTK_limit}}
\label{sec:app:proof_cor_NTK_limit}

\begin{proof}[Proof of Corollary~\ref{cor:NTK_limit}]
    To keep the notation clear,
    let us write
    \begin{equation}
        e^N = \max_{k=0,\dots,K}\N{f_{\theta_k}^N(\x) - f_k^*(\x)}_2.
    \end{equation}
    For sufficiently large $N$ (as required by \eqref{eq:thm:asm:Nlargeenough}), it holds by Theorem~\ref{thm:NTK_limit} that $\bbP\big(e^N \leq \nu_N(\delta) \big) \geq1-2\delta$,
    where we abbreviate
    \begin{equation}
        \nu_N(\delta) \coloneqq \frac{C\sqrt M}{\delta}\left(\frac{1}{N^{\beta-1/2}} + \frac{1}{N^{1-\beta}}+\frac{\log{(2M)}}{N^{1/2}}\right).
    \end{equation}
    Since $\beta \in (1/2,1)$, each exponent w.r.t.\@ $N$ in $\nu_{N}(\delta)$ is negative,
    hence ensuring for fixed $\delta>0$ that $\nu_N(\delta)\rightarrow 0$ as $N\rightarrow\infty$.
    Therefore,
    for arbitrary $\varepsilon>0$,
    there exists $\widetilde{N} = \widetilde{N}(\varepsilon,\delta)$ such that $\nu_N(\delta) \le \varepsilon$ for all $N \geq \widetilde{N}$.
    For such $N\geq \widetilde{N}(\varepsilon,\delta)$, we have
    $\mathbb{P}(e^N > \varepsilon) \leq \mathbb{P}\big(e^N > \nu_N(\delta)\big) \leq 2\delta$.
    Since $\varepsilon>0$ and $\delta>0$ were arbitrary, it follows that $e^N \rightarrow 0$ in probability as $N\rightarrow\infty$.
\end{proof}

\subsection{Proof of Theorem~\ref{thm:convergence_finite}}
\label{sec:app:convergence_finite}

Using the results of Theorems~\ref{thm:convergence} and \ref{thm:NTK_limit}, we can directly infer Theorem~\ref{thm:convergence_finite}.

\begin{proof}[Proof of Theorem~\ref{thm:convergence_finite}]
    For a given number of iterations $K$, the statement follows with triangle inequality
    \begin{equation}
    \begin{split}
        \N{f_{\theta_{K}}^N(\x) - \y}_2
        &\leq \N{f_{\theta_{K}}^N(\x) - f_{K}^*(\x)}_2 + \N{f_{K}^*(\x) - \y}_2,
    \end{split}
    \end{equation}
    where the first summand is bounded with probability $1-2\delta$ by $\frac{C\sqrt{M}}{\delta}\left(\frac{1}{N^{\beta-1/2}} + \frac{1}{N^{1-\beta}}+\frac{\log(2M)}{N^{1/2}}\right)$ thanks to  Theorem~\ref{thm:NTK_limit},
    while the second term is upper bounded by $\left(1-\alpha \lambda_{\min}(\CB^*_0)\right)^K\N{\y}_2$ due to Theorem~\ref{thm:convergence}.
\end{proof}

\subsection{Proof of Corollary~\ref{cor:convergence_finite}}
\label{sec:app:proof_cor_convergence_finite}

\begin{proof}[Proof of Corollary~\ref{cor:convergence_finite}]
    To keep the notation clear,
    let us write
        \begin{equation}
        e^{K,N} = \N{f_{\theta_K}^N(\x) - \y}_2.
    \end{equation}
    For sufficiently large $N$ (as required by \eqref{eq:thm:asm:Nlargeenough}), it holds by Theorem~\ref{thm:convergence_finite} that $\bbP\big(e^{K,N} \leq \nu_{K,N}(\delta) \big)\geq1-2\delta$,
    where we abbreviate
    \begin{align}
        \nu_{K,N}(\delta) &\coloneqq \nu^1_{K,N}(\delta) + \nu^2_{K}
        \intertext{with}
        \nu^1_{K,N}(\delta) &\coloneqq \frac{C(K)\sqrt{M}}{\delta}\left(\frac{1}{N^{\beta-1/2}} + \frac{1}{N^{1-\beta}}+\frac{\log{(2M)}}{N^{1/2}}\right), \\
        \nu^2_{K} &\coloneqq \left(1-\alpha \lambda_{\min}(\CB^*_0)\right)^K\N{\y}_2.
    \end{align}
    Note that the constant $C$ depends, amongst others which we do not denote explicitly, on $K$.
    Since $\left(1-\alpha \lambda_{\min}(\CB^*_0)\right)<1$,
    for arbitrary $\varepsilon>0$,
    there exists $\widetilde{K} = \widetilde{K}(\varepsilon)$ such that $\nu^2_{K} \le \frac{\varepsilon}{2}$ for all $K \geq \widetilde{K}$.
    Since $\beta \in (1/2,1)$, each exponent w.r.t.\@ $N$ in $\nu^1_{K,N}(\delta)$ is negative,
    hence ensuring for any $K\geq \widetilde{K}$ and $\delta>0$ that $\nu^1_{K,N}(\delta)\rightarrow 0$ as $N\rightarrow\infty$.
    Therefore,
    for the above $\varepsilon$,
    there exists $\widetilde{N} = \widetilde{N}(\varepsilon,\delta,K)$ such that $\nu_{K,N}^1(\delta) \le \frac{\varepsilon}{2}$ for all $N \geq \widetilde{N}$.
    For such $K\geq \widetilde{K}(\varepsilon)$ and such $N\geq\widetilde{N}(\varepsilon,\delta,K)$, we have
    $\mathbb{P}(e^{K,N} > \varepsilon) \leq \mathbb{P}\big(e^{K,N} > \nu_{K,N}(\delta)\big) \leq 2\delta$.
    Since $\varepsilon>0$ and $\delta>0$ were arbitrary, it follows that $e^{K,N} \rightarrow 0$ in probability as $K, N \rightarrow \infty$, provided $N$ grows sufficiently fast relative to $K$ (i.e., $N \geq \widetilde{N}(K)$).
\end{proof}

{
\section{Experimental details for the numerical results in Section~\ref{sec:numerics}}
\label{app:experiments}

In the numerical experiments of Figure~\ref{fig:Comparison_GD_Newton},
we compare gradient descent, Adam, the regularized Gauss--Newton method, and the regularized Newton method.

Gradient descent and Adam use the standard implementations provided by PyTorch.
Their learning rates are selected by a logarithmic learning-rate sweep over the interval $[10^{-4},10]$.

For the second-order methods, the resulting linear systems are solved using the conjugate-gradient method together with efficient Gauss--Newton-vector products and Hessian-vector products computed using Pearlmutter's method~\cite{pearlmutter1994fast}, respectively.\\
The regularizer parameter is initialized with the value $\gamma_0=10^{-4}$ and adapted heuristically throughout training.
Whenever a proposed update decreases the training loss, it is accepted and the regularizer parameter is reduced by a factor of $0.8$.
Otherwise, the update is rejected, the previous parameters are restored, and the regularizer parameter is increased by a factor of $2$.
To avoid numerical instabilities, the regularizer parameter is restricted to the interval $[10^{-8},10^2]$.

The reported curves correspond to the empirical mean over $100$ independent random initializations, with shaded regions indicating $5$th--$95$th percentile bands.

\begin{figure}[ht]
    \centering
    \includegraphics[width=0.75\textwidth,clip]
        {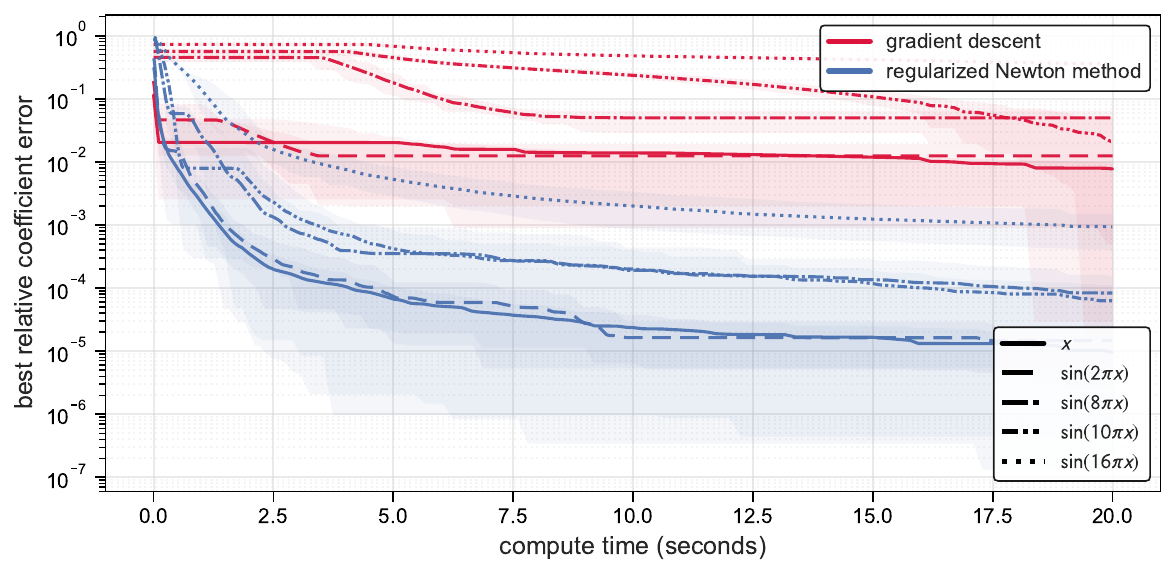}
    \caption{
    Learning of the individual frequency components of the target function $y=4x \!+\! \sin(2\pi x)  \!+\! 0.8(\sin(8\pi x) \!+\! \sin(10\pi x) \!+\! \sin(16\pi x))$ by gradient descent and the regularized Newton method.
    Depicted is the relative error of the corresponding coefficients as a function of elapsed compute time for the experimental setting of Figure~\ref{fig:Comparison_GD_Newton}.
    (Results are averaged across $100$ runs.)}
    \label{fig:Comparison_GD_Newton_FrequenciesLearned}
\end{figure}

}

\ifarxiv

    \acknowledgments

    \bibliographystyle{abbrv}
    \bibliography{biblio}

\fi

\end{document}